\documentclass[twoside,11pt]{article}

\usepackage[preprint]{jmlr2e}

\usepackage{amsmath,amssymb}
\usepackage{bm}

\usepackage{algorithm2e}
\SetKwInput{KwInput}{Input}
\SetKwInput{KwOutput}{Output}
\SetKwComment{Comment}{/* }{ */}
\RestyleAlgo{ruled}

\usepackage{tikz}
\usepackage{natbib}
\usepackage{graphicx}
\usepackage{hyperref}
\usepackage{url}

\bibpunct[, ]{(}{)}{,}{a}{}{,}

\newcommand{\yrcite}[1]{(\citeyear{#1})}
\newcommand{\newcite}[1]{\citeauthor{#1} (\citeyear{#1})}
\newcommand{\resultcite}[2]{(\citeauthor{#1}, \citeyear{#1}, #2)}

\input{my_symbol.sty}
\input{souvik.sty}

\newtheorem{theoremEnv}{Theorem}[section]
\newtheorem{propositionEnv}[theoremEnv]{Proposition}
\newtheorem{lemmaEnv}[theoremEnv]{Lemma}

\newtheorem{assumptionEnv}[theoremEnv]{Assumption}

\usepackage{xcolor}

\jmlrheading{1}{2026}{1-??}{??/??}{??/??}{LeRuizDhara24}{My Le, Luana Ruiz, and Souvik Dhara}

\ShortHeadings{Distance-Preserving Embeddings in Inhomogeneous Random Graphs}{Le, Ruiz, and Dhara}
\firstpageno{1}

\begin{document}

\title{Beyond Worst-Case Distortions: Distance-Preserving Embeddings in Inhomogeneous Random Graphs}

\author{\name My Le \email mle19@jh.edu \\
        \addr Department of Applied Mathematics \& Statistics\\
        Johns Hopkins University\\
        Baltimore, MD 21218, USA
        \AND
        \name Luana Ruiz \email lrubini1@jh.edu \\
        \addr Department of Applied Mathematics \& Statistics\\
        Johns Hopkins University\\
        Baltimore, MD 21218, USA
        \AND
        \name Souvik Dhara \email sdhara@gatech.edu \\
        \addr School of Industrial and Systems Engineering\\
        Georgia Institute of Technology\\
        Atlanta, GA 30332, USA}

\editor{[Editor Name]}

\maketitle

\begin{abstract}
Graph machine learning provides powerful tools for understanding complex networks and learning meaningful node representations. A central challenge, however, is designing embeddings with minimal distortion of \emph{both} local and global functionals, such as shortest path lengths. Prior distortion guarantees for distance-preserving embeddings are worst-case in nature, producing overly pessimistic bounds that fail to capture the structure of \emph{typical} large-scale networks. To address this, we analyze shortest-path approximation via landmark-based embeddings on inhomogeneous random graphs (IHGs), a general model with type-dependent edge probabilities. By retaining shortest paths to a small set of reference nodes called landmarks, landmark-based methods effectively function as virtual \textit{graph spanners}, where structural heterogeneity and controlled neighborhood expansion modeled via multi-type branching processes enable significantly tighter dimension–distortion trade-offs, i.e. $\Omega\left(n^{1-\varepsilon}\log n\right)$, than classical worst-case bounds, i.e. $\Omega\left(n^{\frac{2(1-\varepsilon)}{2-\varepsilon}}\log n\right)$ for $(1-\varepsilon)$-distortion and $\Omega\left(n^{\frac{2}{2+\varepsilon}}\log n\right)$ for $(1+\varepsilon)$-distortion. We extend these guarantees to global, component-wide averages and unify the analysis across finite-type and continuous latent spaces through a novel metric sandwiching framework, establishing universal distortion bounds for general $L^2$ kernel models, including heavy-tailed and power-law networks. Finally, we introduce a GNN-augmented variant that replaces rigid, computationally expensive exact shortest-path queries with flexible, structure-aware neural surrogates. By leveraging the inherent alignment between graph neural message-passing and the dynamic programming principles of shortest-path algorithms, our approach demonstrates that models trained on small-scale random graphs learn to extract universal distance-preserving features, achieving robust generalization to large-scale, real-world networks that match or exceed the fidelity of classical, exact landmark-based embeddings.

\end{abstract}


%%%%%%%%%%%%%%%%%%%%%%%%%%%%%%%%%%%%%%%%%%%%%%%%%%%%%%%%%%%%%%%%%%%%%%

\section{Introduction}\label{sec:Intro}

A central challenge in graph learning is to represent networked data in a form that faithfully captures its essential structural characteristics, including local connectivity patterns, mesoscopic organization, and global topology. A common strategy is to map nodes into low-dimensional metric spaces so that distances between embedded points reflect the graph's original structure. These embeddings provide a compact and mathematically tractable representation of the network, enabling both statistical analysis and efficient computation at scale in a wide range of inference tasks, including node classification, link prediction, clustering, and routing \citep{hamilton2017representation, grover2016node2vec, belkin2003laplacian}.

Despite significant progress, node embedding techniques are primarily designed to capture local and mesoscopic-scale structure, failing to accurately preserve global graph functionals.
In particular, a key global functional that is often neglected are shortest-path distances, which are central to routing, navigation, and network efficiency. Shortest path distances are often distorted by low-dimensional 
embeddings, especially in large-scale and structurally heterogeneous graphs 
\citep{goyal2018graph, tsitsulin2018just, brunner2021distance}. This failure to preserve the metric structure of the graph requires a more principled approach to distance-aware embeddings.

\emph{Landmark-based} distance embeddings represent nodes via their distances to a small set of reference nodes called landmarks, naturally inducing a probabilistic virtual \textit{graph spanner} \citep{ahmed2020graph}. Rooted in both scalable shortest-path approximation algorithms \citep{Sarma2010ASD, potamias2009fast, tretyakov2011fast} and the mathematical theory of metric embeddings into Hilbert spaces \citep{Bourgain85, Mat96, LLR95}, these methods are widely used in practice, offering a coordinate-based alternative to traditional edge-deletion spanners \citep{peleg1989graph, althofer1993sparse} while enjoying provable $(1 \pm \varepsilon)$-distortion guarantees \citep{Bourgain85, Sarma2010ASD}.

Despite their empirical success, existing theoretical guarantees require embedding dimension of $\Omega\left(n^{\frac{2(1-\varepsilon)}{2-\varepsilon}}\log n\right)$ for $(1-\varepsilon)$-distortion and $\Omega\left(n^{\frac{2}{2+\varepsilon}}\log n\right)$ for $(1+\varepsilon)$-distortion, which grows polynomially in $n$ and becomes prohibitive at scale \citep{Mat96, Sarma2012, Loukas2020}. These bounds hold for arbitrary graphs and are tight in the worst case, but may be overly pessimistic for the structured, heterogeneous networks that arise in practice. This raises the question of whether sharper guarantees are achievable under more realistic assumptions on the graph structure.

In this paper, we analyze landmark-based distance embeddings on \emph{inhomogeneous random graphs} (IHGs), a flexible model that subsumes variants of the stochastic block model and Chung--Lu-type models \citep{chung2002average, bollobas2007phase, van2013random}, and captures structural heterogeneity such as community organization and degree variability observed in real-world networks. Rather than seeking worst-case guarantees that hold for arbitrary graphs, we adopt an ``average-case'' perspective grounded in random graph theory, analyzing how type-dependent connectivity governs neighborhood expansion and, in turn, the geometry of landmark-based distance-preserving embeddings.

\paragraph{Theoretical Contributions.} Our main contribution is to demonstrate that, for a broad class of IHGs with either discrete or continuous kernels, landmark-based embeddings achieve a polynomial improvement in the distortion--dimension trade-off compared to classical worst-case guarantees. We establish these results by characterizing the fine-grained expansion of local neighborhoods through the lens of multi-type branching processes \citep{athreya2012branching, van2013random}. This probabilistic framework bridges local, heterogeneous node connectivity with the global operator-theoretic properties of the graph, enabling us to map precisely how structural variance inside an IHG controls distance propagation. To this end, our theoretical framework delivers a three-fold contribution:

\begin{itemize}
    \item \textbf{Sharp Point-Wise Guarantees (Theorems \ref{thm:lower-bound} and \ref{thm:upper-bound}):} Within the supercritical regime of IHGs with finitely many types where a unique giant component emerges with high probability (w.h.p., i.e., with probability approaching 1 as $n\to \infty$), we obtain embedding dimension $\Omega\left(n^{1-\varepsilon}\log n\right)$ for both $(1\pm \varepsilon)$-distortions, which is smaller than the worst-case bounds. Since the gain in the embedding dimension is $n^{O(\varepsilon)}$, it diminishes as $\varepsilon \to 0$. This behavior is consistent with the classical work on Lipschitz embeddings by \cite{Bourgain85}, which shows that dense random graphs achieve nearly worst-case distortion bounds.
    
    \item \textbf{Global Average-Case Stability (Theorem \ref{avg-case}):} 
    We extend our point-wise guarantees to the entire topology by showing that the empirical spatial average of the metric distortion—normalized over all valid, connected node pairs $U = \{(u_1, u_2) : u_1 \leftrightarrow u_2, u_1 \neq u_2\}$—concentrates tightly within a $(1\pm\varepsilon)$ window w.h.p. By decoupling the spatial summation, we prove that the mass of pathologically behaving configurations is asymptotically negligible, ensuring that anomalous structural bottlenecks cannot corrupt global downstream empirical risk optimization.
    
    \item \textbf{Universal Continuum Extension (Theorems \ref{thm:metric-sandwich} and \ref{thm:universal-distortion}):} We generalize our finite-type results to continuous latent spaces via a rigorous probability space coupling that sandwiches an arbitrary $L^2$ kernel $\kappa$ between two finite step-function kernels $\kappa_\delta^\pm$. Under supercriticality, the global metric structure of the continuous network behaves as a stable functional limit of these finite approximations as the approximation resolution $\delta\to 0$. This sandwiching argument allows us to transfer our polynomial dimension-distortion trade-offs to non-parametric, heavy-tailed network architectures, including unbounded power-law Chung--Lu configurations.

\end{itemize}

\noindent \textbf{Methodological Contributions.} Building on the theoretical framework, we propose a GNN-augmented variant of landmark-based embeddings that learns to approximate distances to landmarks directly from graph structure, replacing exact shortest-path computations. GNNs are well-suited for this task due to their alignment with dynamic programming, which underpins shortest-path algorithms 
\citep{xu2019can, dudzik2022graph}. Our experiments demonstrate that GNN-based embeddings match or improve upon exact landmark embeddings, particularly in the strongly supercritical regime where neighborhood expansion grows exponentially at a rate determined by the spectral radius of the graph. More strikingly, GNNs trained on small graphs generalize effectively to much larger graphs and to real-world networks, highlighting the practical value of the IHG framework as a training ground for scalable, distance-aware graph representations.

\subsection{Related Work}

The fundamental limits of distance-preserving graph embeddings have been studied extensively. \newcite{Bourgain85} established that preserving all pairwise distances up to a factor of $(1\pm\varepsilon)$ requires dimension at least $k_\varepsilon = \Omega\big({(\log n)^2}/{(\log\log n)^2}\big)$ for worst-case graphs. Subsequent works sharpened these bounds. \newcite{LLR95} showed $k_\varepsilon = \Omega((\log n)^2)$, and \newcite{Mat96} refined it to $k_\varepsilon = \Omega\left(n^{c/(1+\varepsilon)}\right)$ for certain graph families. More recently, \newcite{Naor16} and \newcite{Naor21} demonstrated that graphs with strong expansion properties require polynomial-dimensional embeddings, underscoring the difficulty of distance preservation in the worst case.

On the algorithmic side, landmark-based methods have been extensively studied as practical alternatives \citep{goldberg2005computing, Sarma2010ASD, potamias2009fast, tretyakov2011fast, akiba2013fast, rizi2018shortest, Qi2020ALB}, with \newcite{Sommer2014queries} providing a comprehensive survey. These methods select a small subset of reference nodes and approximate pairwise distances via the triangle inequality, trading exactness for scalability. While empirically effective, their theoretical guarantees inherit the pessimism of worst-case analyses, motivating the average-case perspective taken in this paper.

A broader line of work in graph representation learning studies how structural properties such as local connectivity, higher-order proximities, and spectral geometry govern the quality of learned embeddings. Random-walk-based methods such as DeepWalk \citep{perozzi2014deepwalk} and Node2Vec \citep{grover2016node2vec} capture local structure, while GraRep \citep{cao2015grarep}, PRONE \citep{zhang2021prone}, and adjacency search embeddings \citep{chaitanya2025adjacency} extend this to higher-order proximities. Spectral approaches including Laplacian Eigenmaps \citep{belkin2003laplacian} capture coarse global structure but can be sensitive to degree variability and edge heterogeneity \citep{chung1997spectral, von2007tutorial}. Cauchy embeddings \citep{tang2019cauchy} address this via heavy-tailed similarity measures, improving robustness in heterogeneous networks. None of these methods are designed to preserve the metric structure of the graph explicitly, which is the focus of this work.

%\red{Our results also connect to the growing literature on the theoretical foundations of graph neural networks \citep{xu2019powerful, oono2019graph}, where expansion and heterogeneity play a central role in governing information propagation and the expressiveness of local aggregation mechanisms. Distance-aware embeddings offer a 
%complementary approach, encoding global geometric information explicitly rather than relying solely on iterative message passing.}

Our work is intimately related to the extensive literature on \emph{graph spanners} in theoretical computer science \citep{ahmed2020graph}. Introduced to compress network topologies while bounding distance distortion, a classical $t$-spanner selects a sparse subgraph wherein the shortest path distance between any pair of nodes is at most $t$ times their true distance \citep{peleg1989graph, althofer1993sparse}. While traditional spanner construction algorithms are primarily algorithmic and worst-case deterministic, our work provides an average-case, probabilistic counterpart. Rather than explicitly computing a sparse physical subgraph, our multi-scale landmark framework acts as a virtual distance spanner, leveraging the underlying spectral expansion of inhomogeneous random graphs to guarantee a $(1 \pm \varepsilon)$ metric stretch w.h.p. using sub-linear storage overhead and query time.

Parallel to these representations is the literature on dense and sparse graph limits (graphons and kernels), which model large networks as continuous latent spaces \citep{bollobas2007phase, lovasz2012large}. While graph limit theory provides strong structural and phase-transition guarantees, its explicit geometric implications for shortest-path metric spaces have remained largely unmapped. None of the aforementioned representation frameworks are designed to explicitly preserve the metric structure of the graph across continuous limits, which is a focus of this work.

\section{Landmark-Based Distance-Preserving Embeddings}\label{sec:Prelim}

In this section, we introduce shortest-path distances and describe landmark-based embeddings for distance-preserving node representation, along with their worst-case distortion guarantees. Classical algorithms for exact shortest-path computation are reviewed briefly to motivate the need for approximate, embedding-based approaches at scale.

Throughout, we consider an undirected and unweighted graph $G=(V,E)$ with vertex set $V$ of size $n$ and edge set $E$ of size $m$. Connected components of $G$ are indexed in decreasing order of size, and $\mathcal{C}_{(i)}$ denotes the $i$-th largest connected component. For any two vertices $u_1,u_2 \in V$, the notation $u_1 \leftrightarrow u_2$ indicates that $u$ and $v$ are connected by a path, i.e., they belong to the same connected component.

\subsection{Shortest Path Distances}

For a graph $G=(V,E)$ and vertices $u_1,u_2 \in V$, the shortest path distance (or simply shortest distance) $d(u_1,u_2)$ is defined as the number of edges in a path with the minimum number of edges among all paths connecting $u$ and $v$. Computing shortest distances is a fundamental problem in graph theory and combinatorial optimization, and serves as a core primitive in our analysis.

Classical algorithms for computing shortest paths include Dijkstra's algorithm and its variants. When implemented with simple data structures, Dijkstra's algorithm runs in $O(n^2)$ time for a single source and $O(n^3)$ time for all node pairs. More efficient implementations based on priority queues reduce the single-source complexity to $O(m \log n)$, or to $O(m + n \log n)$ in sparse graphs \citep{schrijver2012history}. 

A number of refinements have been proposed for the single-source setting, building on early work by \cite{dijkstra1959note} and \cite{moore1959shortest}. These include bucket-based methods such as S-Dial, which achieves a running time of $O(m + n \ell_{\max})$ where $\ell_{\max}$ is the maximum edge length, as well as heap-based approaches with complexity $O(m \log n)$ or $O(n \log n)$ in sparse regimes \citep{gallo1988shortest}. Further improvements were introduced by Fredman and Willard, yielding an $O(m + n \log n / \log \log n)$ time algorithm \citep{fredman1990trans}.

For the all-pairs shortest path problem, classical methods such as the Floyd--Warshall algorithm require $O(n^3)$ time \citep{gallo1988shortest}, while the more advanced hidden-path algorithm achieves $O(mn + n^2 \log n)$ complexity \citep{karger1993finding}. Nevertheless, these remain prohibitive for large graphs, motivating approximate methods and in particular embedding-based approaches for distance estimation.

\subsection{Landmark-Based Embeddings}\label{landmark-embeddings}
Although computing exact shortest distances can be prohibitively expensive on large graphs, often one can precompute distances to local landmarks and use these values to approximate shortest path distances efficiently \citep{Sarma2010ASD}.
This is the core idea of landmark-based distance-preserving embeddings. Explicitly, their construction involves two steps:

\noindent \textbf{Local Step.} Sample $r+1$ sets of landmark nodes $S_0, S_1, \dots, S_r \subset V$. For each vertex $u \in V$ and each landmark set $S_j$, compute
$$
[\mathbf{x}_u]_j = \min_{s \in S_j} d(u, s), 
\qquad 
[\boldsymbol{\sigma}_u]_j = \arg\min_{s \in S_j} d(u, s),
$$
where $d(u, s)$ is the shortest distance between $u$ and $s$. Here, $[\mathbf{x}_u]_j$ stores the distance from $u$ to its closest landmark in $S_j$, and $[\boldsymbol{\sigma}_u]_j$ records the corresponding landmark. Collecting these for all $j=0,\dots,r$ gives the \emph{local embedding} $\mathbf{x}_u$ and associated landmark indices $\boldsymbol{\sigma}_u$ for vertex $u$.

\noindent \textbf{Global Step.} Given local embeddings $\mathbf{x}_{u_1}$ and $\mathbf{x}_{u_2}$ for two vertices $u_1, u_2 \in V$, a lower bound on the shortest path distance is obtained as
$$
\underline{d}(u_1,u_2) = \|\mathbf{x}_{u_1} - \mathbf{x}_{u_2}\|_\infty.
$$
An upper bound is obtained by combining distances through common landmarks:
$$
\bar{d}(u_1,u_2) = \min \{ [\mathbf{x}_{u_1}]_i + [\mathbf{x}_{u_2}]_j \ :\ [\boldsymbol{\sigma}_{u_1} \mathbf{1}^\top - \mathbf{1}\boldsymbol{\sigma}_{u_2}^\top]_{ij} = 0 \}.
$$

These bounds follow directly from the triangle inequality. The lower bound $\underline{d}(u_1,u_2)$ relies solely on coordinate-wise differences and thus involves a search over $r+1$ coordinates.
In contrast, the upper bound $\bar{d}(u_1,u_2)$ considers all pairs of coordinates $(i,j)$ such that the closest landmarks coincide and requires at least one landmark set to have strictly one node to prevent $\bar{d}(u_1,u_2)$ from being undefined. In the worst case, this requires checking all $(r+1)^2$ pairs of landmark indices, since each coordinate of $\mathbf{x}_{u_1}$ may correspond to a different landmark than that of $\mathbf{x}_{u_2}$. Hence, the effective search space for computing $\bar{d}(u_1,u_2)$ can scale quadratically with the embedding dimension.

From a practical standpoint, this two-step procedure allows the landmark embeddings computed in the local step to be stored and later retrieved to efficiently compute $\underline{d}(u_1,u_2)$ and $\bar{d}(u_1,u_2)$ as approximations of $d(u_1,u_2)$ on demand. The key advantage of landmark-based algorithms is that they avoid recomputing paths for each query, substantially reducing computational and memory costs while still preserving the graph's structural information.

Such lower and upper bounds have been shown to provide reliable estimates of exact 
distances \citep{Bourgain85, Mat96, Sarma2010ASD, Gubichev2010, Sommer2014queries, 
Akiba2014queries, Meng2015GRECS, Jiang2021Tripoline, Awasthi2021}, enabling rapid 
online queries for individual node pairs as well as efficient all-pairs distance 
approximation. The quality of these approximations is characterized by their 
distortion relative to the true shortest-path distances. Provable guarantees 
on this distortion holding for any graph are discussed next. 

%Universal approximation guarantees, i.e., approximation guarantees applying to \emph{any} graph, have been established for the lower bound $\underline{d}(u_1,u_2)$ and the upper bound $\bar{d}(u_1,u_2)$. These results ensure that the estimated distances provably track true graph distances up to bounded distortion. The lower bound follows from results of \newcite{Mat96}, based on Bourgain's embedding theorem \yrcite{Bourgain85}, while analogous guarantees for the upper bound were derived by \newcite{Sarma2010ASD}. 

\subsection{Distortion on Worst-Case Graphs}

The distortion guarantee for the shortest distance lower bound follows from \newcite{Mat96}, building on Bourgain's 
embedding theorem \yrcite{Bourgain85}. 

\begin{theoremEnv}[\textbf{Worst-Case Lower-Bound Distortion}]\label{thm:bourgain_distortion}
Let $G=(V,E)$ be a graph with $n \ge 3$ nodes and let $u_1,u_2 \in V$. For any $c>1$, there exist embeddings $\bbx^*_{u_1},\bbx^*_{u_2} \in \reals^{D}$ with $D=\Omega(n^{1/c}\log n)$ such that the lower-bound estimator $\underline{d}(u_1,u_2)$ satisfies
\begin{equation*} \label{eqn:bourgain_distortion}
\frac{1}{2c-1} d(u_1,u_2) \le \underline{d}(u_1,u_2) \le d(u_1,u_2).
\end{equation*}
\end{theoremEnv}

The distortion guarantee for the upper bound was derived by \newcite{Sarma2010ASD}.

\begin{theoremEnv}[\textbf{Worst-Case Upper-Bound Distortion}]\label{thm:sarma_distortion}
Let $G=(V,E)$ be a graph with $n \ge 3$ nodes and let $u_1,u_2 \in V$. For any $c>1$, there exist embeddings $\bbx^*_{u_1},\bbx^*_{u_2} \in \reals^{D}$ with $D=\Omega(n^{1/c}\log n)$ such that the upper-bound estimator $\bar{d}(u_1,u_2)$ satisfies
\begin{equation*} \label{eqn:sarma_distortion}
d(u_1,u_2) \le \bar{d}(u_1,u_2) \le (2c-1) d(u_1,u_2).
\end{equation*}
\end{theoremEnv}

The distortion bounds in Theorems \ref{thm:bourgain_distortion} and \ref{thm:sarma_distortion} rely on embeddings that are optimal in an information-theoretic sense, i.e., in the sense that they match known lower bounds on the minimum dimension required to preserve all pairwise distances (within a prescribed distortion) in worst-case metrics \citep{Bourgain85,Mat96, Sarma2010ASD}. More precisely, these embeddings achieve the best possible trade-off between dimension and distortion up to constant factors, independent of computational considerations. However, Theorems~\ref{thm:bourgain_distortion} and 
\ref{thm:sarma_distortion} only guarantee the existence of embeddings achieving the stated 
dimension--distortion trade-offs; they do not ensure that the landmark construction 
described in Section \ref{landmark-embeddings} realizes them. 

Indeed, the choice of landmark sampling is critical. \newcite{Sarma2010ASD} advocate a multiscale sampling scheme in which landmark sets $S_i$ are drawn at exponentially increasing scales, with $|S_i| = 2^i$ for $i = 0,\ldots,r$. This multiresolution design serves complementary roles in controlling the two estimators $\underline{d}$ and $\bar{d}$. For the lower bound, small-scale landmark sets are required to place a landmark within the $\sigma_1 d(u_1,u_2)$-hop neighborhood of $u_1$ while avoiding the larger $\sigma_2 d(u_1,u_2)$-hop neighborhood of $u_2$, with $0<\sigma_1<\sigma_2$ satisfying $\sigma_1+\sigma_2<1$ (see Figure \ref{fig:examples_lb}). Conversely, for the upper bound, larger landmark sets allow w.h.p. the presence of at least one landmark in the overlap of the $\lceil d(u_1,u_2)/2\rceil$-hop neighborhoods of $u_1$ and $u_2$. We adopt this sampling scheme in our analysis, generating sets of size $|S_0|=M^0$, $|S_1|=M^1$, \ldots, $|S_r|=M^r$ for some integer $M>1$.

Another consideration is mitigating the effect of poor quality sample sets. To address it, we sample $R$ sets of each size. When $R=\Omega(n^{1/c})$ and the landmark set sizes grow exponentially with $r=O(\log n)$, the total embedding dimension scales as $\Theta(n^{1/c}\log n)$. Under this regime, the resulting distance estimates satisfy the guarantees of Theorems~\ref{thm:bourgain_distortion} and~\ref{thm:sarma_distortion} simultaneously for all node pairs on arbitrary graphs w.h.p.

While these guarantees hold for any graph, they may be overly pessimistic for 
structured networks. In the following sections, we show that for inhomogeneous random 
graphs---a broad class of models capturing sparsity, community structure, and 
heterogeneous connectivity, significantly tighter distortion--dimension trade-offs 
are achievable.

\section{Inhomogeneous Random Graphs (IHGs)} \label{sec:theory}

In the following we formally introduce the inhomogeneous random graph model. This model captures structural heterogeneity through type-dependent connectivity and provides a tractable probabilistic framework for studying neighborhood expansion and distance behavior. We then formalize a set of assumptions under which the graph exhibits controlled growth and concentration properties, enabling a precise characterization of neighborhood dynamics (Section \ref{growth}) that will be used in our main results (Section \ref{main-res}).

\subsection{Graph Model}

The inhomogeneous graph model with $n$ nodes and $T$ distinct node types will be denoted $IHG(\vec{n},D)$, where $\vec{n} = [n_1, \dots, n_{T}] \in \mathbb{N}^{T}$ is the \emph{type size vector} and $D \in \mathbb{R}^{{T} \times {T}}$ is the \emph{affinity} (or \emph{connectivity}) matrix.
There are $n_t$ nodes of each type $t \in [T]=\{1, \dots, {T}\}$, and each node belongs to exactly one type. The probability of an edge between a node of type $p$ (parent) and a node of type $c$ (child) is
\(P_{pc} = D_{pc}/n_c\) with $0\leq D_{pc}\leq n_c$ \citep{chung2002average, bollobas2007phase}.

The matrix $D$ is non-negative and has a symmetric support, meaning $D_{ij} > 0$ if and only if $D_{ji} > 0$. However, $D$ is not necessarily symmetric: even in undirected graphs where $P_{pc} = P_{cp}$, differing type sizes with $n_c \neq n_p$ result in $D_{pc} \neq D_{cp}$. Consequently, the spectral radius $\rho(D)$ is at most its spectral norm $\|D\|_2$ but not necessarily equal, and the largest magnitude eigenvalue $\lambda_1$ has value at most $\|D\|_2$.

This model captures heterogeneity in network structure by allowing edge probabilities to depend on the types of the endnodes, rather than being uniform across all nodes. At the same time, it generalizes the classical Erd\H{o}s--R\'enyi random graph since, if $T=1$, then the model reduces to an Erd\H{o}s--R\'enyi graph with connection probability $p = D_{11}/{n_1}$. The type-dependent connectivity matrix $D$ controls the expected number of neighbors of each type and introduces structured variability in local neighborhoods, which can better reflect realistic networks where nodes exhibit group-specific interaction patterns.

\subsection{Model Assumptions}\label{assumptions}

The structure of the affinity matrix $D$---whether it is reducible or irreducible, periodic or aperiodic---determines the asymptotic growth and support of its powers $D^k$. In particular, the entry $[D^k]_{ij}$ admits a natural combinatorial interpretation, corresponding to the total weighted number of length-$k$ type sequences through which interactions originating from type $j$ can influence type $i$. When $[D^k]_{ij} = \ell > 0$, this indicates that there exist $\ell$ distinct admissible walks of length $k$ in the type-interaction graph from type $j$ to type $i$, counted with multiplicity according to the weights of $D$. Consequently, nodes of type $i$ receive contributions from nodes of type $j$ through $\ell$ propagation channels after $k$ steps.

\begin{assumptionEnv}\label{assumption1}
The affinity matrix $D$ is primitive.
\end{assumptionEnv}

We begin with the case where $D$ is \emph{primitive}, i.e., irreducible and aperiodic, and subsequently extend the results to the imprimitive (irreducible but periodic) and reducible settings. If $D$ is irreducible, the Perron--Frobenius theorem ensures that its spectral radius $\rho(D)$ is a positive and simple eigenvalue $\lambda_1$ associated with strictly positive left and right eigenvectors. If, in addition, $D$ is primitive, then the convergence of powers of $D$ toward the Perron eigendirection is uniform and non-oscillatory. In particular, \([D^k]_{ij} = \Theta(\lambda_1^k)\) \cite[Theorem 8.5.1]{horn2012matrix} for sufficiently large $k$ for all type pairs $(i,j)$.

\begin{assumptionEnv}\label{assumption2}
The affinity matrix $D$ is uniformly supercritical; that is, there exists a constant $\epsilon > 0$ such that $\lambda_1(D) \geq 1 + \epsilon$ for all $n$, where $\lambda_1(D)$ denotes the Perron-Frobenius eigenvalue of $D$.
\end{assumptionEnv}

In this regime, where $\lambda_1 > 1$ is strictly bounded away from the critical threshold, we ensure the existence of a unique giant component with high probability. This condition guarantees the exponential growth of the entries of the matrix power $[D^k]_{ij}$. Specifically, since $\lambda_1 \geq 1 + \epsilon$, the expected number of walks of length $k$ between any pair of types grows as $\Theta(\lambda_1^k)$, which remains non-vanishing for $k = \Theta(\log n)$. Conversely, if $\lambda_1 \leq 1$, the expected number of walks decays or grows at most sub-exponentially, precluding macroscopic connectivity. This aligns with the phase transition thresholds established in inhomogeneous random graph theory \cite{bollobas2007phase} (Theorems 3.1 and 3.12), where the emergence of a giant component is guaranteed in the strictly supercritical regime $\rho(D) > 1$.

\begin{assumptionEnv}\label{assumption4}
There exist constants $\alpha_t > 0$ such that $\sum_{t=1}^T \alpha_t = 1$ and the number of nodes of type $t$, denoted $n_t$, satisfies ${n_t}/{n} \to \alpha_t$.
\end{assumptionEnv}

Assumption \ref{assumption4} ensures that each type represents a non-vanishing fraction of the total node population. Because each type contains a linear number of nodes, edge counts concentrate sharply around their expectations, with the number of potential edges between any two types $p$ and $c$ scaling as $\Theta(n^2)$. Consequently, the expected degree contributions $d_{pc}$ serve as consistent, first-order descriptors of pairwise connectivity. This structure guarantees sufficient regularity for the multi-type branching process approximation, as local exploration processes are not distorted by vanishingly small or volatile type populations.

Moreover, because the number of types $T$ is fixed, the empirical type distribution $\nu_n = (n_1/n, \dots, n_T/n)$ converges to a stable limiting distribution $\nu$ as $n \to \infty$. This stability ensures that the affinity matrix $D$ maintains a constant, finite dimension, allowing the macroscopic interaction structure of the graph to remain well-defined and analytically tractable in the asymptotic limit.

\section{Embedding Distortion in IHGs with Primitive Affinity Matrices}

Let $G\sim IHG(\vec{n},D)$. Our main results establish $(1\pm\varepsilon)$-distortion guarantees for landmark-based distance-preserving embeddings of $G$, under suitable scaling of the number of landmark sets and their sizes. The supporting neighborhood growth and intersection lemmas are developed in Section~\ref{growth}.

\subsection{$(1\pm\varepsilon)$-Approximation Guarantees} \label{main-res}

For integers $M, r > 1$, we sample $R$ landmark sets of sizes $M^0, M^1, \ldots, M^r$. Our main results demonstrate that by appropriately scaling $R$ and $r$ as functions of the graph size $n$ and the error tolerance $\varepsilon$, the lower bound estimator $\underline{d}(u_1,u_2)$ (Theorem~\ref{thm:lower-bound}) and the upper bound estimator $\bar{d}(u_1,u_2)$ (Theorem~\ref{thm:upper-bound}) simultaneously achieve $(1\pm\varepsilon)$ multiplicative accuracy w.h.p. This multi-scale sampling scheme guarantees a strictly tighter embedding dimension requirement than classical worst-case bounds.

\begin{theoremEnv}[Lower-Bound Distortion]\label{thm:lower-bound}
Let $u_1$ and $u_2$ be any two nodes chosen uniformly at random from the graph $G\sim IHG(\vec{n},D)$ that satisfies Assumptions~\ref{assumption1}-\ref{assumption4}. 
Let $\underline{d}(u_1,u_2)$ be the lower bound on the shortest distance $d(u_1,u_2)$ as defined in Section \ref{landmark-embeddings}. Let $\varepsilon \in (0,1)$ and $\varsigma>0$ be arbitrarily small.
Define
\begin{align*}
&(1) \;\theta \in (0,\varepsilon), \\
&(2) \;r = \left\lfloor \frac{\theta}{\log M}\log n \right\rfloor, \\
&(3) \;R = \Omega\left(n^{1-\theta-\min\left\{\frac{\varepsilon}{2},\varepsilon -\theta\right\}+\varsigma}\right).
\end{align*}
Then w.h.p., 
$$(1-\varepsilon) d(u_1,u_2) \leq \underline{d}(u_1,u_2)\leq d(u_1,u_2),$$
i.e. $\underline{d}(u_1,u_2)$ provides a $(1-\varepsilon)$-approximation of $d(u_1,u_2)$.
\end{theoremEnv}
\begin{theoremEnv}[Upper-Bound Distortion]\label{thm:upper-bound}
Let $u_1$ and $u_2$ be any two nodes chosen uniformly at random from the graph $G\sim IHG(\vec{n},D)$ that satisfies Assumptions~\ref{assumption1}-\ref{assumption4}. 
Let $\bar{d}(u_1,u_2)$ be the upper bound on the shortest distance $d(u_1,u_2)$ as defined in Section \ref{landmark-embeddings}. Let $\varepsilon \in (0,1)$ and $\varsigma>0$ be arbitrarily small. Define
\begin{align*}
&(1) \;\theta \in \left(0,\frac{1-\varepsilon}{2}\right), \\
&(2) \;r = \left\lfloor \frac{\theta}{\log M}\log n \right\rfloor, \\
&(3) \;R = \Omega\left(
n^{1-\varepsilon+\varsigma}\right).
\end{align*}
Then w.h.p.,
$$ d(u_1,u_2) \leq \bar{d}(u_1,u_2)\leq (1+\varepsilon) d(u_1,u_2),$$
i.e. $\bar{d}(u_1,u_2)$ provides a $(1+\varepsilon)$-approximation of $d(u_1,u_2)$.
\end{theoremEnv}

In the supercritical regime where $\lambda_1 > 1$, the graph $G$ contains a giant connected component, and two nodes $u_1$ and $u_2$ chosen independently and uniformly at random lie in the same connected component with probability bounded away from zero, and the probability of $u_1$ and $u_2$ not in the giant component is negligible as Theorems 3.1 and 3.12 from \cite{bollobas2007phase} imply
$$\PR(u_1 \leftrightarrow u_2 \text{ but } u_1,u_2\notin \mathcal{C}_{\sss (1)}\mid G) = \frac{1}{n^2}\sum_{i\geq 2} |\mathcal{C}_{\sss (i)}|^2 \leq \frac{|\mathcal{C}_{\sss (2)}|}{n} \xrightarrow{\sss \PR} 0,$$
where $\mathcal{C}_{\sss (i)}$ denotes the $i$-th largest connected component.
The key to both theorems is understanding how the neighborhoods of $u_1$ and $u_2$, conditionally on being in the giant component, grow and interact as a function of spectral properties of $D$. Here, we provide a proof sketch and defer the complete proofs to Sections \ref{pf:lower-bound} and \ref{pf:upper-bound}.

\paragraph{Proof Sketch.} 
For any node $u$ and positive integer $k$, let $N_k(u)$ denote the set of nodes at graph distance at most $k$ from $u$ and $\partial N_k(u)$ denote the set of nodes at distance exactly $k$. Let $V_t$ denote nodes of type $t$ and so $n_t=|V_t|$. We further write $\partial N_k(u)_t = \partial N_k(u) \cap V_t$ and $N_k(u)_t = N_k(u) \cap V_t$. We also denote $\overrightarrow{|\partial N_k(u)|}=\left( |\partial N_k(u)_1|, |\partial N_k(u)_2|, \dots, |\partial N_k(u)_T| \right)^\top$ as the vector of type-specific boundary sizes and $t(u)$ as the type of node $u$. Let $e_t$ be the standard basis vector in $\mathbb{R}^T$.

The following result shows that conditionally on $u_1$ and $u_2$ being in the giant component, neighborhoods expand exponentially at rate $\lambda_1$ at radius $L=\Theta(\log n)$ and continue to grow exponentially for an additional $k = O(\log n)$ generations w.h.p. Specifically, letting $\mathcal{E}_{n,k}$ denote the event that the neighborhood sizes $\overrightarrow{|\partial N_{L+k}(u_i)|}$ lie within a $(1\pm\varepsilon)$ multiplicative window around the deterministic prediction $e_{t(u_i)}^\top D^{L+k}$ for $i\in \{1,2\}$, we have the following:

\begin{propositionEnv}\label{prop:neighborhood-growth}
Let $u_1$ and $u_2$ be any two nodes in the giant component of $G\sim IHG(\vec{n},D)$ that satisfies Assumptions~\ref{assumption1}-\ref{assumption4}. Then $\cap_{l=0}^{k}\mathcal{E}_{n,l}$ 
occurs w.h.p.\ for any $L+k < \log_{\lambda_1} n$.
\end{propositionEnv}
\begin{proof}
See Section \ref{pf:neighborhood-growth-prop}.
\end{proof}

\begin{figure*}[ht!]
\centering
\includegraphics[height=9cm]{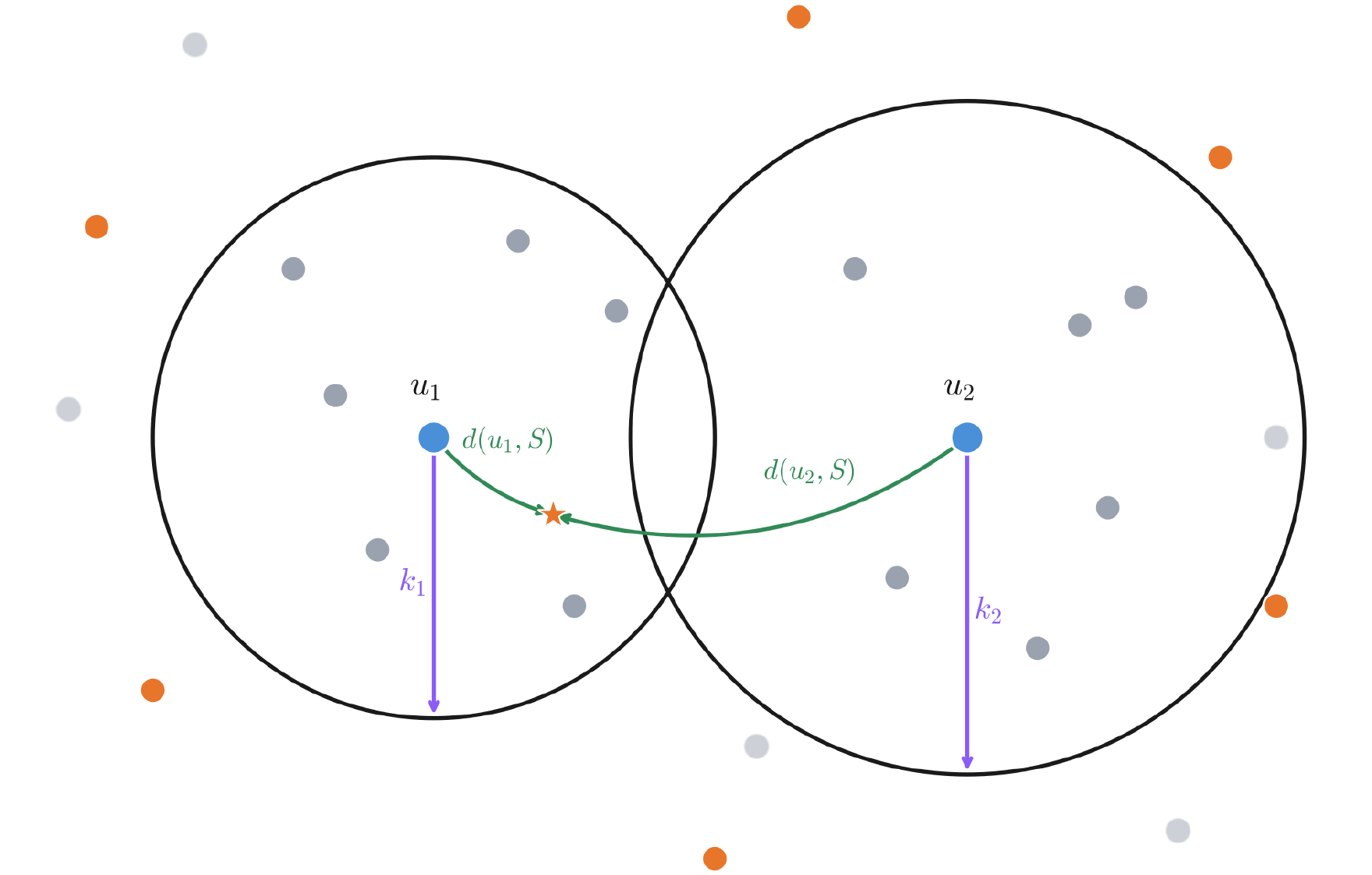}
\caption{Schematic depicting the computation of the lower bound $\underline{d}(u_1,u_2)$, where $k_2 - k_1 \geq (1-\varepsilon)d(u_1,u_2)$. Blue nodes are the source $u_1$ and target $u_2$, orange nodes are landmarks in set $S$, and gray nodes are arbitrary nodes.} 
\label{fig:examples_lb}
\end{figure*}

Since $D$ is primitive, $[D^k]_{ij} = \Theta(\lambda_1^k)$ for any type pair $(i,j)$, so Proposition~\ref{prop:neighborhood-growth} implies that neighborhoods grow exponentially at rate $\lambda_1^k$ for $1 \ll k \leq \log_{\lambda_1} n$. Hence, the local step of the landmark-based approximation, w.h.p., selects a landmark set $S$ that intersects $N_{k_1}(u_1)$ but not the disjoint $N_{k_2}(u_2)$, where $k_2 - k_1 \geq (1-\varepsilon)d(u_1,u_2)$. This yields the lower bound on distortion in Theorem~\ref{thm:lower-bound} as
\begin{equation*}
\underline{d}(u_1,u_2) = d(u_2,S)-d(u_1,S)\geq k_2 - k_1 \geq (1-\varepsilon)d(u_1,u_2) \quad \text{w.h.p.}
\end{equation*}
See Figure \ref{fig:examples_lb} for an illustration.

For Theorem~\ref{thm:upper-bound}, controlling the upper bound requires a landmark to fall in the \emph{intersection} $N_k(u_1) \cap N_k(u_2)$ for $k \leq \frac{1+\varepsilon}{2}d(u_1,u_2)$. This requires a finer analysis of how the two neighborhoods overlap, given by the following result:

\begin{propositionEnv}\label{prop:intersection-growth}
Let $u_1$ and $u_2$ be any two nodes in the giant component of $G\sim IHG(\vec{n},D)$ that satisfies Assumptions~\ref{assumption1}-\ref{assumption4}. Let $\varepsilon\in (0,1)$, $\kappa_0\in (0,1)$, $\kappa\in (0,1-\kappa_0)$, $L=\kappa_0\log_{\lambda_1} n$, and arbitrarily small $\rho > 0$. Then for any $t\in [T]$ and $L < k_1, k_2 \leq (\kappa_0+\kappa)\log_{\lambda_1} n$ satisfying 
$k_1 + k_2 \geq (1+\rho)\log_{\lambda_1} n_t$,
\begin{align*}
|\partial N_{k_1}(u_1)_t \cap \partial N_{k_2}(u_2)_t|
\in \left[
\frac{(1-\varepsilon)^3 c^2 \lambda_1^{k_1+k_2}}{2n_t},
\frac{(1+\varepsilon)^3 C^2 \lambda_1^{k_1+k_2}}{n_t}
\right] \quad \text{w.h.p.}
\end{align*}
for some constants $c, C > 0$.
\end{propositionEnv}
\begin{proof}
See Section \ref{pf:intersection-growth}.
\end{proof}

\begin{figure*}[ht!]
\centering
\includegraphics[height=9cm]{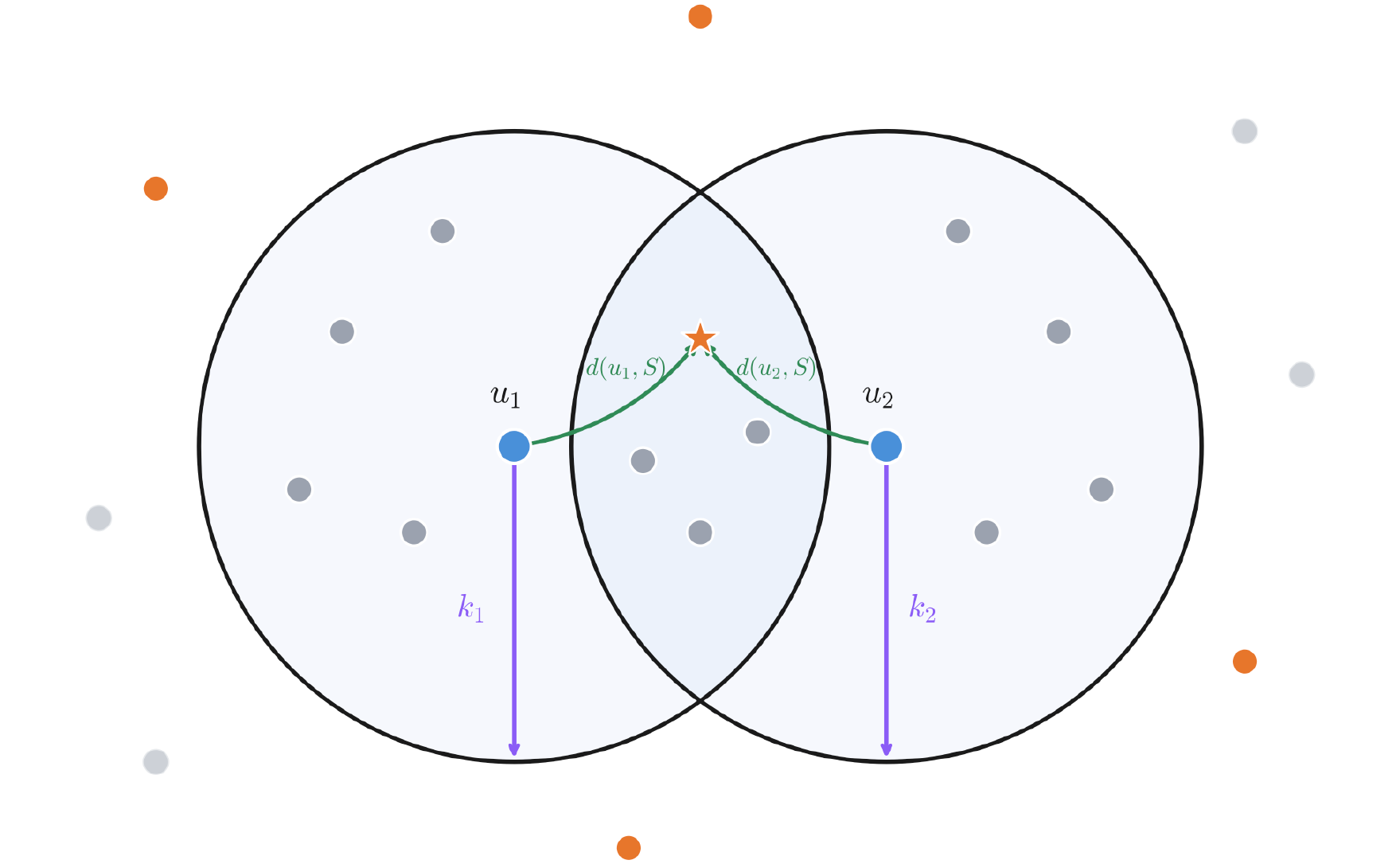}
\caption{Schematic depicting the computation of the upper bound $\bar{d}(u_1,u_2)$, where $k_1 = k_2 = \tfrac{1+\varepsilon}{2}d(u_1,u_2)$. Blue nodes are the source $u_1$ and target $u_2$, orange nodes are landmarks in set $S$, and gray nodes are arbitrary nodes.} 
\label{fig:examples_ub}
\end{figure*}

Proposition~\ref{prop:intersection-growth} shows that once $k_1 + k_2$ exceeds $\log_{\lambda_1} n_t$, the intersection $\partial N_{k_1}(u_1)_t \cap \partial N_{k_2}(u_2)_t$ is non-trivial and grows as ${\lambda_1^{k_1+k_2}}/{n_t}$ w.h.p. In other words, once neighborhoods grow sufficiently large and exceed a logarithmic threshold, overlaps between different neighborhoods become predictable and scale proportionally to the product of their exponential growth factors, normalized by the type size. This ensures that, w.h.p., the local step selects a landmark set $S$ that intersects $N_k(u_1) \cap N_k(u_2)$ but not the exclusive parts of their neighborhoods for $k \leq \tfrac{1+\varepsilon}{2} d(u_1,u_2)$, yielding the upper bound on distortion:
\begin{equation*}
\overline{d}(u_1,u_2) = d(u_1,S)+d(u_2,S) \leq k+k \leq (1+\varepsilon)\,d(u_1,u_2) \quad \text{w.h.p.}
\end{equation*}
See Figure \ref{fig:examples_ub} for an illustration.

The structural regularity of IHGs allows for a strictly lower embedding dimension than what is required under classical, worst-case metric spaces. We provide a comparative analysis of the embedding dimension requirements in Remark \ref{rmk}:

\begin{remark}\label{rmk} 
The embedding dimension requirement decreases from the worst-case bounds of $\Omega\left(n^{\frac{2(1-\varepsilon)}{2-\varepsilon}}\log n\right)$ for $(1-\varepsilon)$-distortion (Theorem \ref{thm:bourgain_distortion}) and $\Omega\left(n^{\frac{2}{2+\varepsilon}}\log n\right)$ for $(1+\varepsilon)$-distortion (Theorem \ref{thm:sarma_distortion}) to $\Omega\left(n^{1-\varepsilon}\log n\right)$ within IHGs. While the dimension requirement for $(1+\varepsilon)$-distortion is direct from Theorem \ref{thm:upper-bound}, the corresponding result for $(1-\varepsilon)$-distortion in Theorem \ref{thm:lower-bound} requires $\theta \geq \frac{\varepsilon}{2}$ to minimize the dimensional footprint to $\Omega\left(n^{1-\theta-\frac{\varepsilon}{2}}\log n\right) \leq \Omega\left(n^{1-\varepsilon}\log n\right)$.
\end{remark}

\subsection{Main Results: Average-Case Metric Distortions}\label{subsec:average-case}

The performance bounds established in Theorems \ref{thm:lower-bound} and \ref{thm:upper-bound} provide strong, point-wise guarantees for pairs of nodes selected uniformly at random across the graph space. However, in many practical networking and machine learning applications, such as matrix factorization or graph neural network embedding calibrations, one is primarily concerned with the aggregate behavior of the metric estimator across all viable pairs simultaneously. To address this, we extend our analysis from point-wise high-probability bounds to global, component-wide distortion averages. The following theorem demonstrates that the landmark embedding framework successfully stabilizes metric distortion across the entire topology, ensuring that anomalous structural bottlenecks do not corrupt the global average accuracy.

\begin{theoremEnv}[Average-Case Distortions]\label{avg-case}
    Let $\varepsilon>0$ and $U=\{(u_1,u_2):u_1\leftrightarrow u_2\}$ denote the set of ordered node pairs belonging to the same connected component. 
    \begin{enumerate}
        \item Under the conditions as in Theorem \ref{thm:lower-bound}, the average lower-bound distortion satisfies:
        \[
        \frac{1}{|U|}\sum_{u_1\leftrightarrow u_2}\frac{\underline{d}(u_1,u_2)}{d(u_1,u_2)}\geq 1-\varepsilon \quad \text{w.h.p.}
        \]
        \item Under the conditions as in Theorem \ref{thm:upper-bound}, the average upper-bound distortion satisfies:
        \[
        \frac{1}{|U|}\sum_{u_1\leftrightarrow u_2}\frac{\bar{d}(u_1,u_2)}{d(u_1,u_2)}\leq 1+\varepsilon \quad \text{w.h.p.}
        \]
    \end{enumerate}
\end{theoremEnv}

\begin{proof}
See Section \ref{pf-avg-case}.
\end{proof}

Theorem \ref{avg-case} highlights the global metric regularity of IHGs under our multi-scale landmark framework. While worst-case metric spaces are prone to localized geometric pathologies—where small, isolated clusters can arbitrarily inflate average distortion metrics—the branching process dynamics driving the IHG topology fundamentally bound these distortions. By decoupling the summation over valid pairs into successful approximation sets and highly localized failure configurations, the proof establishes that any pathologically behaving node pairs are statistically negligible.

Furthermore, this global convergence behavior carries significant implications for the implementation of downstream graph learning algorithms. Because the average distortion remains tightly controlled within a $(1\pm\varepsilon)$ window w.h.p., any empirical risks or loss functions that rely on pairwise distance approximations will remain stable. This guarantees that metric embeddings optimized using this hierarchical landmark strategy will preserve global structural properties. Crucially, it achieves this without requiring dense landmark sampling, avoiding the computational overhead typically required to patch localized estimation errors in more adversarial graph geometries.

\subsection{Supporting Lemmas on Neighborhood Growth} \label{growth}

The proofs of Theorems~\ref{thm:lower-bound} and~\ref{thm:upper-bound} rely on a precise characterization of how node neighborhoods expand in $G \sim IHG(\vec{n}, D)$. This subsection develops supporting lemmas establishing that neighborhoods grow exponentially at rate $\lambda_1$. The first such result is an upper bound on the expected neighborhood growth.

\begin{lemmaEnv}\label{lem:neighborhood-upper-bound}
For every node $u$ in $G$ and every type $t\in[T]$, 
\begin{equation*}
    \E(|\partial N_k(u)_{t}|)\leq \|D\|_2^{k} \quad \text{and} \quad \E(|N_k(u)_t|) = \sum_{l=0}^k \|D\|_2^{l}.
\end{equation*}
Furthermore, under Assumptions \ref{assumption1} and \ref{assumption2} for $k=\Theta(\log n)$,
\begin{equation*}
    \E(|\partial N_k(u)_{t}|) = O(\lambda_1^{k}) \quad \text{and} \quad \E(|N_k(u)_t|) = O(\lambda_1^{k}).
\end{equation*}
\end{lemmaEnv}

\begin{proof}
See Section \ref{pf:neighborhood-upper-bound}.
\end{proof}

Lemma \ref{lem:neighborhood-upper-bound} states that neighborhood growth in $IHG(\vec{n},D)$ is controlled by the spectral norm of the affinity matrix $D$. In expectation, the breadth-first exploration process behaves like a multi-type branching process with mean offspring matrix $D$, and therefore the size of the $k$-th layer grows at most exponentially at rate $\|D\|_2$. Consequently, the expected size of the $k$-hop neighborhood is bounded by a geometric series whose growth rate is determined entirely by the operator norm of $D$. In particular, when $\|D\|_2 < 1$, the model stays in a subcritical regime where neighborhoods remain small and exploration dies out quickly. Under Assumption~\ref{assumption2} ($\|D\|_2\geq \lambda_1 > 1$), the model enters a supercritical regime characterized by exponential expansion. This entails that a positive fraction of nodes belong to a giant connected component, while the remaining components are typically small \resultcite{bollobas2007phase}{Theorems 3.1 and 3.12}.

Furthermore, under Assumption \ref{assumption1} where $D$ is primitive, $[D^k]_{ij} = \Theta(\lambda_1^k)$ for any type pair $(i,j)$. As a result, we achieve a tighter bound on the growth rate with $\E(|\partial N_k(u)_{t}|) = \Theta(\lambda_1^{k})$ and $\E(|N_k(u)_t|) = O(\lambda_1^{k})$.
Building on this exponential growth result, we establish exponential growth of neighborhoods from fixed starting nodes, conditionally on node being in the giant component.
The following lemma shows that this conditioning event is asymptotically equivalent to the local event that both nodes' neighborhoods exhibit the exponential growth up to generation $L=\Theta(\log n)$.
This equivalence is what allows us to work conditionally on connectivity throughout the analysis.

\begin{lemmaEnv}\label{lem:equivalent-events}
Let $u_1$ and $u_2$ be any two nodes in $G\sim IHG(\vec{n},D)$ that satisfies Assumptions \ref{assumption2}-\ref{assumption4}. Let $\varepsilon\in(0,1)$, $\kappa_0\in (0,1)$, $L=\kappa_0\log_{\lambda_1} n$, and $t(u)$ denote the type of node $u$. Define the events
\begin{align*}
    A_n &= \{|\partial N_{L}(u_i)_t| \in \left[(1-\varepsilon)[D^L]_{t(u_i)t},(1+\varepsilon) [D^L]_{t(u_i)t}\right]\::\:i=1,2\:;\: t\in[T]\}, \\
    B_n &= \{\text{$u_1$ and $u_2$ are in the giant component}\}.
\end{align*}
Then $\PR(A_n \setminus B_n) \to 0$ and $\PR(B_n \setminus A_n) \to 0$ as $n\to \infty$.
\end{lemmaEnv}

\begin{proof}
See Section \ref{pf:equivalent-events}.
\end{proof}

Intuitively, since the local exploration is well approximated by a branching process, if $u_1$ and $u_2$ lie in the giant component, then the corresponding branching processes survive. As branching processes either die out or grow exponentially, the neighborhoods grow exponentially conditioned on survival.
Thus, local exponential expansion and global connectivity occur simultaneously in the large-$n$ limit. Once the neighborhood size at generation $L$ is well-approximated by branching process, subsequent generations continue to grow in a stable and multiplicative manner, as described by the following lemma:

\begin{lemmaEnv}\label{lem:neighborhood-growth}
Let $u_1$ and $u_2$ be any two nodes in the giant component of $G\sim IHG(\vec{n},D)$ that satisfies Assumptions \ref{assumption1}-\ref{assumption4}. Let $\varepsilon\in (0,1)$, $\kappa_0\in (0,1)$, $\kappa\in (0,1-\kappa_0)$, and $L=\kappa_0\log_{\lambda_1} n$.
Define 
\begin{align*}
A_{b^{m},b^{M}} &= \left\{\overrightarrow{|\partial N_L(u_i)|}\in [b_i^{m},b_i^{M}]: i=1,2\right\},\\
\mathcal{E}_{n,k} &= \left\{\overrightarrow{|\partial N_{L+k}(u_i)|} \in [b_i^{m}D^k, b_i^{M}D^k]: i=1,2\right\},
\end{align*}
where $b_i^{m}=(1-\varepsilon) e_{t(u_i)}^\top D^L$ and $b_i^{M}=(1+\varepsilon) e_{t(u_i)}^\top D^L$.
Then there exists $\delta>0$ such that
\[
\PR\!\left(\bigcap_{l=0}^{k}\mathcal{E}_{n,l} \middle| A_{b^{m},b^{M}}\right)
\geq \left(1-(2T+1)n^{-\delta}\right)^{k}
\]
for any $k \leq \kappa \log_{\lambda_1} n$, for all sufficiently large $n$.
\end{lemmaEnv}

\begin{proof}
See Section \ref{pf:neighborhood-growth}.
\end{proof}

Conditional on the event $A_{b^{m},b^{M}}$, which ensures that the $L$-th layer lies within a controlled multiplicative window around $D^L$, Lemma~\ref{lem:neighborhood-growth} establishes that the neighborhood sizes remain close to the deterministic trajectory $D^k$ for an additional $k = O(\log n)$ generations w.h.p. This result shows that, in the supercritical regime, neighborhood expansion not only initiates exponentially but persists in a stable, predictable fashion over logarithmic scales that closely tracks the multi-type branching process with mean matrix $D$, reinforcing the connection between local expansion dynamics and the global structure of the graph.

\section{Extension to Continuous Type Spaces (Kernel Model)}\label{kernel}

In the finite model, the geometry of the graph is governed by the spectral structure of the affinity matrix $D$, which controls the rate of neighborhood expansion and the probability that independently growing neighborhoods intersect. These mechanisms are precisely what drive the $(1\pm\varepsilon)$-distortion guarantees established in Section~\ref{main-res}. Many real-world networks, however, are more accurately described by continuously varying node attributes rather than a finite collection of discrete classes, and the analysis developed for finitely many node types extends naturally to a continuous latent-space setting.

To capture this setting, we consider an inhomogeneous random graph model \citep{bollobas2007phase} in which each node $i$ is assigned a latent position
\(x_i \in \mathcal{X},\)
sampled independently from a probability space $(\mathcal{X},\mu)$. Connectivity is determined by a measurable kernel
\(\kappa \in L^2(\mathcal{X}\times\mathcal{X}, \mu\times\mu)\),
where $\kappa(x,y)$ quantifies the affinity between positions $x$ and $y$. Conditional on the latent positions, edges are generated independently with probability
\(P_{ij}
=
\min\left\{
{\kappa(x_i,x_j)}/{n},
1
\right\}.\)

This formulation generalizes the finite-type model: if $\mathcal{X}$ is partitioned into finitely many regions and $\kappa$ is constant on each block, then the model reduces to the previously studied discrete affinity-matrix framework. By working within the $L^2$ space that allows the kernel to possess localized singularities, our framework serves as a unified continuum limit for a vast taxonomy of network architectures---encompassing heavy-tailed scale-free networks via unbounded Chung-Lu kernels \citep{chung2002average} and continuous variants of the stochastic block model with fluid community boundaries \citep{holland1983stochastic}.

\subsection{The Integral Operator}

In the finite-type setting, the affinity matrix $D$ acts on vectors whose coordinates encode type-specific neighborhood densities. Iterating $D$ describes how neighborhoods expand through successive graph distances, while the leading eigenvalue determines the asymptotic growth rate of this exploration process.

In the kernel setting, the matrix $D$ is replaced by the integral operator
\[
(\mathcal{T}_\kappa f)(x)
=
\int_{\mathcal{X}} \kappa(x,y)f(y)\,d\mu(y),
\]
where $\mu$ is the probability measure describing the distribution of latent positions in $\mathcal{X}$, and \(f\in L^2(\mathcal{X},\mu)\) is a square-integrable function representing the density of reachable mass or influence across the latent space. Thus, $\mathcal{T}_\kappa$ plays exactly the same role as matrix multiplication in the finite-dimensional setting: repeated application of the operator models the propagation and expansion of neighborhoods across the graph, effectively describing a multitype branching process in a continuous state space.

The spectral radius \(\rho(\mathcal{T}_\kappa)\), equivalently the principal eigenvalue $\lambda_1$, governs the large-scale geometry of the graph. When \(\lambda_1>1\), the graph is supercritical and contains a giant connected component w.h.p. \citep[Theorem 3.1]{bollobas2007phase}. Moreover, neighborhood volumes grow asymptotically at exponential rate $\lambda_1$, implying that typical graph distances scale as $\log_{\lambda_1}n$ \citep[Theorem 3.14]{bollobas2007phase}. Consequently, the same spectral mechanism underlying the finite-type distortion bounds continues to govern metric behavior in the continuous model.

\subsection{The Sandwiching Argument}

A key technical step in extending these results is the approximation of the continuous kernel by finite-dimensional block models. For every $\delta > 0$, there exists a measurable partition
\(\mathcal{X} = \bigsqcup_{t=1}^{T} \mathcal{X}_t\) 
such that $\kappa$ admits a step-function approximation $\kappa_\delta$ satisfying
\(\|\kappa - \kappa_\delta\|_{\infty} < \delta.\)
This approximating kernel $\kappa_\delta$ is constant on each rectangle $\mathcal{X}_i \times \mathcal{X}_j$, effectively inducing a finite-type random graph model by leveraging the fact that step functions are dense in the space of kernels under the relevant metric \citep{lovasz2012large}.

As the partition is refined, the associated operator $\mathcal{T}_{\kappa_\delta}$ converges to $\mathcal{T}_{\kappa}$ in operator norm. Because the spectral radius is a continuous functional of the operator for compact kernels \citep{dunford1963spectral}, this convergence implies that the leading eigenvalues and the resulting exponential neighborhood growth rates of the approximation approach those of the continuous kernel. By ``sandwiching'' the true kernel between two such step functions, $\kappa_\delta^{-}$ and $\kappa_\delta^{+}$, we can bound the true shortest-path distances $d_{\kappa}$ between those of the finite-type models. 
The lower approximation $\kappa_\delta^{-}$ generates a sparser graph, so graph distances in the corresponding model are typically larger ($d_{\kappa_\delta^-}(u_1,u_2)\ge d_\kappa(u_1,u_2)$).
Conversely, the upper approximation $\kappa_\delta^{+}$ produces a denser graph with shorter distances ($d_{\kappa_\delta^+}(u_1,u_2)\le d_\kappa(u_1,u_2)$).
These approximations create a geometric ``sandwich'' around the true metric structure.

Since both bounding models are finite-type graphs, we will show that the distortion results from Section \ref{main-res} apply directly to them. As the partition becomes finer (i.e. the number of types $T\to \infty$),
\(\|\mathcal{T}_{\kappa_\delta^\pm}-\mathcal{T}_\kappa\|
\to 0\) \citep[Chapter 2]{krasnosel1976integral}, and therefore \(\lambda_1(\mathcal{T}_{\kappa_\delta^\pm}) \to \lambda_1(\mathcal{T}_\kappa)\) \citep[Theorem 3.16]{kato1966perturbation}. Because shortest-path distances scale logarithmically with the effective branching factor ${d(u_1,u_2)}/{\log n}
\overset{p}{\longrightarrow}
{1}/{\log \lambda_1}$ \citep[Theorem 6.2]{RGCN2}, the geometric behavior of the finite approximations converges to that of the original kernel model. We formalize a rigorous probability space coupling and spectral perturbation framework for arbitrary bounded kernels in the following theorem. %This continuity argument transfers the finite-type $(1 \pm \varepsilon)$-distortion guarantees to arbitrary bounded kernels, as formalized in the following theorem.

\begin{theoremEnv}[Metric Sandwiching for Kernel Models]
\label{thm:metric-sandwich}

Let $(\mathcal{X},\mu)$ be a probability space and $\kappa : \mathcal{X} \times \mathcal{X} \to [0,\infty)$ be a bounded, primitive kernel with integral operator spectral radius $\lambda_1(\mathcal{T}_\kappa) > 1$. For any $\varepsilon \in (0,1)$, there exists a $\delta > 0$ and a finite measurable partition $\mathcal{P}_\delta = \{\mathcal{X}_1, \dots, \mathcal{X}_T\}$ of $\mathcal{X}$ defining bounded, primitive step-function kernels $\kappa_\delta^{-}$ and $\kappa_\delta^{+}$ such that $\|\kappa - \kappa_\delta^{\pm}\|_\infty \le \delta$. 

Let $\mathcal{G}_n = (G_{\kappa_\delta^{-}}, G_\kappa, G_{\kappa_\delta^{+}})$ be a joint coupling of random graphs constructed on a shared sequence of latent positions $x_1, \dots, x_n \sim \mu$ and independent edge variables $U_{ij} \sim \mathrm{Uniform}(0,1)$, where an edge exists in $G_f$ if $U_{ij} \le f(x_i, x_j)/n$. Then the following properties hold simultaneously:

\begin{enumerate}
    \item \textbf{Edge-Set Inclusion:} The respective edge sets are monotonically nested on the same probability space, satisfying:
    \[ E(G_{\kappa_\delta^{-}}) \subseteq E(G_\kappa) \subseteq E(G_{\kappa_\delta^{+}}) \quad \text{w.h.p.}\]
    \item \textbf{Spectral Radius Perturbation:} The spectral radii of the corresponding step-kernel operators $\lambda_1(\mathcal{T}_{\kappa_\delta^{-}})$ and $\lambda_1(\mathcal{T}_{\kappa_\delta^{+}})$ satisfy a linear perturbation bound around the true kernel spectral radius:
    $$
    1 < \lambda_1(\mathcal{T}_{\kappa_\delta^{-}}) \leq \min\{\lambda_1(\mathcal{T}_{\kappa_\delta^{+}}),\lambda_1(\mathcal{T}_\kappa)\}\quad \text{with} \quad \lambda_1(\mathcal{T}_{\kappa_\delta^{\pm}}) = \lambda_1(\mathcal{T}_\kappa) \pm O(\delta).
    $$

\end{enumerate}
\end{theoremEnv}

\begin{proof}
See Section \ref{pf:sandwich}.
\end{proof}

This framework serves as a rigorous ``bridge'' between finite-type affinity matrices and continuous latent-space kernels. The result shows that, for supercritical graphs with $\lambda_1 > 1$, the asymptotic structural behavior is a continuous functional of the underlying kernel. Although individual edge modifications are discrete and potentially unstable, the large-$n$ geometric behavior is governed by the integral operator $\mathcal{T}_\kappa$. Consequently, continuous graphons can be treated as limits of finite block models while preserving the underlying random graph mechanics. More generally, the theorem shows that the structural framework of an inhomogeneous random graph is spectrally robust as small perturbations of $\kappa$ in the $L^2$ norm induce only controllable, continuous variations in the network's spectral radius.

From a proof-theoretical standpoint, this theorem provides a rigorous reduction argument. By ``sandwiching'' the true kernel between two finite step-function approximations, we transfer the analysis of infinite-dimensional operators to the more tractable linear algebra of finite affinity matrices. Because the exact edge inclusions and spectral bounds hold for the bounding models $\kappa_\delta^\pm$, and since these coupled systems can be made arbitrarily tight relative to the true kernel operator $\mathcal{T}_\kappa$, the fundamental structural and topological properties of the continuous model are fully constrained.

\subsection{Main Result: Universal Kernel Distortion}

By combining the finite-type distortion guarantees with the metric sandwiching framework, we can now lift our analysis from discrete blocks to continuous latent spaces. The following theorem formalizes this transition, proving that the $(1\pm\varepsilon)$-distortion bounds established in Theorems~\ref{thm:lower-bound} and \ref{thm:upper-bound} generalize universally to graphs generated by arbitrary square-integrable kernels, accommodating both bounded topologies and heavy-tailed, unbounded architectures.

\begin{theoremEnv}[Universal Kernel Distortion]
\label{thm:universal-distortion}

Let $G_\kappa$ be an inhomogeneous random graph generated by a primitive kernel $\kappa \in L^2(\mathcal{X} \times \mathcal{X}, \mu \times \mu)$ with spectral radius $\lambda_1>1$. Let $\varepsilon\in(0,1)$. Suppose the landmark-based estimators $\underline{d}$ and $\overline{d}$ are constructed as defined in Section~\ref{landmark-embeddings}. Then for any two vertices $u_1,u_2$ chosen uniformly at random from the graph $G_\kappa$, we have the following:
\begin{enumerate}
\item If parameters $(\theta,r,R)$ satisfy the conditions in Theorem~\ref{thm:lower-bound},
\begin{equation*}
(1-\varepsilon)\, d_\kappa(u_1,u_2)
\le
\underline{d}(u_1,u_2)
\le
d_\kappa(u_1,u_2) \quad \text{w.h.p.}
\end{equation*}
\item If parameters $(\theta,r,R)$ satisfy the conditions in Theorem \ref{thm:upper-bound},
\begin{equation*}
d_\kappa(u_1,u_2) 
\le
\overline{d}(u_1,u_2)
\le
(1+\varepsilon)\, d_\kappa(u_1,u_2) \quad \text{w.h.p.}
\end{equation*}
\end{enumerate}
Consequently, the landmark-based embedding provides a $(1\pm\varepsilon)$-approximation of the true shortest-path distance for arbitrary $L^2$ kernels.

\end{theoremEnv}

\begin{proof}
See Section \ref{pf:universal-distortion}.
\end{proof}

Theorem \ref{thm:universal-distortion} establishes that the tight $(1\pm\varepsilon)$-distortion guarantees derived for finite-type affinity matrices are not structural artifacts of discrete block models, but instead reflect a fundamental metric property of general inhomogeneous random graphs. By abstracting node attributes into a continuous latent space $(\mathcal{X}, \mu)$, the theorem extends these guarantees to real-world networks in which connectivity is driven by smooth, continuous traits (such as geographical positions, semantic text embeddings, or dense social proximity vectors) as well as heavy-tailed degree distributions governed by unbounded power-law configurations, rather than rigid category assignments.

Through a metric sandwiching argument, we show that when the underlying kernel $\kappa$ is square-integrable and primitive, the continuous graph metric behaves as a stable limit of finite-dimensional constructions governed strictly by the structural parameters $(\theta, r, R)$. This ensures that the computational benefits of landmark-based approaches (namely, avoiding explicit all-pairs shortest path calculations) extend naturally to non-parametric settings. 

Algorithmically, this highlights the robustness of our multiscale sampling scheme. Even though continuous neighborhoods expand along smoothly varying frontiers rather than uniform block boundaries, the contraction of the propagation operator $\mathcal{T}_\kappa$ toward its principal eigendirection preserves the geometric intersection behavior seen in the discrete model. Ultimately, because kernel perturbations and boundary fluctuations induce only controlled, continuous distortions in the operator spectrum, Theorem \ref{thm:universal-distortion} validates distance-preserving landmark embeddings as a robust and scalable framework for structural representation across a broad class of network models.

\section{GNN-based Landmark Embeddings and Experimental Results} \label{sec:numerical}

\cite{Sarma2010ASD} proposed using one BFS to calculate the shortest path distance between every node and each landmark set in the local step given in Section~\ref{landmark-embeddings}, requiring $D$ BFS runs to generate landmark embeddings in $\reals^D$. This is prohibitive for large graphs, with complexity $O(n+m)$ per source and $O(n(n+m))$ for all pairs \citep{cormen2009introduction}. We propose replacing BFS in the local step with a GNN, which approximates landmark distances from graph structure. Once trained, GNN inference is computationally efficient, and crucially GNNs trained on small graphs can be transferred to larger ones, leveraging the transferability properties of graph neural networks on convergent graph sequences \citep{ruiz20-transf, ruiz2021transferability}. This last property is directly motivated by the IHG framework; since our theoretical results characterize distance geometry in terms of the spectral properties of $D$, graphs from the same IHG model share structural regularities that a GNN can learn and exploit across sizes.

Formally, GNNs are deep convolutional architectures tailored to graph data \citep{scarselli2008graph, kipf17-classifgcnn, defferrard17-cnngraphs, ruiz2020gnns}. Focusing on node-level data represented as $\bbX \in \reals^{n \times d}$, each GNN layer applies a graph convolution followed by a pointwise nonlinearity,
\begin{equation*}
    \bbX_\ell = \sigma \left( \sum_{k=0}^{K-1} \bbA^k \bbX_{\ell-1} \bbW_{\ell,k}\right),
\end{equation*}
where $\bbA \in \reals^{n \times n}$ is the graph adjacency, $\bbW_{\ell,k} \in \reals^{d_{\ell-1} \times d_\ell}$ are learnable parameters, and $\sigma$ is a pointwise nonlinearity. The full GNN is written compactly as $\bbY = \Phi(\bbX, \bbA; \ccalW)$. A key property inherited from graph convolutions is locality: each layer exchanges information only within one-hop neighborhoods, so $L$-layer GNNs aggregate information within $L$-hop neighborhoods. This aligns naturally with the landmark-based embedding task, where the relevant quantity---the distance from a node to a landmark---is determined by local neighborhood structure up to the appropriate radius.

\subsection*{Experimental Setup}\label{app:exp_details}

We train GNNs to approximate landmark distances in sparse, undirected, unweighted random graphs. As the canonical $T=1$ special case of the IHG model, we use Erd\H{o}s--R\'enyi graphs $\text{ER}_n(\lambda/n)$, where each pair of nodes is connected independently with probability $\lambda/n$. Setting $1 < \lambda \ll n$ ensures sparsity and the existence of a giant component w.h.p., placing the model firmly in the supercritical regime of our theory. We consider $\lambda \in \{3,4,5,6\}$ and $n \in \{25, 50, 100, 200, 400, 800, 1600, 3200\}$.

We evaluate four standard GNN architectures---GCN \citep{kipf17-classifgcnn}, GraphSAGE \citep{NIPS2017_5dd9db5e}, GAT \citep{Velickovic18-GraphAttentionNetworks}, and GIN \citep{xu2018how}---all using sum aggregation, dropout, and ReLU activations. For each architecture, we evaluate nine models with $\lfloor\sqrt{n}\rfloor$ nodes in the first and last layers and varying hidden-layer depth and width.

Each graph is treated as a batch of nodes with a 200-50-50 train-validation-test split. Input signals $\bbX \in \reals^{n \times r}$ one-hot encode landmark nodes, and outputs $\bbY \in \reals^{n \times r}$ represent shortest path distances $[\bbY]_{us} = d(u,s)$. Training runs for 1000 epochs with early stopping (100 epochs), MSE loss, Adam optimizer (lr=0.01, weight decay=0.0001), and a cyclic-cosine learning rate schedule (0.001--0.1 for 10 cycles).

\subsection{Experiment 1: Learning the GNNs}\label{exp1app}

In the first experiment, we evaluate the ability of trained GNNs to compute end-to-end shortest paths. We consider $n=50$ and set the GNN depth to be larger than $\lceil \log_\lambda n \rceil$. Figure~\ref{fig:exp1app} plots the actual shortest path distances versus those predicted by our selected GNN architectures. Predictions for distances beyond the GNN depth saturate, indicating that GNNs cannot capture longer distances even with depth exceeding the expected path length. As expected, GNNs are not suitable for computing end-to-end shortest path distances, especially on sparser graphs with $\lambda \in \{3,4\}$, which tend to exhibit longer paths.

\begin{figure*}[h]
\centering
\includegraphics[width=0.49\textwidth]{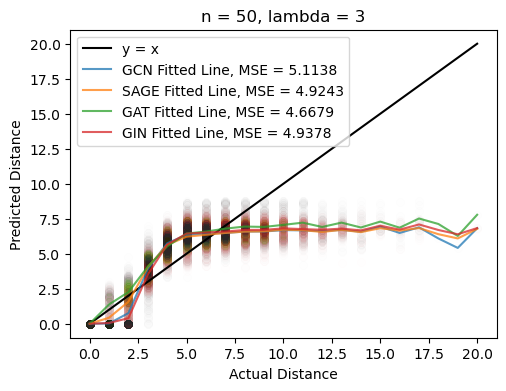}
\includegraphics[width=0.49\textwidth]{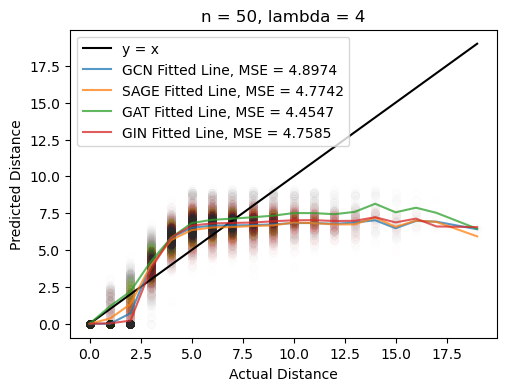}\\
\:\includegraphics[width=0.49\textwidth]{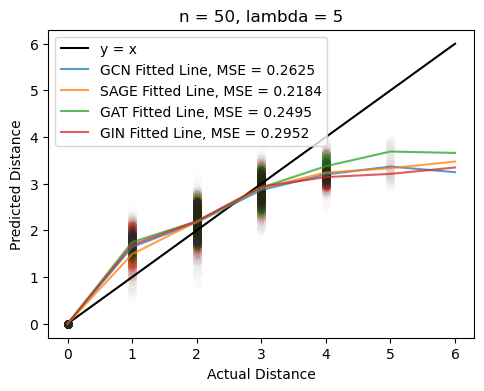}\:
\includegraphics[width=0.49\textwidth]{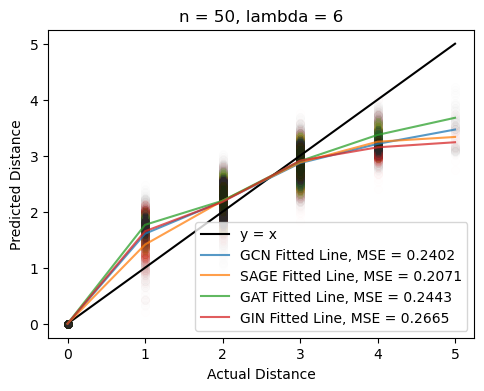}
\caption{End-to-end shortest path distance predictions from $\lfloor \sqrt{n}\rfloor\text{-64-32-16-}\lfloor\sqrt{n}\rfloor$ GNNs trained on graphs generated by $\ER(\lambda/n)$. The evaluation data consists of graphs from the same model.}
\label{fig:exp1app}
\end{figure*}

\subsection{Experiment 2: Comparing BFS-Based and GNN-Based Landmark Embeddings}\label{exp2app}

In this experiment, we compare the lower bounds resulting from BFS-based and GNN-based landmark embeddings against the actual shortest path distances. We focus on lower bounds since they depend only on coordinate-wise differences and are thus a clean measure of embedding quality independent of the landmark index matching required for upper bounds.
%\textit{Only lower bounds $\underline{d}(u_1,u_2)$ are compared to ensure a fair evaluation, as computing the upper bounds $\bar{d}(u_1,u_2)$ requires storing additional information—namely, the indices of the closest landmarks from the landmark sets to each node. Moreover, unlike in lower bound computations, the saturation effect inherent in GNNs cannot be mitigated in upper bound computations, making the upper bound an unreliable metric for shortest path approximation when calculated upon GNN-based landmark distances.}

To construct the landmark embeddings, we sample $r+1$ landmark sets $S_0, S_1, \dots, S_r$ of cardinalities $2^0, 2^1, \dots, 2^r$ with $r = \lfloor \log n \rfloor$ for $R$ repetitions. In Figure~\ref{fig:exp2}(a-d), GNN-based lower bounds underperform the vanilla lower bounds for smaller $\lambda \in \{3,4\}$, but yield substantial improvements for larger $\lambda \in \{5,6\}$ across all three tested values of $R$. Although both $\lambda$ values are in the supercritical regime ($\lambda > 1$), several factors explain this difference. As shown in Figure~\ref{fig:exp1app}, the GNN learns poorer landmark embeddings for $\lambda\in \{3,4\}$, even on small 50-node graphs. Additionally, for large $n$, graphs are almost surely connected when $\lambda\in \{5,6\}$ but not when $\lambda\in \{3,4\}$. Finally, Figure~\ref{fig:exp2}(e) illustrates that GNN-based embeddings can be generated faster than BFS-based embeddings, particularly on large graphs as exact local embedding computations via BFS scale poorly with graph size.

\begin{figure*}[h]
\centering
\includegraphics[width=1\textwidth]{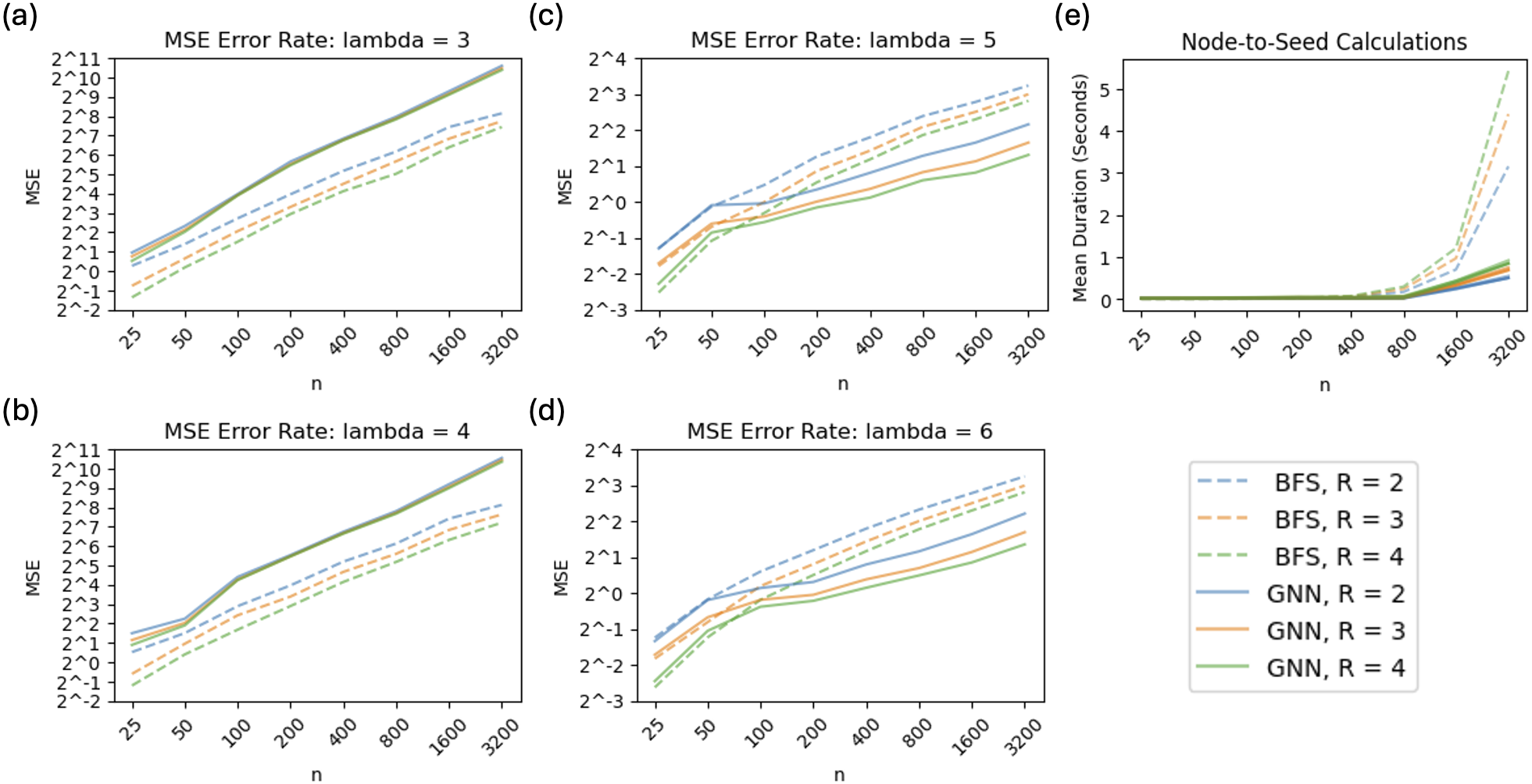}
\caption{(a)-(d) Error rates of BFS-based and GNN-based lower bounds on graphs generated by $\ER(\lambda/n)$, with the GNNs trained on graphs from the same model. (e) Time required to generate all node-to-landmark distances in $n$-node ER graphs by NetworkX's highly optimized BFS compared to our widest and deepest GNNs. All GCN, GraphSage, GAT, and GIN models are represented by the same color and solid lines for the same $R$, and the deviations between them are insignificant.}
\label{fig:exp2}
\end{figure*}

\subsection{Experiment 3: Transferability} \label{exp3}

In our last experiment, we investigate whether GNNs trained on small graphs can be transferred to compute landmark embeddings on larger networks for downstream shortest path approximation via LBs. This is motivated by \newcite{ruiz20-transf} and \newcite{ruiz2021transferability}, which show that GNNs are transferable as their outputs converge on convergent graph sequences. This, in turn, allows models trained on smaller graphs to generalize to similar larger graphs.

Here, we focus on $\lambda \in \{5,6\}$ and train a sequence of eight GNNs on ER graphs ranging from $n=25$ to $n=3200$ nodes. These GNNs are then used to generate local node embeddings on graphs from the ER model with the same $\lambda$ and $n' = 12800$ nodes. Figure~\ref{fig:exp3}(a,d) shows the MSE for each instance as the training graph size increases, with flat dashed lines indicating the MSE of BFS-based LBs on the $n'$-node graph. We observe a steady decrease in MSE as $n$ grows, with GNN-based embeddings matching BFS-based performance when GNNs are trained on graphs of $n=100$, which is 128 times smaller than the target graph.

\begin{figure*}[h]
\centering
\includegraphics[width=1\textwidth]{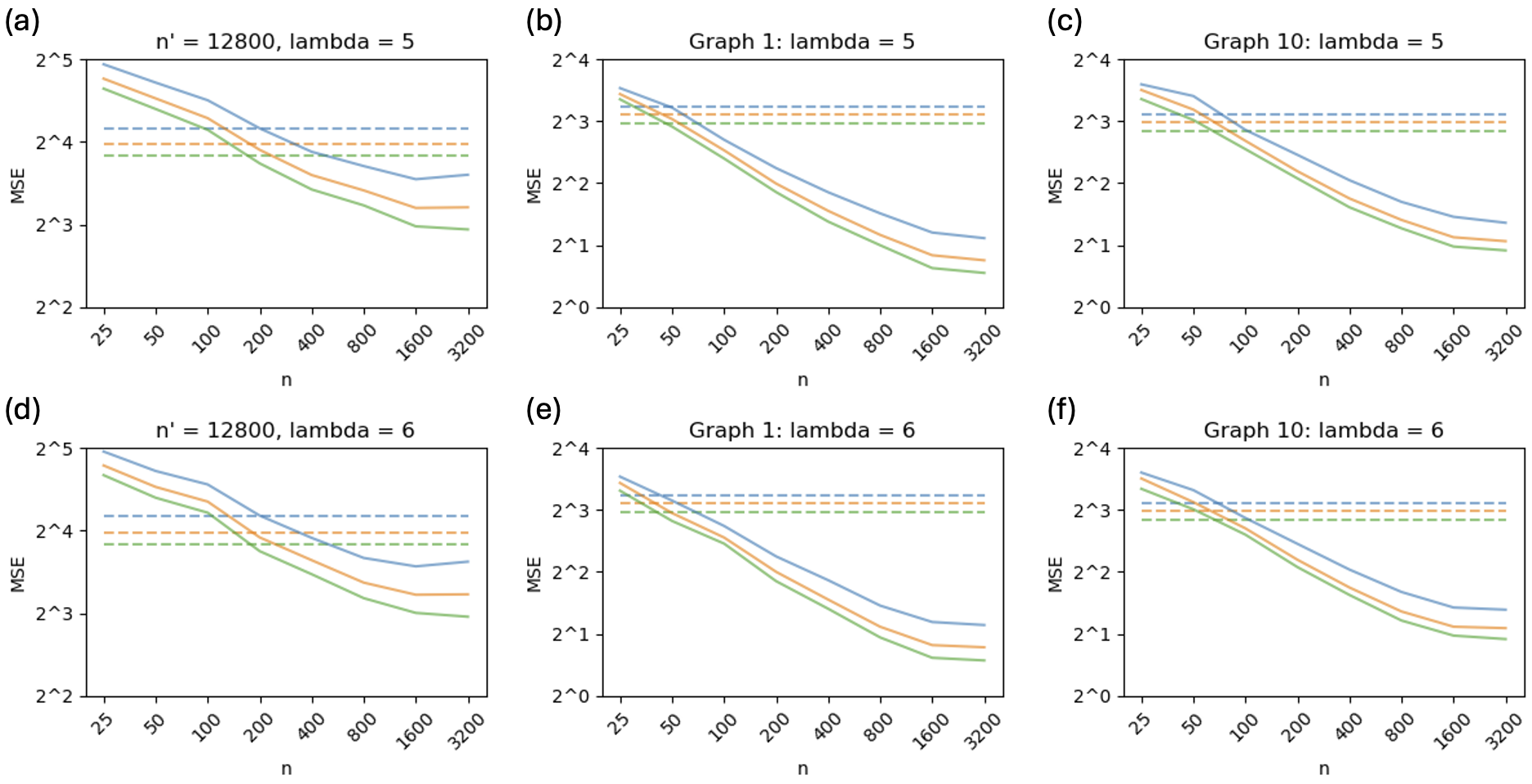}
\caption{Error rates of BFS-based and GNN-based lower bounds on (a,d) test Erd\H{o}s--R\'enyi graphs generated by $\text{ER}_{n'}(\lambda/n')$, (b,e) Arxiv COND-MAT collaboration network with 21,364 nodes, and (c,f) GEMSEC company network with 14,113 nodes, with the GNNs trained on graphs from $\ER(\lambda/n)$. Legend is the same as in Figure \ref{fig:exp2}.}
\label{fig:exp3}
\end{figure*}

\begin{table}[h]
\caption{Details on the largest connected component of selected benchmark networks.}
\centering
\resizebox{\textwidth}{!}{
\begin{tabular}{|c|c|c|c|c|} 
\hline
\# & Name & Category & \# of Nodes & \# of Edges \\
\hline
1 & Arxiv COND-MAT \citep{10.1145/1217299.1217301} & Collaboration Network & 21,364 & 91,315 \\
2 & Arxiv GR-QC \citep{10.1145/1217299.1217301} & Collaboration Network & 4,158 & 13,425 \\
3 & Arxiv HEP-PH \citep{10.1145/1217299.1217301} & Collaboration Network & 11,204 & 117,634 \\
4 & Arxiv HEP-TH \citep{10.1145/1217299.1217301} & Collaboration Network & 8,638 & 24,817 \\
5 & Oregon Autonomous System 1 \citep{10.1145/1081870.1081893} & Autonomous System & 11,174 & 23,409 \\
6 & Oregon Autonomous System 2 \citep{10.1145/1081870.1081893} & Autonomous System & 11,461 & 32,730 \\
7 & GEMSEC Athletes \citep{rozemberczki2019gemsec} & Social Network & 13,866 & 86,858 \\
8 & GEMSEC Public Figures \citep{rozemberczki2019gemsec} & Social Network & 11,565 & 67,114 \\
9 & GEMSEC Politicians \citep{rozemberczki2019gemsec} & Social Network & 5,908 & 41,729 \\
10 & GEMSEC Companies \citep{rozemberczki2019gemsec} & Social Network & 14,113 & 52,310 \\
11 & GEMSEC TV Shows \citep{rozemberczki2019gemsec} & Social Network & 3,892 & 17,262 \\
12 & Twitch-EN \citep{rozemberczki2019multiscale} & Social Network & 7,126 & 35,324 \\
13 & Deezer Europe \citep{feather} & Social Network & 28,281 & 92,752 \\
14 & LastFM Asia \citep{feather} & Social Network & 7,624 & 27,806 \\
15 & Brightkite \citep{nr} & Social Network & 56,739 & 212,945\\
16 & ER-AVGDEG10-100K-L2 \citep{nr} & Labeled Network & 99,997 & 499,359 \\
\hline
\end{tabular}\label{tab}
}
\end{table}

\begin{figure*}
\centering
\includegraphics[width=0.32\textwidth]{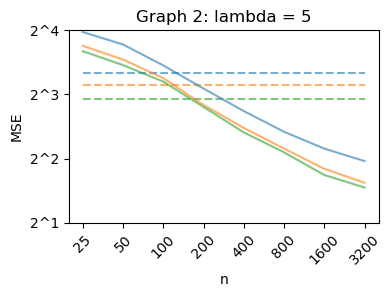}
\includegraphics[width=0.32\textwidth]{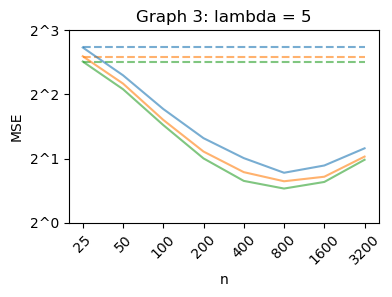}
\includegraphics[width=0.32\textwidth]{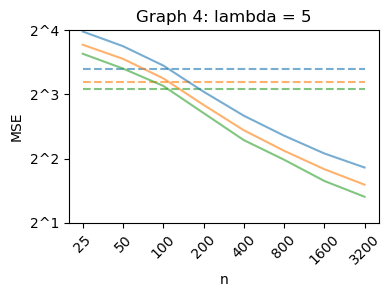}\\
\includegraphics[width=0.32\textwidth]{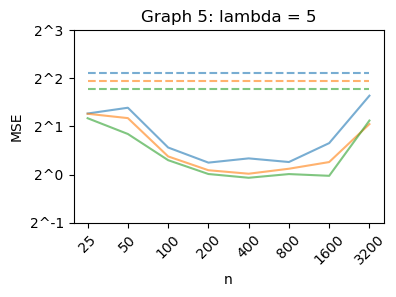}
\includegraphics[width=0.32\textwidth]{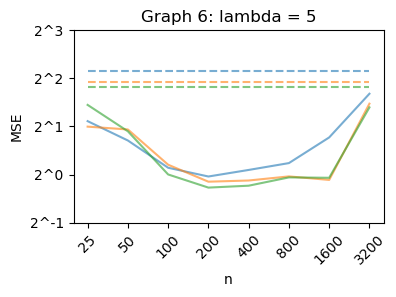}
\includegraphics[width=0.32\textwidth]{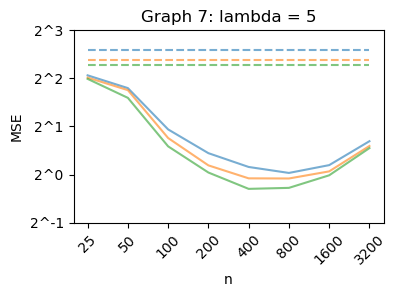}\\
\includegraphics[width=0.32\textwidth]{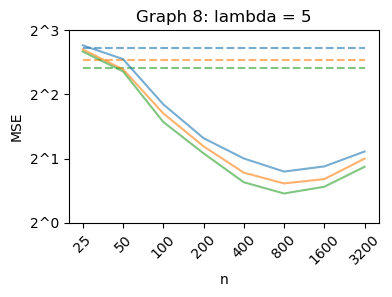}
\includegraphics[width=0.32\textwidth]{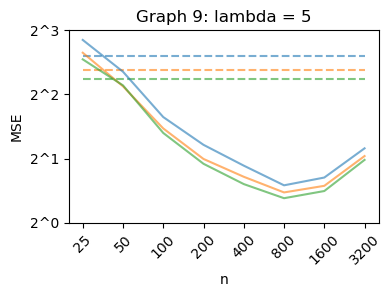}
\includegraphics[width=0.32\textwidth]{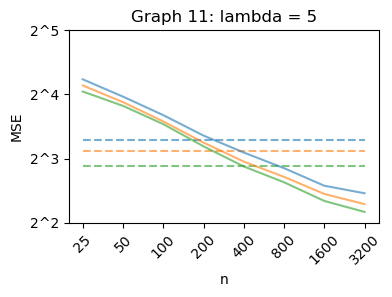}\\
\includegraphics[width=0.32\textwidth]{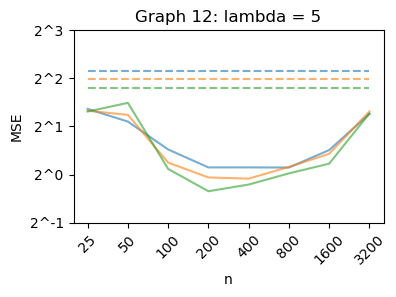}
\includegraphics[width=0.32\textwidth]{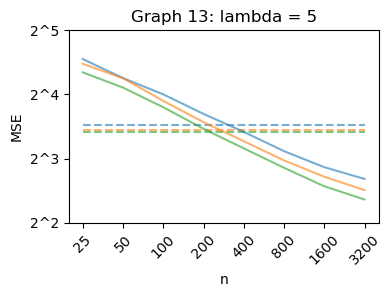}
\includegraphics[width=0.32\textwidth]{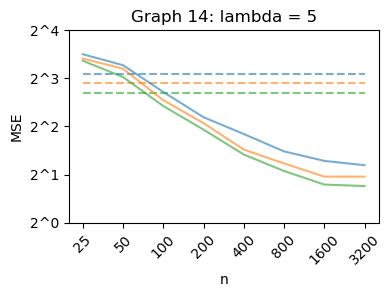}
\includegraphics[width=0.32\textwidth]{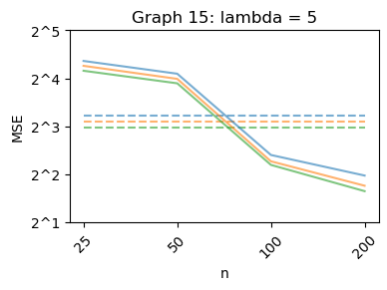}
\includegraphics[width=0.32\textwidth]{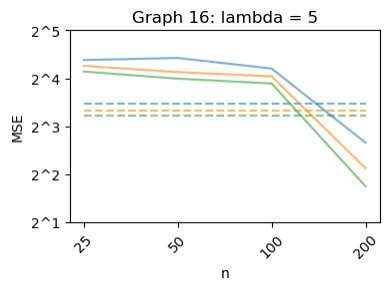}
\caption{Additional transferability results on real networks, with the GNNs trained on graphs from $\ER(\lambda/n)$. Legend is the same as in Figure \ref{fig:exp2}.}
\label{fig:exp3real}
\end{figure*}

When examining the transferability of the same set of GNNs on sixteen real-world networks listed in Table~\ref{tab}, we again observe that MSE improves with training graph size and that GNN-based lower bounds outperform BFS-based lower bounds, even though the landmark embeddings are learned on much smaller graphs (see Figures~\ref{fig:exp3} and \ref{fig:exp3real}). This can be explained as random graphs can model real-world networks in certain scenarios, and networks with similar sparsity likely exhibit similar local structures which local message-passing in GNNs can learn with sufficient training.

All experiments use PyTorch Geometric \citep{fey2019fast} on a Lambda Vector 1 machine (AMD Ryzen Threadripper PRO 5955WX CPU, 16 cores, 128 GB RAM, 2× NVIDIA RTX 4090 GPUs, no parallel training). Code is available at \url{https://github.com/ruiz-lab/shortest-path}.

\section{Proofs}
This section contains the proofs of all theoretical results in the paper.

\subsection{Proof of Lemma \ref{lem:neighborhood-upper-bound}}\label{pf:neighborhood-upper-bound}

Let
$$D
=\begin{bmatrix}
 d_{11} &  d_{12} & \cdots &  d_{1T} \\
 d_{21} &  d_{22} & \vdots &  d_{2T} \\
\vdots & \vdots & \ddots & \vdots \\
 d_{T1} &  d_{T2} & \cdots &  d_{TT}
\end{bmatrix}
=\begin{bmatrix}
v_1 \\
v_2 \\
\vdots \\
v_{T}
\end{bmatrix}
=\begin{bmatrix}
v'_1 & v'_2 & \cdots & v'_{T}
\end{bmatrix}.$$
By construction of $G$, we have $\partial N_k(u)=\cup_{t=1}^{T} \partial N_k(u)_t$ and $\partial N_k(u)_t\cap \partial N_{k'}(u)_{t'}=\emptyset$ for any distances $k\neq k'$ or any types $t\neq t'$. Let $I_{xy}$ be the indicator random variable for the edge $\{x,y\}$ being present. Then for $k\geq 1$,
\begin{align*}
    \E(|\partial N_k(u)_t|\mid N_{k-1}(u))&\leq\E\left(\sum_{t'=1}^{T}\sum_{x\in\partial N_{k-1}(u)_{t'}}\sum_{\substack{y \text{ is of type }t\\y\notin N_{k-1}(u)}}I_{xy}\mid N_{k-1}(u)\right)\\
    &= \sum_{t'=1}^{T} |\partial N_{k-1}(u)_{t'}|\left(n_t-\sum_{l=0}^{k-1}|\partial N_l(u)_t|\right)\frac{ d_{t't}}{n_t}.
\end{align*}
It follows that
\begin{align*}
    \E(|\partial N_k(u)_{t_k}|)&=\E(\E(|\partial N_k(u)_{t_k}|\mid N_{k-1}(u))) \\
    &\leq\E\left(\sum_{t_{k-1}=1}^{T} |\partial N_{k-1}(u)_{t_{k-1}}|\left(n_{t_k}-\sum_{l=0}^{k-1}|\partial N_l(u)_{t_k}|\right)\frac{ d_{t_{k-1}t_k}}{n_{t_k}}\right)\\
    &\leq \sum_{t_{k-1}=1}^{T}  d_{t_{k-1}t_k} \E(|\partial N_{k-1}(u)_{t_{k-1}}|)\\
    &\leq \sum_{t_{k-1}=1}^{T}  d_{t_{k-1}t_k} \sum_{t_{k-2}=1}^{T}  d_{t_{k-2}t_{k-1}} \E(|\partial N_{k-2}(u)_{t_{k-2}}|)\\
    &\leq \cdots \leq \sum_{t_{k-1}=1}^{T}  d_{t_{k-1}t_k} \sum_{t_{k-2}=1}^{T}  d_{t_{k-2}t_{k-1}}\cdots \sum_{t_1=1}^{T}  d_{t_1 t_2} \E(|\partial N_1(u)_{t_1}|).\\
\end{align*}
Since
\begin{align*}
    \E(|\partial N_1(u)_{t_1}|) =
    \begin{cases}
    1 \cdot n_{t_1} \frac{ d_{t_0 t_1}}{n_{t_1}} & \text{if } t_1 \neq t_0 \\
    1 \cdot (n_{t_1} - 1) \frac{ d_{t_0 t_1}}{n_{t_1}} & \text{otherwise}
    \end{cases}
    \leq  d_{t_0 t_1},
\end{align*}
we have for all $k\geq 2$ that
\begin{align}
    \E(|\partial N_k(u)_{t_k}|) &\leq \sum_{t_{k-1}=1}^{T}  d_{t_{k-1}t_k} \sum_{t_{k-2}=1}^{T}  d_{t_{k-2}t_{k-1}}\cdots \sum_{t_1=1}^{T}  d_{t_1 t_2}  d_{t_0 t_1} \notag\\
    &= \sum_{t_{k-1}=1}^{T}  d_{t_{k-1}t_k} \sum_{t_{k-2}=1}^{T}  d_{t_{k-2}t_{k-1}}\cdots \sum_{t_2=1}^{T}  d_{t_2 t_3} \left<v_{t_0},v'_{t_2}\right> \notag\\
    &= \sum_{t_{k-1}=1}^{T}  d_{t_{k-1}t_k} \sum_{t_{k-2}=1}^{T}  d_{t_{k-2}t_{k-1}}\cdots \sum_{t_3=1}^{T}  d_{t_3 t_4} \left<v_{t_0},\sum_{t_2=1}^{T}  d_{t_2 t_3} v'_{t_2}\right>
    \notag\\
    &= \sum_{t_{k-1}=1}^{T}  d_{t_{k-1}t_k} \sum_{t_{k-2}=1}^{T}  d_{t_{k-2}t_{k-1}}\cdots \sum_{t_3=1}^{T}  d_{t_3 t_4} \left< v_{t_0},D v'_{t_3}  \right>\notag\\
    &= \sum_{t_{k-1}=1}^{T}  d_{t_{k-1}t_k} \sum_{t_{k-2}=1}^{T}  d_{t_{k-2}t_{k-1}}\cdots \sum_{t_3=1}^{T}  d_{t_3 t_4} \left< D^\top v_{t_0} ,v'_{t_3} \right>\notag\\
    &= \sum_{t_{k-1}=1}^{T}  d_{t_{k-1}t_k} \sum_{t_{k-2}=1}^{T}  d_{t_{k-2}t_{k-1}}\cdots \sum_{t_4=1}^{T}  d_{t_4 t_5} \left< (D^\top)^2 v_{t_0} ,v'_{t_4} \right>\notag\\
    &= \cdots = \left< (D^\top)^{k-2} v_{t_0},v'_{t_k} \right>=\left< (D^\top)^{k-2} D^\top e_{t_0},De_{t_k} \right> = [D^k]_{t_0t_k}\notag
\end{align}
and
\begin{align*}
    \E(|N_k(&u)_{t}|)=\E(|\partial N_0(u)_{t}|)+\E(|\partial N_1(u)_{t}|)+\sum_{l=2}^k\E(|\partial N_l(u)_{t}|)\\
    &\leq 1 + d_{t_0t}+\sum_{l=2}^k \|D\|_2^{l} \leq 1 + \|D\|_2+\sum_{l=2}^k \|D\|_2^{l}.
\end{align*}

\noindent Under Assumptions \ref{assumption1} and \ref{assumption2} with $k=\Theta(\log n)$, there exists $C>0$ such that for all $k>K$,
\begin{align*}
    \E(|\partial N_k(u)_{t_k}|) &\leq \left< e_{t_0},D^ke_{t_k} \right> = [D^k]_{t_0t_k}\leq C\lambda_1^k
\end{align*}
and 
\begin{align*}
    \E(|N_k(u)_{t}&|)=\E(|\partial N_0(u)_{t}|)+\E(|\partial N_1(u)_{t}|)+\sum_{l=2}^k\E(|\partial N_l(u)_{t}|)\\
    &\leq 1 + d_{t_0t}+\sum_{l=2}^{K-1} \|D\|_2^{l} + \sum_{l=K}^k C\lambda_1^l = O(\lambda_1^k).
\end{align*}

\subsection{Proof of Lemma \ref{lem:equivalent-events}}\label{pf:equivalent-events}

%We start with proving $P(A_n \setminus B_n) \to 0$ and $P(B_n \setminus A_n) \to 0$, which is equivalent to $\mathbb{P}(A_n \triangle B_n)\to 0$, then use this result to prove $P(A_n)\to c$ for some $c>0$. First, we recall inhomogeneous random graph theory:

\noindent Recall inhomogeneous random graph theory:
\begin{enumerate}
\item[(1)] Since $\lambda_1 > 1$, the largest component satisfies $|\mathcal{C}_{(1)}| = \Theta(n)$ and all other components satisfy $|\mathcal{C}_{(i)}| = O(\log n)$ for $i>1$ w.h.p. \citep[Theorems 3.1 and 3.12]{bollobas2007phase}.
\item[(2)] Local neighborhood growth (starting from a fixed node) can be approximated by a multitype Poisson branching process with mean matrix $D$ \citep[Sections 4.1--4.2]{bollobas2007phase}.
\item[(3)] The local neighborhood explorations of two randomly chosen nodes can, w.h.p., be coupled to two independent copies of the multitype Poisson branching process up to a slowly growing logarithmic depth $k \leq \kappa \log_{\lambda_1} n$ with $\kappa \in (0, 1)$ \citep[Lemma 11.5]{bollobas2007phase}.
\item[(4)] Let $(Z^i_k)_{k\ge 0}$ be the branching process starting with one node of type $i$, and let $Z^i_k(j)$ denote the number of nodes of type $j$ at generation $k$. A Kesten–Stigum type theorem for a supercritical multitype branching process {\resultcite{grama2023kesten}{Theorem 2.7, Remark 2.8, and Corollary 2.9}} states that
\begin{itemize}
    \item On the event of survival/explosion, $Z^i_k(j)$ grows asymptotically like $[D^k]_{ij}$ {\resultcite{grama2023kesten}{Equations 1.2, 1.7, and 2.21}}.
    \item On the event of extinction, $Z^i_k(j) \to 0$ w.h.p. {\resultcite{grama2023kesten}{Equations 2.19 and 2.20}}.
\end{itemize}
\end{enumerate}

\noindent If \(B_n\) does not occur, then at least one of \(u_1\) or \(u_2\) must lie in \(\mathcal{C}_{(i)}\) for some \(i>1\). Since \(|\mathcal{C}_{(i)}| = O(\log n)\) for all \(i>1\) w.h.p.\ according to (1), \(A_n\) cannot occur as it explicitly requires both components containing \(u_1\) and \(u_2\) to have size at least \(n^{\Omega(1)}\) w.h.p. Since \(B_n^c \subseteq A_n^c\), we immediately have \(\mathbb{P}(A_n \setminus B_n) \to 0\) as $n \to \infty$.

\noindent Conversely, to establish that \(\mathbb{P}(B_n \setminus A_n) \to 0\), we look at the neighborhood exploration of $u_1$ and $u_2$ given that $B_n$ occurs. By (2) and (3), the local explorations can be jointly coupled to two branching processes until generation $k \leq \kappa \log_{\lambda_1} n$. Under this coupling, a node belongs to the giant component $\mathcal{C}_{(1)}$ if and only if its locally coupled branching process avoids early extinction and enters the survival/explosion regime. Because $B_n$ implies $u_1, u_2 \in \mathcal{C}_{(1)}$, it follows that both independent branching processes $\mathcal{Z}_1$ and $\mathcal{Z}_2$ survive w.h.p. By the Kesten--Stigum theorem (4), conditional on survival, the generation sizes of both processes grow exponentially at rate $\lambda_1^k$, tracking the matrix elements $[D^k]_{ij}$. This exponential growth directly fulfills the conditions defining event $A_n$. Consequently, conditional on $B_n$, event $A_n$ occurs w.h.p., meaning $\mathbb{P}(A_n \mid B_n) \to 1$, which proves $\mathbb{P}(B_n \setminus A_n) \to 0$.

\subsection{Proof of Lemma \ref{lem:neighborhood-growth}}\label{pf:neighborhood-growth}

From Markov's inequality and Lemma \ref{lem:neighborhood-upper-bound}, there exist $\delta>0$ for any $\gamma\in (\kappa_0+\kappa,1)$ such that 
$$\PR(| N_k(u_i)_t|\geq n^\gamma)\leq \frac{O(\lambda_1^{k})}{n^\gamma}\leq \frac{O\left(n^{\kappa_0+\kappa}\right)}{n^\gamma}\leq n^{-\delta}$$
for $i=1,2$ and $t\in [T]$ with $k\leq (\kappa_0+\kappa)\log_{\lambda_1} n$ with sufficiently large $n$. Then for each fixed $k\leq (\kappa_0+\kappa)\log_{\lambda_1} n$,
\begin{align}
\PR(| N_k(u_i)_t | &\leq n^\gamma \text{ for } i = 1,2 \text{ and } t\in [T] )\notag\\
&= 1-\PR(\exists i \in \{1,2\}, t \in [T] s.t.  \left| N_k(u_i)_t \right| \geq n^\gamma)\notag\\
&\geq 1- \sum_{i=1,2} \sum_{t=1}^{T} \PR(| N_k(u_i)_t|\geq n^\gamma)\geq 1-2T n^{-\delta}. \label{markov2}
\end{align}

\noindent Let $\delta_n=n^{-\beta}$ with $0<\beta<\frac{\kappa_0}{2}$. {The upper bound on $\beta$ is chosen so that the upper bound in \eqref{upper-chernoff} vanishes faster than $2T n^{-\delta}$ as $n^{\kappa_0 - 2\beta} \to \infty$, which explains the failure probability in \eqref{failure}. Since ${n_t}/{n} \geq \alpha$ for all $t \in [T]$ for some constant $\alpha > 0$ by Assumption \ref{assumption4}. Define
$$\cE'_{n,k}=\{\overrightarrow{|\partial N_{L+k}(u_i)|} \in [{b'}_i^{m}(1-\delta_n)^k(1-\alpha^{-1} n^{\gamma-1})^kD^k, {b'}_i^{M}(1+\delta_n)^k D^k]:i=1,2\}$$
where ${b'}_i^{m}=(1-\varepsilon') e_{t(u_i)}^\top D^L$ and ${b'}_i^{M}=(1+\varepsilon')e_{t(u_i)}^\top D^L$ with $0<\varepsilon'<\varepsilon$, which reflects tighter bounds than in $\cE_{n,k}$. Our next step is to bound $\E(\overrightarrow{|\partial N_{L+k}(u_i)|}\mid  N_{L+k-1}(u_i))$ and then use the Chernoff-Hoeffding bound \resultcite{dubhashi2009concentration}{Theorem 1.1} to prove that \( \cE'_{n,k} \) occurs w.h.p. given \( \cap_{l=0}^{k-1}\cE'_{n,l} \) and \( A_{b^m,b^M} \); that is, local expansions from two uniformly random nodes continue to grow exponentially up to the neighborhood radius \( (\kappa_0+\kappa)\log_{\lambda_1} n \), provided that the previous layers exhibit exponential growth.}

\noindent Let $I_{xy}$ be the indicator random variable for the edge $\{x,y\}$ being present. Let $[\partial N_{k}(u)]_{t't}$ be nodes of type $t$ not in $N_{k-1}(u)$ that have an edge with a node of type $t'$ in $\partial N_{k-1}(u)$. Conditionally on $A_{b^m,b^M}=\left\{\overrightarrow{|\partial N_L(u_i)|}\in [b_i^{m},b_i^{M}]: i=1,2\right\}$,
\begin{align*}
    \E(|\partial &N_{L+k}(u_i)_{t}|\mid  N_{L+k-1}(u_i)) \\
    &= \E\left(\sum_{\substack{x \notin N_{L+k-1}(u_i)\\ t(x)=t}}\1_{\{\exists t'\in [T] \: \exists y\in \partial N_{L+k-1}(u_i)_{t'}\::\:I_{xy}=1\}}\mid  N_{L+k-1}(u_i)\right)\\
    &=(n_t-|N_{L+k-1}(u_i)_{t}|)\left(1-\prod_{t'=1}^{T}\left(1-\frac{d_{t't}}{n_t}\right)^{|\partial N_{L+k-1}(u_i)_{t'}|}\right)
\end{align*}
Conditioning on $\cE'_{n,k-1}$ with $\lambda_1>1$, $\frac{d_{t't}}{n_t}\in [0,1]$, and $D$ being primitive, we have that
\begin{align*}
|\partial N_{L+k-1}(u_i)_{t'}|\frac{d_{t't}}{n_t}&\leq (1+\varepsilon')(1+\delta_n)^k[D^{L+k-1}]_{t(u_i)t'}\frac{d_{t't}}{\alpha n}\\
&\leq (1+\varepsilon')(1+n^{-\beta})^k C\lambda_1^{L+k-1} \frac{d_{t't}}{\alpha n}\\
&\leq (1+\varepsilon')(1+n^{-\beta})^k C d_{t't} \frac{n^{\kappa_0 + \kappa}}{\alpha n}
\end{align*}
for all $t'\in [T]$ and all $i=1,2$ for some constant $C>0$. Since $-\beta<0$ and $k\leq \kappa\log_{\lambda_1} n$, $(1+n^{-\beta})^k\to 1$ as $n\to \infty$. Since $\kappa_0 + \kappa<1$, $|\partial N_{L+k-1}(u_i)_{t'}|\frac{d_{t't}}{n_t}\leq (1+\varepsilon')(1+n^{-\beta})^k C d_{t't}\alpha \frac{n^{\kappa_0 + \kappa}}{\alpha n}$ vanishes, and so 

\begin{align*}
    \left(1 - \frac{d_{t't}}{n_t}\right)^{|\partial N_{L+k-1}(u_i)_{t'}|} = 1 - |\partial N_{L+k-1}(u_i)_{t'}|\frac{d_{t't}}{n_t}(1+o(1)),
\end{align*}
which implies
\begin{align}
    \E(|\partial &N_{L+k}(u_i)_{t}|\mid  N_{L+k-1}(u_i))\notag\\
    &=(n_t-|N_{L+k-1}(u_i)_{t}|)\left(1-\prod_{t'=1}^{T}\left(1-|\partial N_{L+k-1}(u_i)_{t'}|\frac{d_{t't}}{n_t}(1+o(1))\right)\right).\label{prod-error}
\end{align}
Since $T$ is finite, we have the identity
\begin{align}
    \prod_{t'=1}^{T} \left( 1 - A_{t'}(1 + o(1)) \right) = 1 - \left( \sum_{t'=1}^{T} A_{t'} \right) (1 + o(1)) \label{identity}
\end{align}
Applying this identity to \eqref{prod-error}, we obtain
\begin{align*}
    \E(|\partial &N_{L+k}(u_i)_{t}|\mid  N_{L+k-1}(u_i))\notag\\
    &=(n_t-|N_{L+k-1}(u_i)_{t}|)\left( \sum_{t'=1}^{T} |\partial N_{L+k-1}(u_i)_{t'}|\frac{d_{t't}}{n_t} \right) (1 + o(1))\\
    &=\left(1-\frac{|N_{L+k-1}(u_i)_{t}|}{n_t}\right)(1+o(1))\left\langle\overrightarrow{|\partial N_{L+k-1}(u_i)|},De_t\right\rangle.
\end{align*}
Conditionally on $\cap_{l=0}^{k-1}\cE'_{n,l}$ and $A_{b^m,b^M}$, from \eqref{markov2} we have
\begin{align*}
    {b'}_i^{m}(1-\delta_n)^{k-1}(1-\alpha^{-1} n^{\gamma-1})^{k}D^{k}\leq \E(\overrightarrow{|\partial N_{L+k}(u_i)|}\mid  N_{L+k-1}(u_i)) \leq {b'}_i^{M}(1+\delta_n)^{k-1} D^{k}
\end{align*}
with probability at least $1-2T n^{-\delta}$ since $1-\alpha^{-1} n^{\gamma-1}\leq 1-\frac{|N_{L+k-1}(u_i)_{t}|}{n_t}\leq 1$ for all $i=1,2$ and $t\in [T]$ with sufficiently large $n$. Denote this event $R_k$. Using $\PR(A)\leq \PR(A\mid B)+\PR(B^c)$, we obtain
\begin{align*}
    \PR({\cE'}_{n,k}^c &\mid \cap_{l=0}^{k-1}\cE'_{n,l}, A_{b^m,b^M}) \leq \PR({\cE'}_{n,k}^c \mid R_k, \cap_{l=0}^{k-1}\cE'_{n,l}, A_{b^m,b^M}) + 2T n^{-\delta}.
\end{align*}
Then by union bound and Chernoff-Hoeffding bound \resultcite{dubhashi2009concentration}{Theorem 1.1},
\begin{align*}
    &\PR({\cE'}_{n,k}^c \mid R_k, \cap_{l=0}^{k-1}\cE'_{n,l}, A_{b^m,b^M}) \\
    &\leq \sum_{i=1,2} \sum_{t=1}^{T} \PR(||\partial N_{L+k}(u_i)_t|-\E(|\partial N_{L+k}(u_i)_t|)|\geq \delta_n \E(|\partial N_{L+k}(u_i)_t|) \mid R_{k}, \cap_{l=0}^{k-1}\cE'_{n,l}, A_{b^m,b^M})\\
    &\leq \sum_{i=1,2} \sum_{t=1}^{T} 2\exp\left(-\frac{\delta_n^2}{3}\E(|\partial N_{L+k}(u_i)_t|\mid R_{k}, \cap_{l=0}^{k-1}\cE'_{n,l}, A_{b^m,b^M})\right)\\
    &\leq 4T\exp\left(-\frac{\delta_n^2}{3}{b'}_i^{m}(1-\delta_n)^{k-1}(1-\alpha^{-1} n ^{\gamma-1})^{k}D^{k}e_t\right)\\
    &=4T\exp\left(-\frac{\delta_n^2}{3}(1-\varepsilon')(1-n^{-\beta})^{k-1}(1-\alpha^{-1} n^{\gamma-1})^{k} [D^{L+k}]_{t(u_i)t}\right).
\end{align*}
Since $D$ is primitive, there exists $c>0$ such that $ [D^{L+k}]_{ij}\geq c\lambda_1^{L+k}\geq cn^{\kappa_0}$ for all pairs of types $(i,j)$ and so
\begin{align}
    \PR({\cE'}_{n,k}^c &\mid R_k, \cap_{l=0}^{k-1}\cE'_{n,l}, A_{b^m,b^M}) \notag\\
    &\leq 4T\exp\left(-\frac{n^{-2\beta}}{3}(1-\varepsilon')(1-n^{-\beta})^{k-1}(1-\alpha^{-1} n^{\gamma-1})^{k} cn^{\kappa_0}\right)\label{upper-chernoff}.    
\end{align}
Since $-\beta<0$ and $\gamma-1<0$ with $k\leq \kappa\log_{\lambda_1} n$, $(1-n^{-\beta})^{k-1}\to 1$ and $(1- \alpha^{-1} n^{\gamma-1})^k\to 1$ as $n\to \infty$. Also since $2\beta<\kappa_0$, $4T \exp\left(-\frac{n^{-2\beta}}{3}(1-\varepsilon')(1-n^{-\beta})^{k-1}(1-\alpha^{-1} n^{\gamma-1})^{k} cn^{\kappa_0}\right)$ vanishes faster than $2T n^{-\delta}$. Thus, 
\begin{align}
    \PR(\cE'_{n,k} \mid  \cap_{l=0}^{k-1}\cE'_{n,l}, A_{b^m,b^M})\geq 1-(2T+1)n^{-\delta}\label{failure}
\end{align}
for sufficiently large $n$. Then by induction,
\begin{align*}
    \PR(&\cap_{l=0}^{k}\cE'_{n,l}\mid A_{b^m,b^M})\\
    &=\PR(\cE'_{n,k} \mid  \cap_{l=0}^{k-1}\cE'_{n,l}, A_{b^m,b^M})\PR(\cE'_{n,k-1} \mid  \cap_{l=0}^{k-2}\cE'_{n,l}, A_{b^m,b^M})\dots \PR(\cE'_{n,0} \mid A_{b^m,b^M})\\
    &\geq (1-(2T+1)n^{-\delta})\cdot (1-(2T+1)n^{-\delta})\dots (1-(2T+1)n^{-\delta})\PR(\cE'_{n,0} \mid A_{b^m,b^M})\\
    &= (1-(2T+1)n^{-\delta})^{k}\PR(\cE'_{n,0} \mid A_{b^m,b^M}).
\end{align*}
Since $(1-\varepsilon')(1-\delta_n)^k(1-\alpha^{-1} n^{\gamma-1})^k\geq 1-\varepsilon$ and $(1+\varepsilon')(1+\delta_n)^k\leq 1+\varepsilon$ for sufficiently large $n$, ${b'}_i^{m}(1-\delta_n)^k(1-\alpha^{-1} n^{\gamma-1})^k\geq b_i^{m}$ and ${b'}_i^{M}(1+\delta_n)^k\leq b_i^{M}$. Therefore, $\cE'_{n,0}\subseteq A_{b^m,b^M}$ and $\cE'_{n,k}\subseteq \cE_{n,k}$ for all $k\geq 0$. Hence, $\PR(\cE'_{n,0} \mid A_{b^m,b^M})=1$ and so $$\PR(\cap_{l=0}^{k}\cE_{n,l}\mid A_{b^m,b^M})\geq \PR(\cap_{l=0}^{k}\cE'_{n,l}\mid A_{b^m,b^M})\geq (1-(2T+1)n^{-\delta})^{k}.$$

\subsection{Proof of Proposition \ref{prop:neighborhood-growth}}\label{pf:neighborhood-growth-prop}

Recall all the notation from Lemma \ref{lem:neighborhood-growth} and its proof. By Lemma \ref{lem:neighborhood-growth}, there exists $\delta>0$ such that $\PR(\cap_{l=0}^{k}\cE_{n,l} \mid A_n) \geq (1-(2T+1)n^{-\delta})^{k}\to 1$ for any $k\leq \kappa\log_{\lambda_1} n$ for all sufficiently large $n$. By Lemma \ref{lem:equivalent-events}, $\PR(A_n \setminus B_n) \to 0$ and $\PR(B_n \setminus A_n) \to 0$, so any event that holds w.h.p. under $A_n$ also holds w.h.p. under $B_n$.

\subsection{Proof of Proposition \ref{prop:intersection-growth}}\label{pf:intersection-growth}

{Similar to the proof of Lemma \ref{lem:neighborhood-growth}, this proof also consists of two main steps: first, we bound \( \E \big(|\partial N_{k_1}(u_1)_t \cap \partial N_{k_2}(u_2)_{t}| \big| N_{k_1}(u_1), N_{k_2-1}(u_2)\big) \); then we use the Chernoff--Hoeffding bound \resultcite{dubhashi2009concentration}{Theorem 1.1} to show that the intersection grows w.h.p. While the upper bound on \( \E \big(|\partial N_{k_1}(u_1)_t \cap \partial N_{k_2}(u_2)_{t}| \big| N_{k_1}(u_1), N_{k_2-1}(u_2)\big) \) is straightforward from the neighborhood growth results in Proposition \ref{prop:neighborhood-growth}, the lower bound concerns distinct radius regimes \( j \leq L \) and \( j > L \). The former, where the branching process approximation is valid but does not yield any contribution to the intersection, utilizes Markov's inequality to prove an upper bound on \( |N_{L}(u_2)_t| \), while the latter uses Lemma \ref{lem:neighborhood-growth} together with Theorem 2.8 and Corollary 2.4 from \cite{JLR00} to establish an upper bound on \( |\partial N_{k_1}(u_1)_t \cap \partial N_{j}(u_2)_t| \).}

\noindent Recall all the notation from Lemma \ref{lem:neighborhood-growth} and its proof. 
Then for every $L < k_1, j \leq k= (\kappa_0+\kappa)\log_{\lambda_1} n$ and $t\in [T]$,
\begin{align*}
    \E&(|\partial N_{k_1}(u_1)_t\cap \partial N_{j}(u_2)_t| \mid N_{k_1}(u_1),N_{j-1}(u_2))\\
    &=\E\left(\sum_{x \in \partial N_{k_1}(u_1)_t \setminus N_{j-1}(u_2)_t}\1_{\{\exists t'\in [T] \: \exists y\in \partial N_{j-1}(u_2)_{t'}\: :\: I_{xy}=1\}}\mid N_{k_1}(u_1),N_{j-1}(u_2)\right)\\
    &=\left(|\partial N_{k_1}(u_1)_t|-\sum_{i\leq j-1}|\partial N_{k_1}(u_1)_t\cap \partial N_{i}(u_2)_t|\right)\left(1-\prod_{t'=1}^{T}\left(1-\frac{d_{t't}}{n_t}\right)^{|\partial N_{j-1}(u_2)_{t'}|}\right).
\end{align*}
Conditionally on $\cap_{l=0}^{k-L}\cE_{n,l}$ with $\lambda_1>1$, $\frac{d_{t't}}{n_t}\in [0,1]$, and $D$ being primitive, Lemma \ref{lem:neighborhood-growth} implies with probability at least $\left(1-(2T+1)n^{-\delta}\right)^{k}\geq 1-k(2T+1)n^{-\delta} > 1-\log_{\lambda_1}n(2T+1)n^{-\delta}$ that
\begin{align*}
|\partial N_{j-1}(u_2)_{t'}|\frac{d_{t't}}{n_t}&\leq (1+\varepsilon)[D^{j-1}]_{t(u_2)t'} \frac{d_{t't}}{\alpha n}\leq (1+\varepsilon)C\lambda_1^{j-1} \frac{d_{t't}}{\alpha n}\leq (1+\varepsilon) C d_{t't} \frac{n^{\kappa_0 + \kappa}}{\alpha n}
\end{align*}
for all $t'\in [T]$ and some constant $C>0$. Since $\kappa_0 + \kappa<1$, $|\partial N_{j-1}(u_2)_{t'}|\frac{d_{t't}}{n_t}\leq (1+\varepsilon) C d_{t't} \frac{n^{\kappa_0 + \kappa}}{\alpha n}$ vanishes, and so 
\begin{align*}
    \left(1 - \frac{d_{t't}}{n_t}\right)^{|\partial N_{j-1}(u_2)_{t'}|} = 1 - |\partial N_{j-1}(u_2)_{t'}|\frac{d_{t't}}{n_t}(1+o(1)),
\end{align*}
the identity \eqref{identity} implies
\begin{align}
    &\E(|\partial N_{k_1}(u_1)_t\cap \partial N_{j}(u_2)_t| \mid N_{k_1}(u_1),N_{j-1}(u_2))\notag\\
    &=\left(|\partial N_{k_1}(u_1)_t|-\sum_{i\leq j-1}|\partial N_{k_1}(u_1)_t\cap \partial N_{i}(u_2)_t|\right)\left(1-\prod_{t'=1}^{T}\left(1 - |\partial N_{j-1}(u_2)_{t'}|\frac{d_{t't}}{n_t}(1+o(1))\right)\right)\notag\\
    &=\left(|\partial N_{k_1}(u_1)_t|-\sum_{i\leq j-1}|\partial N_{k_1}(u_1)_t\cap \partial N_{i}(u_2)_t|\right)\sum_{t'=1}^{T}\left(|\partial N_{j-1}(u_2)_{t'}|\frac{d_{t't}}{n_t}\right)(1+o(1))\notag\\
    &= \left(|\partial N_{k_1}(u_1)_t|-\sum_{i\leq j-1}|\partial N_{k_1}(u_1)_t\cap \partial N_{i}(u_2)_t|\right)(1+o(1))\frac{\langle\overrightarrow{|\partial N_{j - 1}(u_2)|},De_{t}\rangle}{n_t}\label{intersect-error}.
\end{align}
Again by Lemma \ref{lem:neighborhood-growth} and $D$ being primitive, we obtain
\begin{align*}
    \E \big(|\partial N_{k_1}&(u_1)_t \cap \partial N_{j}(u_2)_t| \big| N_{k_1}(u_1), N_{j-1}(u_2)\big)\leq (1+\varepsilon)[D^{k_1}]_{t(u_1)t}T(1+\varepsilon)\frac{[D^{k}]_{t(u_2)t}}{\alpha n}\\
    &\leq T(1+\varepsilon)^2C^2\lambda_1^{k_1}\frac{n^{\kappa_0+\kappa}}{\alpha n}\leq \lambda_1^{k_1}\frac{n^{-\gamma}}{7(\lfloor k \rfloor-\lfloor L \rfloor)}
\end{align*}
and so
\begin{align*}
    \E \big(|\partial N_{k_1}(u_1)_t \cap \partial N_{j}(u_2)_t| \big)&=\E\big(\E \big(|\partial N_{k_1}(u_1)_t \cap \partial N_{j}(u_2)_t| \big| N_{k_1}(u_1), N_{j-1}(u_2)\big)\big)\\
    &\leq \lambda_1^{k_1}\frac{n^{-\gamma}}{7(\lfloor k \rfloor-\lfloor L \rfloor)}
\end{align*}
for all $t \in [T]$ and all $L< j\leq k$ with $C>0$ and $0<\gamma<\min\{1-\kappa_0-\kappa,\kappa_0\}$ for sufficiently large $n$. The factor $7$ is chosen so that $x \geq 7\lambda$ where $x = \lambda_1^{k_1}\frac{n^{-\gamma}}{\lfloor k \rfloor-\lfloor L \rfloor}$ is the threshold and $\lambda = \E(|\partial N_{k_1}(u_1)_t \cap \partial N_{j}(u_2)_t|)$ is the mean, satisfying the condition required by Corollary 2.4 of \cite{JLR00} that we use in later steps.

\noindent Now let $A$ be the event that there exist $t\in[T]$ and $L< j\leq k$ such that 
$$\left|\partial N_{k_1}(u_1)_t \cap \partial N_{j}(u_2)_t\right| \geq \lambda_1^{k_1} \frac{n^{-\gamma}}{\lfloor k \rfloor-\lfloor L \rfloor}.$$ Let $B$ be the event that 
$$\E(|\partial N_{k_1}(u_1)_t \cap \partial N_{j}(u_2)_t|)\leq \lambda_1^{k_1} \frac{n^{-\gamma}}{7(\lfloor k \rfloor-\lfloor L \rfloor)}$$
for all $t \in [T]$ and all $L < j\leq k$. Hence, $$\PR(A)\leq \PR(A\mid B)+\PR(B^c) < \PR(A\mid B) + \log_{\lambda_1}n(2T+1)n^{-\delta}.$$
By Theorem 2.8 and Corollary 2.4 from \cite{JLR00} with union bound, 
\begin{align*}
    \PR(A\mid B) &\leq T(\lfloor k\rfloor-\lfloor L\rfloor)\exp\left( - \lambda_1^{k_1} \frac{n^{-\gamma}}{\lfloor k \rfloor-\lfloor L \rfloor} \right)\\
    &< T ((\kappa_0+\kappa)\log_{\lambda_1}n-\kappa_0\log_{\lambda_1}n+1) \exp\left(-
    \frac{n^{\kappa_0-\gamma}}{\lfloor k \rfloor-\lfloor L \rfloor}\right) \\
    &= T (\kappa\log_{\lambda_1}n+1) \exp \left(-n^{\gamma'}\right)
\end{align*}
for $\gamma'=\kappa_0-\gamma>0$. It follows that
\begin{align}
    \PR\Bigg(A^c \mid N_{k_1}(u_1), N_{k_2-1}(u_2)\Bigg) &> 1-T (\kappa\log_{\lambda_1}n+1) \exp \left(-n^{\gamma'}\right)-\log_{\lambda_1}n(2T+1)n^{-\delta} \notag\\
    &\geq 1-\log_{\lambda_1}n(2T+2)n^{-\delta} \label{prop-all}
\end{align}
for sufficiently large $n$ as $T (\kappa\log_{\lambda_1}n+1)\exp \left(-n^{\gamma'}\right)$ vanishes faster than $\log_{\lambda_1}n(2T+1)n^{-\delta}$. 

\noindent Since $k_1> \kappa_0\log_{\lambda_1}n$, Markov's inequality and Lemma \ref{lem:neighborhood-upper-bound} imply that there exist $\gamma''\in\left(\kappa_0,\frac{k_1}{\log_{\lambda_1}n}\right)$ and $\delta'>0$ such that 
$$\PR(| N_L(u_2)_t|\geq n^{\gamma''})\leq \frac{O(\lambda_1^{L})}{n^{\gamma''}}= \frac{O\left(n^{\kappa_0}\right)}{n^{\gamma''}}\leq n^{-\delta'}$$
for $t\in [T]$ with sufficiently large $n$. Therefore,
\begin{align}
\PR(| N_L&(u_2)_t | \leq n^{\gamma''} \text{ for } t\in [T]) = 1-\PR(\exists t \in [T]  s.t.  \left| N_L(u_2)_t \right| \geq n^{\gamma''})\notag\\
&\geq 1- \sum_{t=1}^{T} \PR(| N_L(u_2)_t|\geq n^{\gamma''}) \geq 1-T n^{-\delta'}. \label{markov3}
\end{align}
Combining \eqref{intersect-error}, \eqref{prop-all}, \eqref{markov3} with Proposition \ref{prop:neighborhood-growth} and $D$ being primitive, we have w.h.p. that
\begin{align}
    &\E \big(|\partial N_{k_1}(u_1)_t \cap \partial N_{k_2}(u_2)_{t}| \big| N_{k_1}(u_1), N_{k_2-1}(u_2)\big)\notag\\
    &{\geq \left(|\partial N_{k_1}(u_1)_t| - \sum_{L<j \leq k_2 - 1} |\partial N_{k_1}(u_1)_t \cap \partial N_{j}(u_2)_t|-|N_{L}(u_2)_t|\right) (1+o(1))\frac{\langle\overrightarrow{|\partial N_{k_2 - 1}(u_2)|},De_{t}\rangle}{n_t}}\notag\\
    &\geq \left((1-\varepsilon)[D^{k_1}]_{t(u_1)t} - (\lfloor k_2\rfloor -1-\lfloor L \rfloor) \lambda_1^{k_1} \frac{n^{-\gamma}}{\lfloor k \rfloor-\lfloor L \rfloor}-n^{\gamma''}\right) \frac{\langle(1-\varepsilon)e_{t(u_2)}^\top D^{k_2-1},De_{t}\rangle}{n_t}\notag\\
    &\geq \left((1-\varepsilon)[D^{k_1}]_{t(u_1)t} -  \lambda_1^{k_1} n^{-\gamma}-n^{\gamma''}\right) \frac{(1-\varepsilon)[D^{k_2}]_{t(u_2)t}}{n_t}\notag\\
    &\geq \left((1-\varepsilon)c\lambda_1^{k_1} - \lambda_1^{k_1} n^{-\gamma}-n^{\gamma''}\right) \frac{(1-\varepsilon)c\lambda_1^{k_2}}{n_t} \geq \frac{(1-\varepsilon)^2c^2\lambda_1^{k_1+k_2}}{2n_t} \label{lower-intersect}
\end{align}
and 
\begin{align}
    &\E \big(|\partial N_{k_1}(u_1)_t \cap \partial N_{k_2}(u_2)_{t}| \big| N_{k_1}(u_1), N_{k_2-1}(u_2)\big)
    \notag\\
    &\leq |\partial N_{k_1}(u_1)_t| (1+o(1))\frac{\langle\overrightarrow{|\partial N_{k_2 - 1}(u_2)|},De_{t}\rangle}{n_t} \leq (1+\varepsilon)[D^{k_1}]_{t(u_1)t}\frac{\langle(1+\varepsilon)e_{t(u_2)}^\top D^{k_2-1},De_{t}\rangle}{n_t} \notag\\
    &= (1+\varepsilon)^2 [D^{k_1}]_{t(u_1)t}\frac{[D^{k_2}]_{t(u_2)t}}{n_t}\leq \frac{(1+\varepsilon)^2C^2\lambda_1^{k_1+k_2}}{n_t} \label{upper-intersect}
\end{align}
for some $c,C>0$, where the last "$\geq$" holds since $\lambda_1^{k_1} n^{-\gamma}$ and $n^{\gamma''}$ grow strictly slower than $(1-\varepsilon)c\lambda_1^{k_1}$ as $\gamma>0$ and $\gamma''<\frac{k_1}{\log_{\lambda_1}n}$. {Here, the two regimes $j \leq L$ and $j > L$ are controlled by different tools: for $j > L$, the event $\mathcal{E}_{n,l-L}$ from Lemma~\ref{lem:neighborhood-growth} is in force 
and the intersection layers are bounded using the event $A^c$; for $j \leq L$, the branching process approximation 
is not yet running, so the cruder Markov bound \eqref{markov3} is used instead.}
The bounds in \eqref{lower-intersect} and \eqref{upper-intersect} hold for all $t \in [T]$ w.h.p., and we denote this event $S$.

\noindent Let $\rho>0$ and $R$ denote the event that 
$$|\partial N_{k_1}(u_1)_t \cap \partial N_{k_2}(u_2)_{t}|\notin\left(\frac{(1-\varepsilon)^3c^2\lambda_1^{k_1+k_2}}{2n_t},\frac{(1+\varepsilon)^3C^2\lambda_1^{k_1+k_2}}{n_t}\right)$$ 
for all $k_1,k_2,t$ such that $k_1+k_2\geq(1+\rho) \log_{\lambda_1}n_t$. It follows that 
\begin{align*}
    \PR(R\mid N_{k_1}(u_1), N_{k_2-1}(u_2)) \leq \PR(R\mid S, N_{k_1}(u_1), N_{k_2-1}(u_2)) + \PR(S^c\mid N_{k_1}(u_1), N_{k_2-1}(u_2))
\end{align*}
Using Chernoff-Hoeffding bound \resultcite{dubhashi2009concentration}{Theorem 1.1} and union bound,
\begin{align*}
    \PR&(R\mid S, N_{k_1}(u_1), N_{k_2-1}(u_2))\leq T(\lfloor k\rfloor-\left\lceil L \right\rceil+1)^2 2\exp\left(-\frac{\varepsilon^2}{3}\frac{(1-\varepsilon)^2c^2\lambda_1^{k_1+k_2}}{2n_t}\right)\\
    &\leq T((\kappa_0+\kappa)\log_{\lambda_1}n - \kappa_0\log_{\lambda_1}n +1)^2 2\exp\left(-\frac{\varepsilon^2}{3}\frac{(1-\varepsilon)^2c^2n^{1+\rho}}{2n}\right)\\
    &= T(\kappa\log_{\lambda_1}n +1)^2 2\exp\left(-\frac{\varepsilon^2}{3}\frac{(1-\varepsilon)^2c^2n^{\rho}}{2}\right)
\end{align*}
Therefore,
\begin{align*}
    \E(\mathbf{1}_{\{R^c\}}|N_{k_1}(u_1), N_{k_2-1}(u_2))&=\PR\left(R^c\mid N_{k_1}(u_1), N_{k_2-1}(u_2)\right)\\
    &\geq \PR(S\mid N_{k_1}(u_1), N_{k_2-1}(u_2))-\PR(R\mid S, N_{k_1}(u_1), N_{k_2-1}(u_2))
\end{align*}
where the right-hand side converges to 1. Taking expectation on both sides, we get 
$$|\partial N_{k_1}(u_1)_t \cap \partial N_{k_2}(u_2)_{t}|\in\left(\frac{(1-\varepsilon)^3c^2\lambda_1^{k_1+k_2}}{2n_t},\frac{(1+\varepsilon)^3C^2\lambda_1^{k_1+k_2}}{n_t}\right)$$ 
for all $t\in [T]$ and $L\leq k_1, k_2\leq k$ satisfying $k_1+k_2\geq(1+\rho)\log_{\lambda_1}n_t$ w.h.p.

\subsection{Proof of Theorem \ref{thm:lower-bound}}\label{pf:lower-bound}

If $u_1$ and $u_2$ are not in the same connected component, $d(u_1,u_2) = \infty$ and either $d(u_1,s) = \infty$ or $d(u_2,s) = \infty$ for any landmark $s$ sampled in the local step (see Section \ref{landmark-embeddings}). Hence, $\underline{d}(u_1,u_2) = \infty$. Now we focus on the case that $u_1$ and $u_2$ are in the same connected component (i.e., $u_1\leftrightarrow u_2$).

\noindent Let $k_1 = \varepsilon' d(u_1,u_2)$ and $k_2 = (1- \varepsilon+\varepsilon') d(u_1,u_2)$ with $0<\varepsilon'=\min\left\{\frac{\varepsilon}{2},\varepsilon -\theta\right\}-\varepsilon''<\min\left\{\frac{\varepsilon}{2},\varepsilon -\theta\right\}$
for some $\varepsilon''\in \left(0,\min\left\{\frac{\varepsilon}{2},\varepsilon -\theta\right\}\right)$ (to be chosen later).
Since $\varepsilon'<\frac{\varepsilon}{2}$, $k_1+k_2<d(u_1,u_2)$, and so $N_{k_1}(u_1) \cap N_{k_2}(u_2) = \varnothing$. 

\noindent Let $S_{ij}$ be the landmark set of size $M^i$ sampled in the $j$-th round and $Z_{ij}$ denote the event that $S_{ij}\cap N_{k_1}(u_1) \neq \varnothing$ but $S_{ij}\cap N_{k_2}(u_2) = \varnothing$. If $Z_{ij}$ happens for some $i\leq r$ and $j\leq R$, then $d(u_1,S_{ij}) \leq k_1$ and $d(u_2,S_{ij}) \geq k_2$, and consequently, $\underline{d} (u_1,u_2) \geq k_2-k_1= (1-\varepsilon) d(u_1,u_2)$. Thus, denoting $Z = \cup_{i\leq r, j\leq R} Z_{ij}$, it suffices to prove that $\PR(Z \mid u_1 \leftrightarrow u_2) \xrightarrow{\sss \PR} 1$. Since $$\PR(u_1 \leftrightarrow u_2 \text{ but } u_1,u_2\notin \mathcal{C}_{\sss (1)}\mid G) = \frac{1}{n^2}\sum_{i\geq 2} |\mathcal{C}_{\sss (i)}|^2 \leq \frac{|\mathcal{C}_{\sss (2)}|}{n} \xrightarrow{\sss \PR} 0$$ by Theorems 3.1 and 3.12 from \cite{bollobas2007phase}, it suffices to show that $\PR(Z \mid u_1, u_2\in \mathcal{C}_{\sss (1)} ) \xrightarrow{\sss \PR} 1$ (or equivalently $\PR(Z^c \mid u_1, u_2\in \mathcal{C}_{\sss (1)} )\xrightarrow{\sss \PR} 0$).

\noindent Conditioning throughout on the event
$X = \{u_1,u_2 \in \mathcal{C}_{\sss(1)}\}$, we have for each $(i,j)$ that 
\begin{align*}
  \PR(Z_{ij}\mid X) \mathbf{1}_{X}
  &= \left[\bigg(1-\frac{|N_{k_2}(u_2)|}{n}\bigg)^{M^i} - \bigg(1-\frac{|N_{k_1}(u_1)| + |N_{k_2}(u_2)|}{n}\bigg)^{M^i}\right]\mathbf{1}_{X} ,
\end{align*}
since $\PR(A^c \cap B) = \PR(B) - \PR(A\cap B)$.
By independence of $Z_{ij}$'s,
\begin{align}
  &\PR(Z^c\mid X)\mathbf{1}_{X} \\
  &=\bigg(\prod_{i=0}^{r}\bigg(1-\bigg(1-\frac{|N_{k_2}(u_2)|}{n}\bigg)^{M^i} + \bigg(1-\frac{|N_{k_1}(u_1)| + |N_{k_2}(u_2)|}{n}\bigg)^{M^i}\bigg)\bigg)^R\mathbf{1}_{X} \notag\\
  &\leq \exp\bigg(-R\sum_{i=0}^r \bigg(\bigg(1-\frac{|N_{k_2}(u_2)|}{n}\bigg)^{M^i} - \bigg(1-\frac{|N_{k_1}(u_1)| + |N_{k_2}(u_2)|}{n}\bigg)^{M^i}\bigg)\bigg)\mathbf{1}_{X} \notag\\
  & = \exp\bigg(-R \sum_{i=0}^r \frac{|N_{k_1}(u_1)|}{n} \sum_{j=0}^{M^i-1} \bigg(1-\frac{|N_{k_2}(u_2)|}{n}\bigg)^{M^i - 1 - j} \bigg(1-\frac{|N_{k_1}(u_1)| + |N_{k_2}(u_2)|}{n}\bigg)^{j} \bigg)\mathbf{1}_{X} \notag\\
  & \leq \exp\bigg(-R \frac{| N_{k_1}(u_1)|}{n} \sum_{i=0}^r M^i \bigg(1-\frac{|N_{k_1}(u_1)| + |N_{k_2}(u_2)|}{n}\bigg)^{M^i-1} \bigg)\mathbf{1}_{X} \notag\\
  &\leq \exp\bigg(-R \frac{|\partial N_{k_1}(u_1)|}{n} M^{r} \bigg(1-\frac{|N_{k_1}(u_1)| + |N_{k_2}(u_2)|}{n}\bigg)^{M^r-1} \bigg)\mathbf{1}_{X} \label{exp-eqn}
\end{align}
where the first inequality uses $1-x\leq \exp(-x)$ and the last inequality follows from $\sum_{i=0}^r M^i=\frac{M^{r+1}-1}{M-1}>\frac{M^{r+1}-M^{r}}{M-1}=M^r$.

\noindent Conditionally on $u_1,u_2$ being in the same connected component, Theorem 6.2 from \cite{RGCN2} implies that $d(u_1,u_2)/ \log_{\lambda_1}  n \xrightarrow{\sss \PR} 1$. In other words, $$(1-\epsilon)\log_{\lambda_1}  n\leq d(u_1,u_2)\leq (1+\epsilon)\log_{\lambda_1}  n \quad \text{w.h.p.}$$ for any fixed $\epsilon>0$. 
Choosing $\epsilon$ small enough so that $\varepsilon'\epsilon<1-\varepsilon'$, we then obtain $$k_1\leq \varepsilon'(1+\epsilon)\log_{\lambda_1}  n < \log_{\lambda_1} n \quad \text{w.h.p.},$$ which allows us to apply Proposition \ref{prop:neighborhood-growth} on $|\partial N_{k_1}(u_1)|$. By Proposition \ref{prop:neighborhood-growth} and $D$ being primitive, there exists $c>0$ such that
\begin{align}
    |\partial N_{k_1}(u_1)|\geq \sum_{t=1}^{T}(1-\varepsilon)[D^{k_1}]_{t(u_1)t}\geq T(1-\varepsilon)c\lambda_1^{k_1}\geq T(1-\varepsilon)c\cdot n^{\varepsilon'(1-\epsilon)} \quad \text{w.h.p.}\label{bound1}
\end{align}

\noindent Next, choosing $\epsilon$ small enough so that $\epsilon(1-\varepsilon+\varepsilon')<(1-\theta)-(1-\varepsilon+\varepsilon')$, we have that $$k_2\leq (1-\varepsilon+\varepsilon')(1+\epsilon)\log_{\lambda_1} n<\left(1-\theta\right)\log_{\lambda_1} n \quad \text{w.h.p.},$$ and so there exists $\gamma\in \left(0,1-\theta\right)$ such that $k_1<k_2<\gamma\log_{\lambda_1} n$. By Markov's inequality and Lemma \ref{lem:neighborhood-upper-bound}, there exists $\delta>0$ such that 
$$\PR(| N_{k_i}(u_i)_t|\geq n^\gamma)\leq \frac{O(\lambda_1^{k_i})}{n^\gamma}\leq n^{-\delta}$$
for $i=1,2$ and $t\in [T]$ with sufficiently large $n$. Therefore,
\begin{align*}
\PR(|N_{k_i}(u_i)_t | &\leq n^\gamma \text{ for } i = 1,2 \text{ and } t\in [T]) \\
&= 1-\PR(\exists i \in \{1,2\}, t \in [T]  s.t.  \left| N_{k_i}(u_i)_t \right| \geq n^\gamma)\notag\\
&\geq 1- \sum_{i=1,2}\sum_{t=1}^{T} \PR(| N_{k_i}(u_i)_t|\geq n^\gamma) \geq 1-2Tn^{-\delta},
\end{align*}
and so 
\begin{align}
    |N_{k_i}(u_i)|\leq T n^\gamma \quad \text{for }i=1,2 \text{ w.h.p.}\label{bound2}
\end{align}

\noindent From \eqref{bound1} and \eqref{bound2},
\begin{align*}
  \PR(Z^c\mid X) \mathbf{1}_{X} &\leq \exp\bigg(-R \frac{T(1-\varepsilon)c\cdot n^{\varepsilon'(1-\epsilon)}}{n} M^{\frac{\theta}{\log M} \log n-1} \bigg(1-\frac{2T n^\gamma}{n}\bigg)^{M^{\frac{\theta}{\log M} \log n}} \bigg)\notag\\
  &= \exp\bigg(-R \frac{T (1-\varepsilon)c\cdot n^{\left(\min\left\{\frac{\varepsilon}{2},\varepsilon -\theta\right\}-\varepsilon''\right)(1-\epsilon)}}{nM} n^{\theta} \bigg(1-\frac{2T n^\gamma}{n}\bigg)^{n^{\theta}} \bigg).
\end{align*}
Since $\gamma+\theta<1$, $\bigg(1-\frac{2T n^{\gamma}}{n}\bigg)^{n^{\theta}} \geq 1-\frac{2T n^{\gamma+\theta}}{n} \to 1$ as $n\to \infty$. Since $\epsilon$ can be chosen small enough so that $-\varepsilon''\epsilon+\varepsilon''+\epsilon\min\left\{\frac{\varepsilon}{2},\varepsilon -\theta\right\}<\varsigma$ for any $\varsigma>0$, $$R=\Omega\left(n^{1-\theta-\min\left\{\frac{\varepsilon}{2},\varepsilon -\theta\right\}+\varsigma}\right)$$ is sufficient for the final bound to tend to 0. By Assumption \ref{assumption2}, $\mathbb{P}(x\in \mathcal{C}_{(1)})$ is strictly bounded away from zero \citep[Theorem 3.1]{bollobas2007phase}, so $\mathbb{P}(X)$ does not vanish. Then since $\PR(Z^c\mid X)\mathbf{1}_{X}$ vanishes, $\PR(Z^c\mid X)$ thus vanishes.

\subsection{Proof of Theorem \ref{thm:upper-bound}}\label{pf:upper-bound}

If $u_1$ and $u_2$ are not in the same connected component, $d(u_1,u_2) = \infty$ and either $d(u_1,s) = \infty$ or $d(u_2,s) = \infty$ for any landmark $s$ sampled in the local step (see Section \ref{landmark-embeddings}). Hence, $\bar{d}(u_1,u_2) = \infty$. Now we focus on the case that $u_1$ and $u_2$ are in the same connected component (i.e., $u_1\leftrightarrow u_2$).

\noindent Let $k =\varepsilon' d(u_1,u_2)$ with 
$0<\varepsilon'=\frac{1+\varepsilon}{2}-\varepsilon''<\frac{1+\varepsilon}{2}$
for some $\varepsilon''\in \left(0,\frac{1+\varepsilon}{2}\right)$ (to be chosen later).

\noindent Let $S_{ij}$ be the landmark set of size $M^i$ sampled in the $j$-th round and $Z_{ij}$ be the event that $S_{ij}$ contains at least one landmark node in $N_k(u_1)\cap N_k(u_2)$ and none in $(N_k(u_1)\cup N_k(u_2))\setminus (N_k(u_1)\cap N_k(u_2))$. If $Z_{ij}$ happens for some $i\leq r$ and $j\leq R$, the landmarks in the intersection will be the common landmarks for calculating $\bar{d}(u_1,u_2)$, and so $\bar{d}(u_1,u_2) \leq 2k\leq (1+\varepsilon) d(u_1,u_2)$. Thus, denoting $Z = \cup_{i\leq r, j\leq R} Z_{ij}$, it suffices to prove that $\PR(Z \mid u_1 \leftrightarrow u_2) \xrightarrow{\sss \PR} 1$. Since $$\PR(u_1 \leftrightarrow u_2 \text{ but } u_1,u_2\notin \mathcal{C}_{\sss (1)}\mid G) = \frac{1}{n^2}\sum_{i\geq 2} |\mathcal{C}_{\sss (i)}|^2 \leq \frac{|\mathcal{C}_{\sss (2)}|}{n} \xrightarrow{\sss \PR} 0$$ by Theorems 3.1 and 3.12 from \cite{bollobas2007phase}, it suffices to show that $\PR(Z \mid u_1, u_2\in \mathcal{C}_{\sss (1)} ) \xrightarrow{\sss \PR} 1$ (or equivalently $\PR(Z^c \mid u_1, u_2\in \mathcal{C}_{\sss (1)} )\xrightarrow{\sss \PR} 0$). 

\noindent Conditioning throughout on the event
$X = \{u_1,u_2 \in \mathcal{C}_{\sss(1)}\}$, we have for each (i,j) that
\begin{align*}
    \PR(Z_{ij} &\mid X) \mathbf{1}_{X}\\
    &=\frac{|N_k(u_1)\cap N_k(u_2)|}{n}\left(\frac{|N_k(u_1)\cap N_k(u_2)|}{n}+1-\frac{|N_k(u_1)\cup N_k(u_2)|}{n}\right)^{|S_{ij}|-1}\mathbf{1}_{X}.
\end{align*}
By independence of $Z_{ij}$'s,
\begin{align*}
    &\PR(Z^c \mid X)\mathbf{1}_{X} \\
    &= \left(\prod_{i=0}^r\left(1-\frac{|N_k(u_1)\cap N_k(u_2)|}{n}\left(\frac{|N_k(u_1)\cap N_k(u_2)|}{n}+1-\frac{|N_k(u_1)\cup N_k(u_2)|}{n}\right)^{|S_{ij}|-1}\right)\right)^R\mathbf{1}_{X}\\
    &\leq \exp\left(-R\sum_{i=0}^{r}\frac{|\partial N_k(u_1)_{t}\cap \partial N_k(u_2)_{t}|}{n}\left(\frac{|\partial N_k(u_1)_{t}\cap \partial N_k(u_2)_{t}|}{n}+1-\frac{|N_k(u_1)|+|N_k(u_2)|}{n}\right)^{M^i-1}\right)\mathbf{1}_{X}
\end{align*}
for any node type $t\in [T]$, since $1-x\leq \exp(-x)$ for $x\geq 0$.

\noindent Conditionally on $u_1,u_2$ being in the same connected component, Theorem 6.2 from \cite{RGCN2} implies that $d(u_1,u_2)/ \log_{\lambda_1}  n \xrightarrow{\sss \PR} 1$. In other words, $$(1-\epsilon)\log_{\lambda_1}  n\leq d(u_1,u_2)\leq (1+\epsilon)\log_{\lambda_1}  n \quad \text{w.h.p.}$$ for any fixed $\epsilon>0$. Choosing $\epsilon$ small enough so that $\varepsilon'\epsilon<1-\varepsilon'$, we then obtain $$k\leq \varepsilon'(1+\epsilon)\log_{\lambda_1}n<\log_{\lambda_1}n \quad \text{w.h.p.}$$ This allows us to use Proposition \ref{prop:neighborhood-growth}.

\noindent Since $(1+\varepsilon)\log_{\lambda_1}n>\log_{\lambda_1} n_{t}$, we can choose $\varepsilon'',\epsilon$ small enough so that $$\left(-2\varepsilon''\epsilon+2\varepsilon''+\epsilon(1+\varepsilon)\right)\log_{\lambda_1}n<(1+\varepsilon)\log_{\lambda_1}n-\log_{\lambda_1} n_{t}$$ and obtain $$2k\geq 2\left(\frac{1+\varepsilon}{2}-\varepsilon''\right)(1-\epsilon)\log_{\lambda_1}n> \log_{\lambda_1} n_{t},$$ which allows us to also use Proposition \ref{prop:intersection-growth}.

\noindent Choosing $L=\kappa_0\log_{\lambda_1}n <\min\{k,\gamma\log_{\lambda_1} n\}$ for some $\kappa_0\in (0,\gamma)$ and $\gamma\in \left(0,1-\theta\right)$, we obtain from Markov's inequality and Lemma \ref{lem:neighborhood-upper-bound} that $$\PR(| N_L(u_i)_t|\geq n^\gamma)\leq \frac{O(\lambda_1^{L})}{n^\gamma}\leq n^{-\delta}$$ for all $i=1,2$ and $t \in [T]$ with some $\delta>0$ and sufficiently large $n$. Therefore,
\begin{align*}
\PR(| N_L (u_i)_t | \leq &n^\gamma \text{ for } t \in [T] \text{ and } i = 1,2) \\
&= 1-\PR(\exists t \in [T], i \in \{1,2\}  s.t.  \left| N_L(u_i)_t \right| \geq n^\gamma)\notag\\
&\geq 1- \sum_{i=1,2}\sum_{t=1}^{T} \PR(| N_L(u_i)_t|\geq n^\gamma) \geq 1-2T n^{-\delta},
\end{align*}
and so $|N_{L}(u_i)|\leq T n^\gamma$ for $i=1,2$ w.h.p. Then by Proposition \ref{prop:neighborhood-growth} and $D$ being primitive, there exists $C>0$ such that
\begin{align*}
    |N_{k}(u_i)|&=|N_{L}(u_i)|+\sum_{l=L+1}^{k}|\partial N_{l}(u_i)|\leq T n^\gamma+\sum_{l=L+1}^{k}\sum_{t=1}^{T}(1+\varepsilon)[D^{l}]_{t(u_i)t}\\
    &\leq T n^\gamma+\sum_{l=L+1}^{k}\sum_{t=1}^{T}(1+\varepsilon)C\lambda_1^l < T n^\gamma+\log_{\lambda_1}nT (1+\varepsilon)C\lambda_1^k
\end{align*}
for $i=1,2$ w.h.p. Combining these with Proposition \ref{prop:intersection-growth} and $n_{t}\leq n$, we have w.h.p. that
\begin{align*}
    &\PR(Z^c \mid X)\mathbf{1}_{X} \\
    &\leq \exp\left(-R\frac{(1-\varepsilon)^3c^2\lambda_1^{2k}}{2n^2}\sum_{i=0}^{r}\left(\frac{(1-\varepsilon)^3c^2\lambda_1^{2k}}{2n^2}+1-\frac{2T n^\gamma+2\log_{\lambda_1}nT (1+\varepsilon)C\lambda_1^k}{n}\right)^{M^i-1}\right)
\end{align*}

\noindent Since $\gamma<1$ and $0<\lambda_1^{k}<n$, $$0<\frac{2T n^\gamma+2\log_{\lambda_1}nT (1+\varepsilon)C\lambda_1^k}{n}-\frac{(1-\varepsilon)^3c^2\lambda_1^{2k}}{2n^2}< 1$$
for sufficiently large $n$. Also since $\sum_{i=0}^r(M^i-1)< \sum_{i=0}^r M^i=\frac{M^{r+1}-1}{M-1}\leq \frac{n^\theta M}{M-1}$,
\begin{align*}
    &\PR(Z^c \mid X)\mathbf{1}_{X}\\
    &\leq \exp\left(-R\frac{(1-\varepsilon)^3c^2\lambda_1^{2k}}{2n^2}\sum_{i=0}^{r}\left(1-\left(\frac{2T n^\gamma}{n}+\frac{2\log_{\lambda_1}nT (1+\varepsilon)C\lambda_1^k}{n}-\frac{(1-\varepsilon)^3c^2\lambda_1^{2k}}{2n^2}\right)(M^i-1)\right)\right)\\
    &< \exp\left(-R\frac{(1-\varepsilon)^3c^2\lambda_1^{2k}}{2n^2}\left(r -\left(\frac{2T n^\gamma}{n}+\frac{2\log_{\lambda_1}nT (1+\varepsilon)C\lambda_1^k}{n}-\frac{(1-\varepsilon)^3c^2\lambda_1^{2k}}{2n^2}\right)\frac{n^\theta M}{M-1}\right)\right).
\end{align*}
Since $\gamma<1-\theta$, $0<\frac{2T n^{\gamma+\theta}}{n}\to 0$ as $n\to \infty$. Since we can choose $\varepsilon'',\epsilon$ small enough so that $$-\varepsilon''\epsilon-\varepsilon''+\epsilon\frac{1+\varepsilon}{2}<\frac{1-\varepsilon}{2}-\theta,$$ we then have $$0<\frac{\lambda_1^{2k}}{n^2}n^{\theta}< \frac{\lambda_1^{k}}{n}n^{\theta}\leq \frac{n^{\left(\frac{1+\varepsilon}{2}-\varepsilon''\right)(1+\epsilon)+\theta}}{n}<1$$ for sufficiently large $n$. Then w.h.p.,
\begin{align*}
    \PR(Z^c \mid X)\mathbf{1}_{X}
    &< \exp\left(-R\frac{(1-\varepsilon)^3c^2n^{2\varepsilon'(1-\epsilon)}}{2n^2}\left(\frac{\theta}{\log M}\log n-1-\left(1+1-0\right)\right)\right)\label{prop-ub}\\
    &< \exp\left(-R\frac{(1-\varepsilon)^3c^2n^{2\left(\frac{1+\varepsilon}{2}-\varepsilon''\right)(1-\epsilon)}}{2n^2}\frac{\theta}{2\log M}\log n\right)\notag
\end{align*}
for some $c>0$. Since $\varepsilon''$ and $\epsilon$ can be chosen to be small enough so that $2\left(-\varepsilon''\epsilon+\varepsilon''+\epsilon\frac{1+\varepsilon}{2}\right)<\varsigma$ for any $\varsigma>0$, $$R=\Omega\left(n^{2-2\frac{1+\varepsilon}{2}+\varsigma}\right)=\Omega\left(n^{1-\varepsilon+\varsigma}\right)$$ is sufficient for the final bound to tend to 0. By Assumption \ref{assumption2}, $\mathbb{P}(x\in \mathcal{C}_{(1)})$ is strictly bounded away from zero \citep[Theorem 3.1]{bollobas2007phase}, so $\mathbb{P}(X)$ does not vanish. Then since $\PR(Z^c\mid X)\mathbf{1}_{X}$ vanishes, $\PR(Z^c\mid X)$ thus vanishes.

\subsection{Proof of Theorem \ref{avg-case}}\label{pf-avg-case}

\paragraph{Lower-Bound Distortion.} Suppose Proposition \ref{prop:neighborhood-growth} holds with probability at least $1-n^{-4\delta}$ instead of w.h.p. In that case, 
\[
\mathbb{P}\left(\frac{\underline{d}(u,v)}{d(u,v)} < 1-\varepsilon \;\middle|\; u \leftrightarrow v\right) \le n^{-4\delta}.
\]
Therefore, we have:
\begin{align}
\mathbb{P}\left(\underline{d}(u,v) < (1-\varepsilon)d(u,v), \; u \leftrightarrow v\right) &\le \mathbb{P}(u \leftrightarrow v)n^{-4\delta} = \mathbb{E}\left[\frac{\sum_i|\mathcal{C}_{(i)}|^2}{n^2}\right]n^{-4\delta} \le n^{-4\delta} \label{prop-upper-bound1},
\end{align}
where in the final step we use $\mathbb{E}\left[\frac{\sum_i|\mathcal{C}_{(i)}|^2}{n^2}\right] \le \mathbb{E}\left[\frac{|\mathcal{C}_{(1)}|}{n}\right] \to \zeta < 1$ because $|\mathcal{C}_{(1)}| = \Theta(n)$ and $|\mathcal{C}_{(i)}| = O(\log n)$ for $i>1$ w.h.p. \citep[Theorems 3.1 and 3.12]{bollobas2007phase}.

\noindent Let $Y = \left\{(i,j) : \underline{d}(i,j) \ge (1-\varepsilon)d(i,j), \; i \leftrightarrow j\right\}$ denote the set of node pairs where the lower-bound distance approximation succeeds. To evaluate the global average distortion over all connected pairs, we decompose the sum over $Y$ and its complement $Y^c$:
\begin{align*}
    \frac{1}{|U|}\sum_{i\leftrightarrow j}&\frac{\underline{d}(i,j)}{d(i,j)} = \frac{1}{|U|}\sum_{(i,j)\in Y}\frac{\underline{d}(i,j)}{d(i,j)} + \frac{1}{|U|}\sum_{(i,j)\in Y^c,\,i\leftrightarrow j}\frac{\underline{d}(i,j)}{d(i,j)} \\
    &\ge \frac{1}{|U|}\sum_{(i,j)\in Y}(1-\varepsilon) + 0 = (1-\varepsilon)\left(1-\frac{|Y^c|}{|U|}\right).
\end{align*}

\noindent To show $\frac{1}{|U|}\sum_{u_1\leftrightarrow u_2}\frac{\underline{d}(u_1,u_2)}{d(u_1,u_2)}\geq 1-\varepsilon$ w.h.p., it suffices to show $\frac{|Y^c|}{|U|}\xrightarrow{\mathbb{P}} 0$, and we proceed by contradiction. Suppose $\lim_{n\to \infty}\mathbb{P}\left(\frac{|Y^c|}{|U|}> \delta\right) > 0$. Then there exists a subsequence $(n_l)_{l\ge 1}$ such that $\mathbb{P}\left(\frac{|Y^c|}{|U|} > \delta\right) \ge \varepsilon'$ for some constant $\varepsilon' > 0$. By selecting a pair of vertices $(u,v)$ uniformly at random from the graph, the joint event $\left\{\underline{d}(u,v) < (1-\varepsilon)d(u,v), \; u \leftrightarrow v\right\}$ is precisely the event that the sampled pair belongs to $Y^c$. Therefore,
\begin{align*}
    &\mathbb{P}\left(\underline{d}(u,v) < (1-\varepsilon)d(u,v), \; u \leftrightarrow v\right) = \mathbb{E}\left[\frac{|Y^c|}{n_l^2}\right] = \mathbb{E}\left[\frac{|Y^c|}{|U|} \cdot \frac{|U|}{n_l^2}\right]\\
    &\geq \mathbb{E}\left[\frac{|Y^c|}{|U|} \cdot \frac{|U|}{n_l^2} \;\middle|\; \frac{|Y^c|}{|U|} > \delta\right] \mathbb{P}\left(\frac{|Y^c|}{|U|} > \delta\right) \geq \delta \varepsilon' \mathbb{E}\left[\frac{|U|}{n_l^2} \;\middle|\; \frac{|Y^c|}{|U|} > \delta\right]
\end{align*}
Again since $|\mathcal{C}_{(1)}| = \Theta(n)$ and $|\mathcal{C}_{(i)}| = O(\log n)$ for $i>1$ w.h.p. \citep[Theorems 3.1 and 3.12]{bollobas2007phase}, $\frac{|U|}{n_l^2} = \frac{\sum_i|\mathcal{C}_{(i)}|^2}{n_l^2} \xrightarrow{\mathbb{P}} \zeta^2 > 0$ as $n_l \to \infty$. It follows that for sufficiently large $n_l$,
\begin{align*}
    \mathbb{P}&\left(\underline{d}(u,v) < (1-\varepsilon)d(u,v), \; u \leftrightarrow v\right) \ge \frac{\delta \varepsilon' \zeta^2}{2}.\\
\end{align*}
However, as $n_l \to \infty$, this strictly positive constant bound directly contradicts the polynomial upper bound established in \eqref{prop-upper-bound1} as $\frac{\delta \varepsilon' \zeta^2}{2} \gg n_l^{-4\delta}$,
where the right-hand side vanishes. Thus, $\frac{|Y^c|}{|U|} \xrightarrow{\mathbb{P}} 0$, completing the proof.

\paragraph{Upper-Bound Distortion.} Suppose that Proposition \ref{prop:intersection-growth} holds with probability at least $1-n^{-4\delta}$ instead of w.h.p. In that case, 
\[
\mathbb{P}\left(\frac{\bar{d}(u,v)}{d(u,v)} > 1+\varepsilon \;\middle|\; u \leftrightarrow v\right) \le n^{-4\delta}.
\]
Therefore, we have:
\begin{align}
\mathbb{P}\left(\bar{d}(u,v) > (1+\varepsilon)d(u,v), \; u \leftrightarrow v\right) &\le \mathbb{P}(u \leftrightarrow v)n^{-4\delta} = \mathbb{E}\left[\frac{\sum_i|\mathcal{C}_{(i)}|^2}{n^2}\right]n^{-4\delta} \le n^{-4\delta} \label{prop-upper-bound2},
\end{align}
where in the final step we use $\mathbb{E}\left[\frac{\sum_i|\mathcal{C}_{(i)}|^2}{n^2}\right] \le \mathbb{E}\left[\frac{|\mathcal{C}_{(1)}|}{n}\right] \to \zeta < 1$ as $|\mathcal{C}_{(1)}| = \Theta(n)$ and $|\mathcal{C}_{(i)}| = O(\log n)$ for $i>1$ w.h.p. \citep[Theorems 3.1 and 3.12]{bollobas2007phase}.

\noindent Furthermore, with probability $1-n^{-4\delta}$ for a sufficiently small $\delta$, the diameter of the graph components satisfies $\text{diameter}(G) \le C\log_{\lambda_1} n$ \citep{Riordan2010}. For any connected pair $i \leftrightarrow j$ and a chosen landmark $s$, the upper-bound metric satisfies $\bar{d}(i,j) \le d(i,s) + d(j,s) \le 2C\log_{\lambda_1} n$. Since $d(i,j) \ge 1$ for all distinct pairs, $\frac{\bar{d}(i,j)}{d(i,j)} \le 2C \log_{\lambda_1} n$.

\noindent Let $Y$ denote the set of node pairs where the upper-bound distance approximation succeeds. To evaluate the global average distortion over all connected pairs $|U|$, we decompose the sum over $Y$ and its complement $Y^c$:
\begin{align*}
    \frac{1}{|U|}\sum_{i\leftrightarrow j}\frac{\bar{d}(i,j)}{d(i,j)} &= \frac{1}{|U|}\sum_{(i,j)\in Y}\frac{\bar{d}(i,j)}{d(i,j)} + \frac{1}{|U|}\sum_{(i,j)\in Y^c}\frac{\bar{d}(i,j)}{d(i,j)} \\
    &\le \frac{1}{|U|}\sum_{(i,j)\in Y}(1+\varepsilon) + \frac{1}{|U|}\sum_{(i,j)\in Y^c} 2C\log_{\lambda_1} n \\
    &= (1+\varepsilon)\frac{|Y|}{|U|} + 2C\log_{\lambda_1} n \frac{|Y^c|}{|U|}\\
    &= 1 + \varepsilon + \left(2C\log_{\lambda_1} n - (1+\varepsilon)\right)\frac{|Y^c|}{|U|}.
\end{align*}
Thus, to show that $\frac{1}{|U|}\sum_{i\leftrightarrow j}\frac{\bar{d}(i,j)}{d(i,j)} \le 1+\varepsilon$ w.h.p., it suffices to establish that $\frac{|Y^c|}{|U|} \xrightarrow{\mathbb{P}} 0$, and we prove this by contradiction. Suppose $\lim_{n\to \infty}\mathbb{P}\left(\frac{|Y^c|}{|U|} > \delta\right) > 0$. Then there exists a subsequence $(n_l)_{l\ge 1}$ such that $\mathbb{P}\left(\frac{|Y^c|}{|U|} > \delta\right) \ge \varepsilon'$ for some constant $\varepsilon' > 0$. By selecting a pair of vertices $(u,v)$ uniformly at random from the graph, the joint event $\left\{\bar{d}(u,v) > (1+\varepsilon)d(u,v), \; u \leftrightarrow v\right\}$ maps precisely to sampling a pair from $Y^c$. Hence,
\begin{align*}
    &\mathbb{P}\left(\bar{d}(u,v) > (1+\varepsilon)d(u,v), \; u \leftrightarrow v\right) = \mathbb{E}\left[\frac{|Y^c|}{n_l^2}\right] = \mathbb{E}\left[\frac{|Y^c|}{|U|} \cdot \frac{|U|}{n_l^2}\right]\\
    &\ge \mathbb{E}\left[\frac{|Y^c|}{|U|} \cdot \frac{|U|}{n_l^2} \;\middle|\; \frac{|Y^c|}{|U|} > \delta\right] \mathbb{P}\left(\frac{|Y^c|}{|U|} > \delta\right) \ge \delta \varepsilon' \mathbb{E}\left[\frac{|U|}{n_l^2} \;\middle|\; \frac{|Y^c|}{|U|} > \delta\right].
\end{align*}
Again since $|\mathcal{C}_{(1)}| = \Theta(n)$ and $|\mathcal{C}_{(i)}| = O(\log n)$ for $i>1$ w.h.p. \citep[Theorems 3.1 and 3.12]{bollobas2007phase}, $\frac{|U|}{n_l^2} = \frac{\sum_i|\mathcal{C}_{(i)}|^2}{n_l^2} \xrightarrow{\mathbb{P}} \zeta^2 > 0$ as $n_l \to \infty$. It follows that for sufficiently large $n_l$,
\begin{align*}
    \mathbb{P}\left(\bar{d}(u,v) > (1+\varepsilon)d(u,v), \; u \leftrightarrow v\right) &\ge \frac{\delta \varepsilon' \zeta^2}{2}.
\end{align*}
However, as $n_l \to \infty$, this logarithmic decay rate strictly dominates the polynomial upper bound established in \eqref{prop-upper-bound2} as 
\(\frac{\delta \varepsilon'\zeta^2}{2\log_{\lambda_1} n_l} \gg n_l^{-4\delta}\),
yielding a clear contradiction. Thus, $\frac{|Y^c|}{|U|}\log_{\lambda_1} n \xrightarrow{\mathbb{P}} 0$, completing the proof.

\subsection{Proof of Theorem \ref{thm:metric-sandwich}}\label{pf:sandwich}

\paragraph{Step 1. Construction and Coupling.}
For any $\delta>0$, since $\kappa \in L^2(\mathcal{X}^2, \mu\times\mu)$ may be unbounded, we use a truncation threshold $M_n > 0$ to isolate the singularities with $\lim_{n \to \infty} \frac{n}{M_n} = 0$. Let $\kappa_{M_n}(x,y) = \min(\kappa(x,y), M_n)$ be the truncated kernel. We then construct a finite measurable partition $\mathcal{P}_\delta=\{\mathcal{X}_1,\dots,\mathcal{X}_T\}$ on the bounded domain to define the step functions:
\[
\kappa_\delta^{-}(x,y) = \sum_{i,j=1}^{T} \left( \mathrm{inf}_{(x,y)\in X_i\times X_j}\kappa_{M_n}(x,y) \right) \mathbf{1}_{X_i\times X_j}(x,y)
\]
\[
\kappa_\delta^{+}(x,y) = \sum_{i,j=1}^{T} \left( \mathrm{sup}_{(x,y)\in X_i\times X_j}\kappa_{M_n}(x,y) \right) \mathbf{1}_{X_i\times X_j}(x,y)
\]
such that 
$$\|\kappa_{M_n} - \kappa_\delta^{\pm}\|_{2}\leq \frac{\delta}{2}.$$
Consequently, $\kappa_\delta^{-}(x,y) \le \kappa_{M_n}(x,y) \le \kappa_\delta^{+}(x,y)$. On the other hand, the lower bound remains globally valid on $\kappa$, i.e. $\kappa_\delta^{-}(x,y) \le \kappa(x,y)$ $\mu\times\mu$-a.e., while the upper bound holds everywhere except on the singular tail set $\Omega_n = \{(x,y) : \kappa(x,y) > M_n\}$ where $\kappa_\delta^{+}(x,y) \le \kappa(x,y)$. Since $\kappa \in L^2$, the global energy of the kernel is finite 
($\iint \kappa^2 \, d\mu^2 < \infty$), so
\begin{equation*}
    \|\kappa - \kappa_{M_n}\|_{2}^2 = \iint_{\Omega_n} |\kappa(x,y) - M_n|^2 \, d\mu(x) \, d\mu(y)
\end{equation*}
is finite and vanishes as $M_n\to\infty$. Since $\lim_{n \to \infty} \frac{n}{M_n} = 0$, there exists an $N_\delta > 0$ such that for all $n \ge N_\delta$ the tail mass satisfies $\|\kappa - \kappa_{M_n}\|_{2} \le \frac{\delta}{2}$. Then by triangle inequality,
$$\|\kappa - \kappa_\delta^{\pm}\|_{2} \le \|\kappa - \kappa_{M_n}\|_{2} + \|\kappa_{M_n} - \kappa_\delta^{\pm}\|_{2}\leq \frac{\delta}{2} + \frac{\delta}{2} = \delta.$$

\noindent We define a coupling of the graphs $G_{\kappa_\delta^{-}}, G_{\kappa_{M_n}},G_{\kappa_\delta^{+}}$, and $G_{\kappa}$ by sampling latent positions $x_1,\dots,x_n \sim \mu$ and independent edge variables $U_{ij}\sim \mathrm{Uniform}(0,1)$. An edge $(i,j)$ exists in $G_f$ if $U_{ij}\le {f(x_i,x_j)}/{n}$. Since $\kappa_\delta^{-} \le \kappa_{M_n} \le \kappa_\delta^{+}$ holds everywhere, we preserve the deterministic inclusion $E(G_{\kappa_\delta^{-}}) \subseteq E(G_{\kappa_{M_n}}) \subseteq E(G_{\kappa_\delta^{+}})$. Let 
\begin{equation*}
    X_n = \sum_{1 \le i < j \le n} \mathbf{1}_{\{(i,j) \in E(G_\kappa) \setminus E(G_{\kappa_{M_n}})\}}.
\end{equation*}
It follows that
\begin{equation*}
    \mathbb{E}[X_n] = \sum_{1 \le i < j \le n} \mathbb{P}\big((i,j) \in E(G_\kappa) \setminus E(G_{\kappa_{M_n}})\big)
\end{equation*}
An edge $(i,j)$ lands in the singular discrepancy set if and only if its uniform variable satisfies $\frac{\kappa_{M_n}(x_i,x_j)}{n} < U_{ij} \le \frac{\kappa(x_i,x_j)}{n}$. Conditioning on the sampled latent positions $x_i, x_j \sim \mu$, the width of this window is:
\begin{align*}
    \mathbb{P}\big((i,j) \in E(G_\kappa) \setminus E(G_{\kappa_{M_n}}) \mid x_i, x_j\big) &= \frac{\kappa(x_i,x_j) - \kappa_{M_n}(x_i,x_j)}{n} \nonumber \\
    &= \frac{\kappa(x_i,x_j) - M_n}{n} \mathbf{1}_{\Omega_n}(x_i,x_j)
\end{align*}
Taking the expectation over the latent space positions and noting that there are exactly $\binom{n}{2} = \frac{n(n-1)}{2}$ identically distributed pairs, we get:
\begin{align*}
    \mathbb{E}[X_n] &= \frac{n(n-1)}{2} \cdot \frac{1}{n} \iint_{\Omega_n} \big(\kappa(x,y) - M_n\big) \, d\mu(x) \, d\mu(y) \nonumber \\
    &\le \frac{n}{2} \iint_{\Omega_n} \kappa(x,y) \, d\mu(x) \, d\mu(y)
\end{align*}

\noindent On the domain $\Omega_n$, the inequality $\kappa(x,y) > M_n$ implies $1 \le \frac{\kappa(x,y)}{M_n}$. Substituting this into the integrand allows us to pull out the truncation ceiling:
\begin{align*}
    \mathbb{E}[X_n] &\leq \frac{n}{2}\iint_{\Omega_n} \kappa(x,y) \, d\mu(x) \, d\mu(y) \le \frac{n}{2}\iint_{\Omega_n} \kappa(x,y) \left( \frac{\kappa(x,y)}{M_n} \right) \, d\mu(x) \, d\mu(y) \nonumber \\
    &\le \frac{n}{2}\frac{1}{M_n} \iint_{\mathcal{X}^2} \kappa^2(x,y) \, d\mu(x) \, d\mu(y) = \frac{n\|\kappa\|_{2}^2}{2M_n}.
\end{align*}
Applying Markov's inequality, 
\begin{equation*}
    \mathbb{P}\big(E(G_\kappa) \neq E(G_{\kappa_{M_n}})\big) \le \frac{n \|\kappa\|_{2}^2}{2M_n} \longrightarrow 0 \quad \text{as } n \to \infty,
\end{equation*}
so $E(G_\kappa) = E(G_{\kappa_{M_n}})$ w.h.p., which yields the stochastic edge set inclusion:
\[
E(G_{\kappa_\delta^{-}}) \subseteq E(G_\kappa) \subseteq E(G_{\kappa_\delta^{+}}) \quad \text{w.h.p.}
\]
Since the shortest-path distance $d(u_1,u_2)$ is a monotonically non-increasing function of the edge set, we obtain:
\[
d_{\kappa_\delta^{+}}(u_1,u_2) \le d_\kappa(u_1,u_2) \le d_{\kappa_\delta^{-}}(u_1,u_2) \quad \text{w.h.p.}
\]

\paragraph{Step 2. Spectral Convergence.}

Let $\mathcal{T}_\kappa, \mathcal{T}_{\kappa_\delta^-}, \mathcal{T}_{\kappa_\delta^+}$ denote the integral operators associated with their respective kernels, acting on $L^2(\mathcal{X},\mu)$ via $(T_f \phi)(x) = \int_{\mathcal{X}} f(x,y)\phi(y) \, d\mu(y).$. Since $\kappa \in L^2(\mathcal{X}^2, \mu \times \mu)$, these are Hilbert-Schmidt operators whose operator norms are strictly bounded by the $L^2$ norm of their kernels \citep[Theorem VI.22]{reed1980methods}:
\begin{align}
    \|\mathcal{T}_\kappa - \mathcal{T}_{\kappa_\delta^\pm}\| &\le \|\mathcal{T}_\kappa - \mathcal{T}_{\kappa_\delta^\pm}\|_{\mathrm{HS}} = \|\kappa - \kappa_\delta^\pm\|_{2} \le \delta. \label{T-eq}
\end{align}

\noindent The primitivity of $\kappa$ implies that the principal eigenvalue $\lambda_1(\mathcal{T}_\kappa)$ is algebraically simple and isolated from the rest of the spectrum by a positive spectral gap. By analytic perturbation theory for linear operators \citep[Theorem 3.16]{kato1966perturbation}, the map $T \mapsto \lambda_1(T)$ is locally Lipschitz continuous with respect to the operator norm topology. Combining this directly with our $L^2$-driven operator norm bound $\|\mathcal{T}_\kappa - \mathcal{T}_{\kappa_\delta^\pm}\| \le \delta$, the bilateral spectral deviations scale linearly with $\delta$:
\begin{equation*}
    \lambda_1(\mathcal{T}_{\kappa_\delta^\pm}) = \lambda_1(\mathcal{T}_\kappa) + O(\delta). \label{eq5}
\end{equation*}

\noindent Since the true kernel is assumed to be supercritical ($\lambda_1(\mathcal{T}_\kappa) > 1$), choosing $\delta$ sufficiently small ensures that $\lambda_1(\mathcal{T}_{\kappa_\delta^-}) > 1$. Furthermore, by construction, $\kappa_\delta^-(x,y) \le \kappa_\delta^+(x,y)$ and $\kappa_\delta^-(x,y) \le \kappa(x,y)$ hold $\mu \times \mu$-a.e.\ globally across the entire domain. By the monotonicity property of the spectral radius for positive operators \citep[Proposition 4.1]{schaefer1974banach}, this entrywise operator dominance yields $\lambda_1(\mathcal{T}_{\kappa_\delta^-}) \le \min\big\{\lambda_1(\mathcal{T}_{\kappa_\delta^+}), \lambda_1(\mathcal{T}_\kappa)\big\}$.

\paragraph{Step 3. Primitivity of $\kappa_\delta^{\pm}$.}

\noindent Since the truncation $\kappa_{M_n}(x,y) = \min(\kappa(x,y), M_n)$ only dampens large values without altering the kernel's support, $\kappa_{M_n}$ natively inherits primitivity from $\kappa$. The pointwise dominance $\kappa_\delta^{+}(x,y) \ge \kappa_{M_n}(x,y)$ everywhere on $\mathcal{X}^2$ ensures that $\kappa_\delta^{+}$ preserves these positive-measure pathways, thereby inheriting primitivity as well.

%\noindent Conversely, while taking blockwise infima can introduce zeros for coarse partitions, the primitivity of $\kappa_\delta^{-}$ is guaranteed for a sufficiently small $\delta > 0$. By \eqref{T-eq}, choosing a small $\delta$ places $\mathcal{T}_{\kappa_\delta^{-}}$ within a tight operator norm neighborhood of the unperturbed operator $\mathcal{T}_\kappa$. Because $\kappa$ is primitive, the compact positive operator $\mathcal{T}_\kappa$ possesses a unique dominant eigenvalue $\lambda_1(\mathcal{T}_\kappa)$ isolated by a positive spectral gap, and a strictly positive principal eigenfunction $\phi_1 \ge c > 0$ $\mu$-a.e.\ \citep[Proposition 4.1 and Theorem 5.2]{schaefer1974banach}. By the stability of isolated, simple eigenvalues and their spectral projections \citep{kato1966perturbation}, the perturbed operator $\mathcal{T}_{\kappa_\delta^{-}}$ retains a unique, simple dominant eigenvalue, ruling out periodicity via Theorem 3.16. Furthermore, its principal eigenfunction $\phi_1^{-}$ converges in the $L^2(\mathcal{X}, \mu)$ norm to $\phi_1$. Since $\phi_1$ is bounded away from zero $\mu$-a.e., this $L^2$ convergence ensures that $\phi_1^{-}$ remains strictly positive $\mu$-a.e.\ for sufficiently small $\delta$, ruling out reducibility via Theorem 3.5. Consequently, $\kappa_\delta^{-}$ is guaranteed to inherit primitivity.

\noindent Conversely, while blockwise infima can introduce zeros, the primitivity of $\kappa_\delta^{-}$ is preserved for a sufficiently small $\delta > 0$. Because $\kappa$ is primitive, $\mathcal{T}_\kappa$ possesses an isolated, simple dominant eigenvalue $\lambda_1(\mathcal{T}_\kappa)$ and a strictly positive principal eigenfunction $\phi_1 \ge c > 0$ $\mu$-a.e.\ \citep{schaefer1974banach}. By the spectral stability of simple eigenvalues and projections \citep{kato1966perturbation}, under the operator norm perturbation $\|\mathcal{T}_\kappa - \mathcal{T}_{\kappa_\delta^{-}}\| \le \delta$ in \eqref{T-eq}, the operator $\mathcal{T}_{\kappa_\delta^{-}}$ retains a unique dominant eigenvalue $\lambda_1(\mathcal{T}_{\kappa_\delta^{-}})$, and its principal eigenfunction satisfies $\|\phi_1^{-} - \phi_1\|_{2} \to 0$ as $\delta \to 0$. For a sufficiently small $\delta$, this $L^2$ convergence guarantees that $\phi_1^{-} > 0$ $\mu$-a.e., thereby ruling out both periodicity and reducibility (via Theorems 3.16 and 3.5). Consequently, $\kappa_\delta^{-}$ inherits primitivity.

\subsection{Proof of Theorem \ref{thm:universal-distortion}} \label{pf:universal-distortion}

\paragraph{Step 1: Kernel Approximation.}

Let $\varepsilon'\in(0,\varepsilon)$ such that $(1+\varepsilon')^2<1+\varepsilon$. By Theorem~\ref{thm:metric-sandwich}, there exists a $\delta > 0$ with primitive step-function kernels $\kappa_{\delta}^{-}$ and $\kappa_{\delta}^{+}$ such that $\|\kappa - \kappa_{\delta}^{\pm}\|_2 \le \delta$. As shown in the proof of Theorem~\ref{thm:metric-sandwich}, we can construct from these step-function kernels finite-type inhomogeneous random graphs $G_{\kappa_{\delta}^-}$ and $G_{\kappa_{\delta}^+}$ on $T$ types via a monotone coupling such that 
\begin{equation}
    E(G_{\kappa_{\delta}^-}) \subseteq E(G_\kappa) \subseteq E(G_{\kappa_{\delta}^+}) \quad \text{w.h.p.} \label{edge}
\end{equation}
and
\begin{align}
    1 < \lambda_1(\mathcal{T}_{\kappa_\delta^{-}}) \leq \min\{\lambda_1(\mathcal{T}_{\kappa_\delta^{+}}),\lambda_1(\mathcal{T}_\kappa)\} \quad \text{with} \quad \lambda_1(\mathcal{T}_{\kappa_\delta^{\pm}}) = \lambda_1(\mathcal{T}_\kappa) \pm O(\delta). \label{lambda}
\end{align}

\paragraph{Step 2: Landmark Estimation on Finite Types.}

Since $G_{\kappa_{\delta}^-}$ and $G_{\kappa_{\delta}^+}$ are finite-type graphs with primitive step functions, they satisfy Assumptions~\ref{assumption1}--\ref{assumption4}. Applying Theorems~\ref{thm:lower-bound} and \ref{thm:upper-bound} under the required conditions for $\theta, r, R$, the estimators $\underline{d}$ and $\overline{d}$ constructed on $G_\kappa$ satisfy:
\begin{align}
(1-\varepsilon')d_{\kappa_{\delta}^+}(u_1,u_2)
\le
\underline{d}_{\kappa_{\delta}^+}(u_1,u_2)
\quad\text{and}\quad
\overline{d}_{\kappa_{\delta}^-}(u_1,u_2)
\le
(1+\varepsilon')d_{\kappa_{\delta}^-}(u_1,u_2) \quad \text{w.h.p.} \label{main-bounds}
\end{align}

\paragraph{Step 3: Upper-Bound Distortion.} As explained in Section~\ref{landmark-embeddings},
\[
\bar{d}(u_1,u_2)=\min_{i,j\in \{0,\dots,(r+1)R-1\}}
\left\{
d(u_1,S_i)+d(u_2,S_j):
\arg\min_{s\in S_i} d(u_1,s)
=
\arg\min_{s\in S_j} d(u_2,s)
\right\},
\]
where \(d(u,S)=\min_{s\in S} d(u,s)\) denotes the distance between \(u\) and the closest landmark in \(S\). Since \(\bar{d}\) concerns common landmarks, without loss of generality we write
\[
\overline{d}_{\kappa_{\delta}^-}(u_1,u_2)
=
d_{\kappa_{\delta}^-}(u_1,s^*)
+
d_{\kappa_{\delta}^-}(u_2,s^*)
\]
for some landmark \(s^*\). Since shortest-path distances are monotone with respect to edge addition, \eqref{edge} implies
\[
\overline{d}_{\kappa_{\delta}^-}(u_1,u_2)
= d_{\kappa_{\delta}^-}(u_1,s^*) + d_{\kappa_{\delta}^-}(u_2,s^*) \geq d_{\kappa}(u_1,s^*) + d_{\kappa}(u_2,s^*) \geq \overline{d}_{\kappa}(u_1,u_2) \quad \text{w.h.p.}
\]
Combining this with \eqref{main-bounds}, we obtain
\[
\overline{d}_\kappa(u_1,u_2)
\le
\overline{d}_{\kappa_{\delta}^-}(u_1,u_2)
\le
(1+\varepsilon')d_{\kappa_{\delta}^-}(u_1,u_2) \quad \text{w.h.p.}
\]

\noindent To complete the proof for the upper-bound distortion, it suffices to show that $d_{\kappa_{\delta}^-}(u_1,u_2)\leq (1+\varepsilon')d_\kappa(u_1,u_2)$ w.h.p. Conditionally on $u_1,u_2$ being in the same connected component, Theorem 6.2 from \cite{RGCN2} implies that $d(u_1,u_2)/ \log_{\lambda_1} n \xrightarrow{\sss \PR} 1$. Thus, for any fixed $\epsilon>0$, the typical distances satisfy
\[
(1-\epsilon)\log_{\lambda_1(\mathcal{T}_{\kappa})} n \le d_\kappa(u_1,u_2) \le (1+\epsilon)\log_{\lambda_1(\mathcal{T}_{\kappa})} n \quad \text{w.h.p.}
\]
and similarly for $d_{\kappa_\delta^{-}}(u_1,u_2)$ with respect to $\lambda_1(\mathcal{T}_{\kappa_\delta^{-}})$. It follows that for any $\epsilon > 0$ and sufficiently large $n$, the distance ratio satisfies
\[
\frac{d_{\kappa_\delta^{-}}(u_1,u_2)}{d_\kappa(u_1,u_2)}
\le
\frac{(1+\epsilon)\log \lambda_1(\mathcal{T}_{\kappa})}{(1-\epsilon)\log \lambda_1(\mathcal{T}_{\kappa_\delta^{-}})} \quad \text{w.h.p.}
\]
Since the function $g(\lambda)=\frac{1}{\log \lambda}$ is continuous and differentiable for $\lambda>1$, the Mean Value Theorem paired with the spectral proximity $\lambda_1(\mathcal{T}_{\kappa_\delta^{-}}) = \lambda_1(\mathcal{T}_\kappa) - O(\delta)$ from \eqref{lambda} implies
\[
0
\le
\frac{1}{\log \lambda_1(\mathcal{T}_{\kappa_\delta^{-}})}
-
\frac{1}{\log \lambda_1(\mathcal{T}_{\kappa})}
\le
C\delta
\]
where $C = c_0 \sup_{\lambda \in [\lambda_1(\mathcal{T}_{\kappa})-\delta,\,\lambda_1(\mathcal{T}_{\kappa})]} \left|g'(\lambda)\right|$ for a structural constant $c_0 > 0$ matching the error bound in \eqref{lambda}. This supreme remains strictly bounded because $g'(\lambda)=\frac{-1}{\lambda (\log \lambda)^2}$ is non-zero for all $\lambda>1$. By selecting $\delta$ and $\epsilon$ sufficiently small, we guarantee that
\[
\frac{(1+\epsilon)\log \lambda_1(\mathcal{T}_{\kappa})}{(1-\epsilon)\log \lambda_1(\mathcal{T}_{\kappa_\delta^{-}})}
\le
\frac{1+\epsilon}{1-\epsilon}(1+C\delta\log \lambda_1(\mathcal{T}_{\kappa}))
\le
1+\varepsilon'
\]
which yields $d_{\kappa_\delta^{-}}(u_1,u_2) \le (1+\varepsilon')d_\kappa(u_1,u_2)$ w.h.p. Compounding these bounds into our original chain yields:
\[
\overline{d}_\kappa(u_1,u_2)
\le
(1+\varepsilon')d_{\kappa_{\delta}^-}(u_1,u_2) \le (1+\varepsilon')^2 d_\kappa(u_1,u_2) \le (1+\varepsilon) d_\kappa(u_1,u_2) \quad \text{w.h.p.}
\]

\paragraph{Step 4: Lower-Bound Distortion.}

\noindent To prove $\underline{d}_{\kappa}(u_1,u_2) \ge (1-\varepsilon)d_{\kappa}(u_1,u_2)$ w.h.p., we use proof techniques similar to those in the proof of Theorem~\ref{thm:lower-bound}. Let $Z$ be defined as in the proof of Theorem~\ref{thm:lower-bound} and $X$ be the event that $u_1,u_2$ are in the giant component for the graph $G_{\kappa}$. Let $(Z^{+},X^{+})$ and $(Z^{-},X^{-})$ denote the corresponding events for the graphs $G_{\kappa_{\delta}^{+}}$ and $G_{\kappa_{\delta}^{-}}$, respectively. From \eqref{edge}, we have $X^-\subseteq X \subseteq X^+$ w.h.p., so
\begin{align*}
    \mathbb{P}(Z|X) &\geq \mathbb{P}(Z\cap X^-|X)=\frac{\mathbb{P}(Z\cap X^- \cap X)}{\mathbb{P}(X)} =\frac{\mathbb{P}(Z\cap X^-)}{\mathbb{P}(X^-)}\frac{\mathbb{P}(X^-)}{\mathbb{P}(X)} = \mathbb{P}(Z|X^-)\frac{\mathbb{P}(X^-)}{\mathbb{P}(X)}.
\end{align*}

\noindent By \citet[Theorem 3.1]{bollobas2007phase}, supercriticality from \eqref{lambda} ensures that the giant component probabilities match the integrated branching survival profiles $\int_{\mathcal{X}} \zeta_\kappa \,d\mu$ and $\int_{\mathcal{X}} \zeta_{\kappa_\delta^-} \,d\mu$. Since $\mu(\mathcal{X})=1$, the Cauchy--Schwarz inequality bridges the norms via $\|\kappa - \kappa_\delta^-\|_1 \le \|\kappa - \kappa_\delta^-\|_2 \le \delta$, so $\|\zeta_\kappa - \zeta_{\kappa_\delta^-}\|_1 = O(\delta)$ follows from \citet[Theorem 6.4]{bollobas2007phase} which establishes the qualitative continuity of these profiles under $L^1$ kernel perturbations. Consequently, using the algebraic identity $a^2 - b^2 = (a+b)(a-b)$ and noting that both survival functions are bounded above by $1$, we obtain
\[
\mathbb{P}(X) - \mathbb{P}(X^-) \le 2 \int_{\mathcal{X}} \left|\zeta_\kappa(x) - \zeta_{\kappa_\delta^-}(x)\right| d\mu(x) = 2 \|\zeta_\kappa - \zeta_{\kappa_\delta^-}\|_1 = O(\delta).
\]
Given that $\mathbb{P}(x\in \mathcal{C}_{(1)}|\kappa)$ is strictly bounded away from zero by the supercriticality of $\kappa$ \citep[Theorem 3.1]{bollobas2007phase}, we obtain $\mathbb{P}(X) > 0$, which implies
\[
\frac{\mathbb{P}(X^-)}{\mathbb{P}(X)} = 1 - \frac{\mathbb{P}(X) - \mathbb{P}(X^-)}{\mathbb{P}(X)} \ge 1 - \frac{2\|\zeta_\kappa - \zeta_{\kappa_\delta^-}\|_2}{\mathbb{P}(X)} = 1 - O(\delta).
\]
Taking $\delta \to 0$ yields $\frac{\mathbb{P}(X^-)}{\mathbb{P}(X)} \to 1$, which implies 
\begin{align*}
    \mathbb{P}(Z|X) &\geq \mathbb{P}(Z|X^-)\frac{\mathbb{P}(X^-)}{\mathbb{P}(X)} \to \mathbb{P}(Z|X^-)
\end{align*}

\noindent Next, we lower bound $\mathbb{P}(Z \mid X^-)$ by upper bounding $\mathbb{P}(Z^c\mid X^-)$. Recall the definitions of the neighborhood $N_{k}(u)$ and the boundary $\partial N_{k}(u)$ as formulated in the proof of Theorem \ref{thm:lower-bound} for the graph $G_\kappa$. Analogously, we define $(N^+_{k}(u), \partial N^+_{k}(u))$ and $(N^-_{k}(u), \partial N^-_{k}(u))$ to be the corresponding node sets in the random graphs $G_{\kappa_{\delta}^{+}}$ and $G_{\kappa_{\delta}^{-}}$, respectively. We also define $k_1 = \varepsilon' d_\kappa(u_1,u_2)$ and $k_2 = (1- \varepsilon+\varepsilon') d_\kappa(u_1,u_2)$ with 
\begin{align*}
    0<\varepsilon'=\min\left\{\frac{\varepsilon}{2},\varepsilon -\theta\right\}-\varepsilon''<\min\left\{\frac{\varepsilon}{2},\varepsilon -\theta\right\}
\end{align*} 
for some $\varepsilon''\in \left(0,\min\left\{\frac{\varepsilon}{2},\varepsilon -\theta\right\}\right)$ chosen so that $N_{k_1}(u_1) \cap N_{k_2}(u_2) = \varnothing$. Conditioning throughout on the event $X^-$, we have similar to \eqref{exp-eqn} that
\begin{align*}
  \PR(Z^c\mid X^-)\mathbf{1}_{X^-} &\leq \exp\bigg(-R \frac{|\partial N_{k_1}(u_1)|}{n} M^{r} \bigg(1-\frac{|N_{k_1}(u_1)| + |N_{k_2}(u_2)|}{n}\bigg)^{M^r-1} \bigg)\mathbf{1}_{X^-}\\
  &\leq \exp\bigg(-R \frac{|\partial N^-_{k_1}(u_1)|}{n} M^{r} \bigg(1-\frac{|N^+_{k_1}(u_1)| + |N^+_{k_2}(u_2)|}{n}\bigg)^{M^r-1} \bigg)\mathbf{1}_{X^-},
\end{align*}
where "$\leq$" follows from the monotone edge coupling in \eqref{edge}.

\noindent Conditioned on $X^-$ with $X^-\subseteq X \subseteq \{u_1,u_2 \text{ are in the same connected component of } G_\kappa\}$, Theorem 6.2 from \cite{RGCN2} implies that $d_\kappa(u_1,u_2)/ \log_{\lambda_1(\mathcal{T}_\kappa)} n \xrightarrow{\mathbb{P}} 1$. In other words, 
\[
(1-\epsilon)\log_{\lambda_1(\mathcal{T}_\kappa)} n \le d_\kappa(u_1,u_2) \le (1+\epsilon)\log_{\lambda_1(\mathcal{T}_\kappa)} n \quad \text{w.h.p.}
\]
for any fixed $\epsilon>0$. Choosing $\epsilon$ small enough so that $\varepsilon'\epsilon < 1 - \varepsilon'$, we obtain 
\[
k_1 \le \varepsilon'(1+\epsilon)\log_{\lambda_1(\mathcal{T}_\kappa)}n < \frac{\log {\lambda_1(\mathcal{T}_{\kappa_\delta^-})}}{\log{\lambda_1(\mathcal{T}_\kappa)}}\log_{\lambda_1(\mathcal{T}_{\kappa_\delta^-})}n \leq \log_{\lambda_1(\mathcal{T}_{\kappa_\delta^-})}n \quad \text{w.h.p.},
\]
which allows us to apply Proposition \ref{prop:neighborhood-growth} to $|\partial N^-_{k_1}(u_1)|$. By Proposition \ref{prop:neighborhood-growth} and the fact that $\kappa_{\delta}^-$ is primitive, there exists a constant $c>0$ such that
\begin{align}
    |\partial N^-_{k_1}(u_1)| \geq \sum_{t=1}^{T}&(1-\varepsilon)c\lambda_1(\mathcal{T}_{\kappa_\delta^-})^{k_1} \geq T(1-\varepsilon)c\cdot n^{\varepsilon'(1-\epsilon)\log_{\lambda_1(\mathcal{T}_{\kappa})} \lambda_1(\mathcal{T}_{\kappa_\delta^-})} \quad \text{w.h.p.}\label{bound5}
\end{align}

\noindent Next, recall from \eqref{lambda} that $\lambda_1(\mathcal{T}_{\kappa_\delta^{+}}) = \lambda_1(\mathcal{T}_\kappa) + O(\delta)$. Choosing $\epsilon$ and $\delta$ small enough so that $\epsilon(1-\varepsilon+\varepsilon')<(1-\theta)\log_{\lambda_1(\mathcal{T}_{\kappa_\delta^+})}\lambda_1(\mathcal{T}_{\kappa})-(1-\varepsilon+\varepsilon')$, we obtain 
\begin{align*}
k_2&\leq (1-\varepsilon+\varepsilon')(1+\epsilon)\log_{\lambda_1(\mathcal{T}_{\kappa})} n<\left(1-\theta\right)\log_{\lambda_1(\mathcal{T}_{\kappa_\delta^+})}n \quad \text{w.h.p.},
\end{align*}
and so there exists $\gamma\in \left(0,1-\theta\right)$ such that $k_1<k_2<\gamma\log_{\lambda_1(\mathcal{T}_{\kappa_{\delta}^+})} n$. By Markov's inequality and Lemma \ref{lem:neighborhood-upper-bound}, there exists $\delta'>0$ such that 
$$\PR(| N^+_{k_i}(u_i)_t|\geq n^\gamma)\leq \frac{O(\lambda_1(\mathcal{T}_{\kappa_\delta^+})^{k_i})}{n^\gamma}\leq n^{-\delta'}$$
for $i=1,2$ and $t\in [T]$ with sufficiently large $n$. Therefore,
\begin{align*}
\PR(|N^+_{k_i}(u_i)_t |&\leq n^\gamma \text{ for } i = 1,2 \text{ and } t\in [T]) \\
&= 1-\PR(\exists i \in \{1,2\}, t \in [T] \text{ s.t. } \left| N_{k_i}(u_i)_t \right| \geq n^\gamma)\notag\\
&\geq 1- \sum_{i=1,2}\sum_{t=1}^{T} \PR(| N_{k_i}(u_i)_t|\geq n^\gamma) \geq 1-2Tn^{-\delta'},
\end{align*}
and so 
\begin{align}
    |N^+_{k_i}(u_i)|\leq T n^\gamma \quad \text{for }i=1,2 \text{ w.h.p.}\label{bound6}
\end{align}

\noindent From \eqref{bound5} and \eqref{bound6},
\begin{align*}
  \PR&(Z^c\mid X^-)\mathbf{1}_{X^-} \\
  &\leq \exp\bigg(-R \frac{T(1-\varepsilon)c\cdot n^{\varepsilon'(1-\epsilon)\log_{\lambda_1(\mathcal{T}_{\kappa})} \lambda_1(\mathcal{T}_{\kappa_\delta^-})}}{n} M^{\frac{\theta}{\log M} \log n-1} \bigg(1-\frac{2T n^\gamma}{n}\bigg)^{M^{\frac{\theta}{\log M} \log n}} \bigg)\notag\\
  &= \exp\bigg(-R \frac{T (1-\varepsilon)c\cdot n^{\left(\min\left\{\frac{\varepsilon}{2},\varepsilon -\theta\right\}-\varepsilon''\right)(1-\epsilon)\log_{\lambda_1(\mathcal{T}_{\kappa})} \lambda_1(\mathcal{T}_{\kappa_\delta^-})}}{nM} n^{\theta} \bigg(1-\frac{2T n^\gamma}{n}\bigg)^{n^{\theta}} \bigg).
\end{align*}
Since $\gamma+\theta<1$, $\bigg(1-\frac{2T n^{\gamma}}{n}\bigg)^{n^{\theta}} \geq 1-\frac{2T n^{\gamma+\theta}}{n} \to 1$ as $n\to \infty$. Since $\epsilon$ can be chosen small enough so that $(-\varepsilon''\epsilon+\varepsilon''+\epsilon\min\left\{\frac{\varepsilon}{2},\varepsilon -\theta\right\})\log_{\lambda_1(\mathcal{T}_{\kappa})} \lambda_1(\mathcal{T}_{\kappa_\delta^-})<\varsigma$ for any $\varsigma>0$, $$R=\Omega\left(n^{1-\theta-\min\left\{\frac{\varepsilon}{2},\varepsilon -\theta\right\}\log_{\lambda_1(\mathcal{T}_{\kappa})} \lambda_1(\mathcal{T}_{\kappa_\delta^-})+\varsigma}\right)$$ is sufficient for the final bound to tend to 0. By the supercriticality of $\kappa_\delta^-$ from \eqref{lambda}, $\mathbb{P}(x\in \mathcal{C}_{(1)}|\kappa_\delta^-)$ is strictly bounded away from zero \citep[Theorem 3.1]{bollobas2007phase}, so $\mathbb{P}(X^-)$ does not vanish. Then since $\PR(Z^c\mid X^-)\mathbf{1}_{X^-}$ vanishes, $\PR(Z^c\mid X^-)$ vanishes. Taking $\delta \to 0$ yields 
\begin{align*}
    \mathbb{P}(Z|X) &\geq \mathbb{P}(Z|X^-)\frac{\mathbb{P}(X^-)}{\mathbb{P}(X)} \to \mathbb{P}(Z|X^-) = 1-\mathbb{P}(Z^c|X^-)
\end{align*}
which converges to 1 when
$$R=\Omega\left(n^{1-\theta-\min\left\{\frac{\varepsilon}{2},\varepsilon -\theta\right\}+\varsigma}\right),$$ 
consistent with the result in Theorem \ref{thm:lower-bound}.

\section{Conclusion}
\label{sec:conclusion}

In this paper, we established sharp lower and upper $(1 \pm \varepsilon)$-distortion guarantees for landmark-based distance-preserving embeddings in inhomogeneous random graphs. 
Our analysis demonstrates that the key mechanism underlying accurate distance-preserving embeddings is exponential neighborhood growth and sufficiently rapid intersection behavior, as governed by the spectral properties of the affinity matrix $D$ or its continuous integral operator counterpart $\mathcal{T}_\kappa$. 

% ==================== MODIFIED TEXT ====================
By expanding our analysis from point-wise high-probability bounds to global empirical averages, we demonstrated that the hierarchical landmark framework successfully stabilizes metric distortion across whole topologies, proving that localized structural bottlenecks are statistically negligible. Furthermore, via a novel metric sandwiching coupling, we successfully transferred these guarantees from finite block models to infinite-dimensional, continuous latent-space kernels. This mathematical bridge establishes universal distortion bounds that easily accommodate non-parametric, scale-free, and heavy-tailed network configurations like Chung--Lu power-law models.
% =======================================================

Our results provide a unified spectral perspective linking affinity structure, neighborhood expansion, and metric distortion. By demonstrating that a small set of landmarks can accurately bound distance stretch globally, this framework bridges random graph theory with the design of probabilistic graph spanners for heterogeneous spaces. They offer theoretical justification for scalable distance-preserving node embeddings in large-scale networks exhibiting community structure, cyclic connectivity, or block decomposition. By characterizing distortion behavior up to radius $O(\log_{\lambda_1} n)$, the analysis also clarifies the precise regime in which landmark-based methods operate effectively in sparse, locally tree-like graphs.

%\paragraph{Limitations and Future Directions.}

%Overall, this work lays the theoretical groundwork for 
%distance approximation in structurally heterogeneous graphs, 
%while highlighting several avenues for extending the analysis 
%to broader graph models and practical settings.

%\red{\acks{We would like to express our sincere gratitude to [acknowledge individuals, organizations, or institutions] for their invaluable contributions to this research. We are also grateful to [mention any additional acknowledgements, such as technical assistance, data providers, or colleagues] for their support and assistance throughout the course of this work.}}

%\bibliographystyle{plainnat}
\bibliography{myIEEEabrv,er_shortest_path,bib_dissertation,bibliography}

\end{document}